\def\isarxivversion{1} 

\ifdefined\isarxivversion
\documentclass[11pt]{article}
\else
\documentclass{article}
\usepackage{icml2019}
\fi

\usepackage[numbers]{natbib}
\usepackage{amsmath}
\usepackage{amsthm}
\usepackage{amssymb}
\usepackage{algorithm}
\usepackage{subfig}
\usepackage{color}
\usepackage[english]{babel}
\usepackage{graphicx}
\usepackage{grffile}
\usepackage{wrapfig,epsfig}
\usepackage{epstopdf}
\usepackage{url}
\usepackage{color}
\usepackage{epstopdf}
\usepackage{algpseudocode}
\usepackage[T1]{fontenc}
\usepackage{bbm}
\usepackage{comment}
\usepackage{dsfont}

\let\C\relax
\usepackage{tikz}
\usepackage{hyperref}  
\hypersetup{colorlinks=true,citecolor=blue,linkcolor=blue} 
\usetikzlibrary{arrows}

\ifdefined\isarxivversion
\usepackage[margin=1in]{geometry}
\else

\fi

\newtheorem{theorem}{Theorem}[section]
\newtheorem{lemma}[theorem]{Lemma}
\newtheorem{definition}[theorem]{Definition}

\newtheorem{proposition}[theorem]{Proposition}

\newtheorem{assumption}[theorem]{Assumption}

\newtheorem{fact}[theorem]{Fact}
\newtheorem{remark}[theorem]{Remark}
\newtheorem{claim}[theorem]{Claim}

\newtheorem{open}[theorem]{Open Problem}

\newcommand{\wt}{\widetilde}
\newcommand{\ov}{\overline}

\newcommand{\N}{\mathcal{N}}
\newcommand{\R}{\mathbb{R}}

\newcommand{\LHS}{\mathrm{LHS}}
\renewcommand{\d}{\mathrm{d}}

\renewcommand{\varepsilon}{\epsilon}

\DeclareMathOperator*{\E}{{\mathbb{E}}}

\DeclareMathOperator*{\C}{\mathbb{C}}

\DeclareMathOperator{\poly}{poly}

\DeclareMathOperator{\vect}{vec}

\DeclareMathOperator{\dis}{dis}
\DeclareMathOperator{\cts}{cts}

\makeatletter
\newcommand*{\RN}[1]{\expandafter\@slowromancap\romannumeral #1@}
\makeatother

\usepackage{lineno}

\begin{document}

\ifdefined\isarxivversion

\date{}
\title{Over-parametrization for Learning and Generalization in Two-Layer Neural Networks}
\author{
Zhao Song\thanks{\texttt{zhaosong@uw.edu}. Work done while visiting University of Washington and hosted by Yin Tat Lee.}
\and
Xin Yang\thanks{\texttt{yx1992@cs.washington.edu}. University of Washington.}
}

\else

\icmltitlerunning{Over-parametrization for Learning and Generalization in Two-Layer Neural Networks}


\twocolumn[
\icmltitle{Over-parametrization for Learning and Generalization in Two-Layer Neural Networks}



\icmlsetsymbol{equal}{*}

\begin{icmlauthorlist}
\icmlauthor{Aeiau Zzzz}{equal,to}
\icmlauthor{Bauiu C.~Yyyy}{equal,to,goo}
\icmlauthor{Cieua Vvvvv}{goo}
\icmlauthor{Iaesut Saoeu}{ed}
\icmlauthor{Fiuea Rrrr}{to}
\icmlauthor{Tateu H.~Yasehe}{ed,to,goo}
\icmlauthor{Aaoeu Iasoh}{goo}
\icmlauthor{Buiui Eueu}{ed}
\icmlauthor{Aeuia Zzzz}{ed}
\icmlauthor{Bieea C.~Yyyy}{to,goo}
\icmlauthor{Teoau Xxxx}{ed}\label{eq:335_2}
\icmlauthor{Eee Pppp}{ed}
\end{icmlauthorlist}

\icmlaffiliation{to}{Department of Computation, University of Torontoland, Torontoland, Canada}
\icmlaffiliation{goo}{Googol ShallowMind, New London, Michigan, USA}
\icmlaffiliation{ed}{School of Computation, University of Edenborrow, Edenborrow, United Kingdom}

\icmlcorrespondingauthor{Cieua Vvvvv}{c.vvvvv@googol.com}
\icmlcorrespondingauthor{Eee Pppp}{ep@eden.co.uk}

\icmlkeywords{Machine Learning, ICML}

\vskip 0.3in
]

\fi

\ifdefined\isarxivversion
\begin{titlepage}
  \maketitle
  \begin{abstract}
\ifdefined\isarxivversion
We improve the over-parametrization size over three beautiful results [Li and Liang' 2018], [Du, Zhai, Poczos and Singh' 2019] and [Arora, Du, Hu, Li and Wang' 2019] in deep learning theory.

\else

Over-parametrized neural networks with random initialization have been successfully applied to even non-convex and non-smooth objective function in practice,
and recent theoretical works explain this phenomenon for two-layer neural networks with ReLU activation.
In this work,
we improve the over-parametrization size over three beautiful results [Li and Liang' 2018], [Du, Zhai, Poczos and Singh' 2019] and [Arora, Du, Hu, Li and Wang' 2019] in deep learning theory,
and provide the following results:

\begin{itemize}
	\item We show that  two-layer neural networks with $m$ hidden neurons on $n$ training data,
 gradient descent can achieve zero training loss when $m=\Omega(n^2)$ under certain separation properties on training data.
 \item \text{[Arora, Du, Hu, Li and Wang' 2019]} gave a convincing explanation for different performances of true labels and random labels based on over-parametrization. 
 We improve the over-parametrized bound needed for this explanation. 
\end{itemize}
The key idea is to make use of independence at initialization together with advanced concentration tools.

\fi

  \end{abstract}
  \thispagestyle{empty}
\end{titlepage}

{\hypersetup{linkcolor=black}
\tableofcontents
}
\newpage

\else

\begin{abstract}
\ifdefined\isarxivversion
We improve the over-parametrization size over three beautiful results [Li and Liang' 2018], [Du, Zhai, Poczos and Singh' 2019] and [Arora, Du, Hu, Li and Wang' 2019] in deep learning theory.

\else

Over-parametrized neural networks with random initialization have been successfully applied to even non-convex and non-smooth objective function in practice,
and recent theoretical works explain this phenomenon for two-layer neural networks with ReLU activation.
In this work,
we improve the over-parametrization size over three beautiful results [Li and Liang' 2018], [Du, Zhai, Poczos and Singh' 2019] and [Arora, Du, Hu, Li and Wang' 2019] in deep learning theory,
and provide the following results:

\begin{itemize}
	\item We show that  two-layer neural networks with $m$ hidden neurons on $n$ training data,
 gradient descent can achieve zero training loss when $m=\Omega(n^2)$ under certain separation properties on training data.
 \item \text{[Arora, Du, Hu, Li and Wang' 2019]} gave a convincing explanation for different performances of true labels and random labels based on over-parametrization. 
 We improve the over-parametrized bound needed for this explanation. 
\end{itemize}
The key idea is to make use of independence at initialization together with advanced concentration tools.

\fi

\end{abstract}

\fi

\section{Introduction}

Neural networks have gained great success on many applications, including image recognition~\cite{ksh12,hzrs13}, speech recognition~\cite{gmh13,aaab+16}, game playing \cite{shmc+16,sssa+17} and so on.
Over-parametrization,
which refers to using much more parameters than necessary,
is widely believed to be important in the success of deep learning ~\cite{ksh12,lss14}. 
A mysterious observation is that over-parameterized neural networks trained with first order method can fit all training data no matter whether the data is properly labeled or randomly labeled, even when the target function is non-smooth and non-convex due to modern architecture with ReLU activations~\cite{zbhrv16}. 
Another surprising phenomenon is that in practice, over-parameterized network can improve generalization~\cite{sgs15,zk16}, which is quite different from traditional VC-dimension theory.

The expressibility of over-parameterized neural networks partially explains these phenomenons,
as the networks are wide enough to ``remember'' all input labels.
Yet this does not explain why the simple (stochastic) gradient descent(GD) scheme can find the global optima with non-smooth/non-convex objective functions,
as well as why such neural networks can generalize.
To better understand the role of over-parameterization,
there is a long line (still growing very quickly) of work proving that (stochastic) gradient descent algorithm is able to find the global minimum if the network is wide enough \cite{d17,als18,als19,all18,ll18,dzps19,dllwz19,adhlw19,cb18b,jgh18}. 
Generalization in the over-parameterized setting has been studied in \cite{bgms18,mmn18,adhlw19}.

The breakthrough result by Li and Liang \cite{ll18} is the first one that is able to explain why the greedy algorithm works very well in practice for ReLU neural network from over-parameterization perspective. 
Moreover,
their results extend to the generalization if the training data is sufficiently structured. 
Formally speaking, their results show that as long as the width $m$ is at least polynomial of number of input data $n$, then (S)GD-type algorithm can work in the following sense: we first randomly pick a weight matrix to be the initialization point, update the weight matrix according to gradient direction over each iteration, and eventually find the global minimum.
There are other work relied on input data to be random \cite{bg17,t17,zsjbd17,s17,ly17,zsd17,dltps18,glm18,brw19}, however over-parameterization theory only needs to make very mild assumption on data, e.g. separable.
The state-of-the-art result for the training process of one-hidden-layer neural network with ReLU activation function is due to Du, Zhai, Poczos and Singh \cite{dzps19}. Their beautiful result proves that $m = \Omega(n^6 \poly(\log n , 1/\delta) )$ is sufficient. 
Here $\delta$ is the failure probability and the randomness is from the random initialization and also the algorithm itself, but not from data.  
Beyond minimizing training error, 
Arora, Du, Hu, Li and Wang~\cite{adhlw19} apply over-parametrization theory to obtain a network size free generalization error bound;
they also obtain a measure on the training speed, which can explain the difference of training with true labels and random labels.
Both results require the width $m=\Omega(\poly(n,1/\delta))$ where the exponent on $n$ is relatively large.
However, the training time of GD per iteration is proportional to the width $m$,
and in popular datasets like MNIST~\cite{lbbh98} and ImageNet~\cite{ddsllf09},
the size of training samples $n$ can usually be $10K$-$100K$,
hence the current over-parametrization bound does not scale well with so large training data size.
A natural question then arises:

\emph{What is the minimal over-parameterization for provable learning and generalization in two-layer neural networks?}

 It is conjectured \cite{l18} that  $m = \Omega(n \poly(\log (n/\delta)) )$ is the right answer.
In this work, we take a step towards the theoretical hypothesis by tightening the over-parameterization bound.
To be specific,
we make the following contributions:

\begin{itemize}
  \item For training neutral networks, we improve the result \cite{dzps19} from two perspectives : one is the dependence on failure probability, and the other is the dependence on the number of input data. More precisely, we show that $m = \Omega(n^4 \poly(\log (n/\delta)) )$ is sufficient via a careful concentration analysis. More interestingly, when the input data have certain property, we can improve the bound to $m = \Omega(n^2 \poly(\log (n/\delta)))$ via a more careful concentration analysis for random matrices. 
\item For the training speed as well as the generalization,
we improve the over-parametrization bound needed in \cite{adhlw19}.
We lower the exponent on the size of training samples $n$,
and we improve the dependency on the failure probability $\delta$ from $1/\delta$ to $\poly(\log(1/\delta))$.
\item We study the problem of training over-parametrized network with regularization. 
In practice, optimizing $\ell_2$-regularized loss function usually leads to a robust model with good generalization.
We show that with proper choice of the regularization factor, the training error can converge to 0 as long as the width is sufficiently large. 
\end{itemize}
Our work is built on top of the analysis in recent works \cite{dzps19,adhlw19} combined with random matrix theory.
We draw an interesting connection between deep learning theory and Matrix Chernoff bound : we can view the width of neural network as the number of independent random matrices.


The study on concentration of summation of random variables dates back to Central Limit Theorem.
 The first modern concentration bounds were probably proposed by Bernstein \cite{b24}. Chernoff bound is an extremely popular variant, which was introduced by Rubin and published by Chernoff \cite{c52}. Chernoff bound is a fundamental tool in Theoretical Computer Science and has been used in almost every randomized algorithm paper without even stating it. One common statement is the following: given a list of independent random variables $x_1, \cdots, x_m \in [0,1]$ with mean $\mu$, then
\begin{align*}
\Pr \left[ \left| \frac{1}{m} \sum_{i=1}^m x_i - \mu \right| > \epsilon \right] \leq 2 \exp( - \Omega( m \epsilon^2 ) ) .
\end{align*}

In many applications, we are not just dealing with scalar random variables. A natural generalization of the Chernoff bound appeared in the works of Rudelson \cite{r99}, Ahlswede-Winter \cite{aw02}, and Tropp \cite{t12}. They proved that a similar concentration phenomenon is true even for matrix random variables. Given a list of independent complex Hermitian random matrices $X_1, \cdots, X_m \in \C^{n \times n}$ with mean $\mu$ and $\| X_i \| \leq 1$, $\forall i \in [m]$, then
\begin{align*}
\Pr \left[ \left\| \frac{1}{m} \sum_{i=1}^m X_i - \mu \right\| > \epsilon \right] \leq 2 n \exp( - \Omega( m \epsilon^2 ) ).
\end{align*}
For a more detailed survey and recent progress on the topic Matrix Chernoff bound,
we refer readers to \cite{t15,glss18,ks18}.

\subsection{Our Results}

We start with the definition of Gram matrix, which can be found in \cite{dzps19}.
\begin{definition}[Data-dependent function $H$]\label{def:data_dependent_function}
Given a collection of data $\{x_1, \cdots, x_n \} \subset \R^d$. For any vector $w \in \R^d$, we define symmetric matrix $H(w) \in \R^{n \times n}$ as follows
\begin{align*}
H(w)_{i,j} = x_i^\top x_j {\bf 1}_{ w^\top x_i \geq 0, w^\top x_j \geq 0 }, \forall (i,j) \in [n] \times [n].
\end{align*}
Then we define continuous Gram matrix $H^{\cts} \in \R^{n \times n}$ in the following sense
\begin{align*}
H^{\cts} = \E_{w \sim {\N}(0,I)} [ H(w) ].
\end{align*}
Similarly, we define discrete Gram matrix $H^{\dis} \in \R^{n \times n}$ in the following sense
\begin{align*}
H^{\dis} = \frac{1}{m} \sum_{r=1}^m H(w_r).
\end{align*}
\end{definition}

We use ${\N}(0,I)$ to denote Gaussian distribution. We use $\E_{w}$ to denote $\E_{w \sim {\N}(0,I)}$ and $\Pr_{w}$ to denote $\Pr_{w \sim {\N}(0,I)}$. We introduce some mild data-dependent assumption. Without loss of generality, we can assume that $\| x_i \|_2 \leq 1$, $\forall i \in [n]$. 
\begin{assumption}[Data-dependent assumption]\label{ass:data_dependent_assumption}
We made the following data-dependent assumption:\\
1. Let $\lambda = \lambda_{\min} ( \E_w [ H(w) ] )$ and $\lambda \in (0,1]$. \\
2. Let $\alpha \in [0,n]$ and $\gamma \in [0,1)$ be the parameter such that \footnote{For simplicity, let us assume $\gamma = 0$.}
\begin{align*}
\Pr_w \left[ \Big\| H(w) - \E_{w}[H] \Big\| \leq \alpha \right] \geq 1 - \gamma.
\end{align*} 
3. Let $\beta \in [0,n^2]$ be the parameter such that
\begin{align*}
 \left\|  \E_{w} \left[ \Big( H(w) - \E_w [ H(w) ] \Big) \Big( H(w) - \E_w [ H(w) ] \Big)^\top \right] \right\| \leq \beta.
\end{align*}
4. Let $\theta \in [0,\sqrt{n}]$ be parameter such that
\begin{align*}
|x_i^\top x_j| \leq \theta / \sqrt{n}, \forall i\neq j.
\end{align*}
\end{assumption}
We validate our assumptions with some examples in Appendix \ref{sec:example}.
The first assumption is from \cite{dzps19}. For more detailed discussion about that assumption, we refer the readers to \cite{dzps19}. The last assumption is similar to assumption in \cite{als18,als19}, where they assumed that for $i\neq j$, $\| x_i - x_j \|_2 \geq \theta'$. If we think of $\| x_i \|_2 = 1, \forall i \in [n]$, then we know that $(\theta')^2 \leq 2 - 2 x_i^\top x_j $. It indicates $(\theta')^2 + 2 \theta/\sqrt{n} \leq 2$. The second and the third assumption are motivated by Matrix Chernoff bound. The reason for introducing these Matrix Chernoff-type assumption is, the goal is to bound the spectral norm of the sums of random matrices in several parts of the proof. One way is to relax the spectral norm to the Frobenious norm, and bound each entry of the matrix, and finally union bound over all entries in the matrix. This could potentially lose a $\sqrt{n}$ factor compared to applying Matrix Chernoff bound. We feel these assumptions can indicate how the input data affect the over-parameterization size $m$ in a more clear way.

We state our result for the concentration of sums of independent random matrices:
\begin{proposition}[Informal of Theorem~\ref{thm:3.1_matrix_chernoff}]\label{pro:main}
Assume Part 1,2 and 3 of Assumption~\ref{ass:data_dependent_assumption}.
If $m = \Omega ( (\lambda^{-2} \beta  + \lambda^{-1} \alpha ) \log (n/\delta)  )$, then
\begin{align*}
\Pr_{ w_1, \cdots, w_m \in {\N}(0,I) } [ \| H^{\dis} - H^{\cts} \|_2 \leq \lambda / 4 ] \geq 1 - \delta.
\end{align*}
\end{proposition}
Proposition~\ref{pro:main} is a direct improvement compared to Lemma 3.1 in \cite{dzps19}, which requires $m = \Omega(\lambda^{-2} n^2 \log (n/\delta))$. Proposition~\ref{pro:main} is better when input data points have some good properties, e.g., $\beta, \alpha = o(n^2)$. However the result in \cite{dzps19} always needs to pay $n^2$ factor, no matter what the input data points are.

\begin{table*}\caption{Summary of Convergence Result. Let $m$ denotes the width of neural network. Let $n$ denote the number of input data points. Let $\delta$ denote the failure probability. Let $T$ be the number of iterations to reach $\epsilon$ training error. All the algorithms have $O(mnd)$ running time per iteration. We improve the exponent on $n$, and we also improve the dependency on $\delta$ from $1/\delta$ to $\log(1/\delta)$.}
\centering
\begin{tabular}{ | l| l| l| l| l|l|} 
\hline
{\bf Reference} & $m$ & $T$ & $\lambda$ & $\alpha$ & $\theta$ \\\hline
\cite{dzps19} & $\lambda^{-4} n^6 \poly( \log n,1/\delta)$&$\lambda^{-2}n^2\log(1/\epsilon)$ & Yes & No & No \\ \hline
Theorem~\ref{thm:main_1} & $\lambda^{-4} n^4 \log^3(n/\delta)$ & $\lambda^{-2}n^2\log(1/\epsilon)$& Yes & No & No \\ \hline
Theorem~\ref{thm:main_2} & $\lambda^{-4} n^3 \log^3(n/\delta) \cdot \alpha$ & $\lambda^{-2}\alpha n\log(1/\epsilon)$& Yes & Yes & No \\ \hline
Theorem~\ref{thm:main_3} & $\lambda^{-4} n^2 \log^3(n/\delta) \cdot \alpha (\alpha + \theta^2)$ & $\lambda^{-2}\alpha n\log(1/\epsilon)$& Yes & Yes & Yes \\ \hline
\end{tabular}
\end{table*}

We state our convergence result as follows:
\begin{theorem}[Informal of Theorem~\ref{thm:quartic}]\label{thm:main_1}
Assume Part 1 of Assumption~\ref{ass:data_dependent_assumption}. Let $m$ denote the width of neural network, let $n$ denote the number of input data points. 
If $m = \Omega( \lambda^{-4} n^4 \poly( \log (n/\delta) ) )$, then gradient descent is able to find the global minimum from a random initialization point with probability $1-\delta$.
\end{theorem}
Theorem~\ref{thm:main_1} is a direct improvement compared to Theorem 4.1 in \cite{dzps19}, which requires $m = \Omega(\lambda^{-4} n^6 \poly(\log n , 1/\delta))$. %
We improve the exponent on $n$.
Moreover,
we improve the dependency on failure probability $\delta$ from $\poly(1/\delta)$ to $\poly(\log(1/\delta))$,
which is exponentially better.

If we also allow Part 2 of Assumption~\ref{ass:data_dependent_assumption}, we can slightly improve Theorem~\ref{thm:main_1} from $n^4$ to $n^3$,
\begin{theorem}[Informal of Theorem~\ref{thm:cubic}]\label{thm:main_2}
Assume Part 1 and 2 of Assumption~\ref{ass:data_dependent_assumption}. 
If $m = \Omega( \lambda^{-4} n^3 \alpha \poly( \log (n/\delta) ) )$, then gradient descent is able to find the global minimum from a random initialization point with probability $1-\delta$.
\end{theorem}
Besides the bound on $m$, Theorem 4.1 in \cite{dzps19} requires step size $\eta$ to be $\Theta(\lambda/n^2)$. Theorem~\ref{thm:main_2} only needs step size $\eta$ to be $\Theta(\lambda/(\alpha n))$. 

Further, if we also allow Part 4 of Assumption~\ref{ass:data_dependent_assumption}, we can slightly improve Theorem~\ref{thm:main_1} from $n^4$ to $n^2$,
\begin{theorem}[Informal of Theorem~\ref{thm:quadratic}]\label{thm:main_3}
Assume Part 1, 2 and 4 of Assumption~\ref{ass:data_dependent_assumption}. 
If 
\begin{align*}
m = \Omega ( \lambda^{-4} n^2 \alpha (\theta^2+\alpha) \poly( \log(n/\delta) ) ),
\end{align*} 
the gradient descent is able to find the global minimum from a random initialization point with probability $1-\delta$.
\end{theorem}

\begin{table*}\caption{Summary of Convergence Result. Let $m$ denotes the width of neural network. Let $n$ denote the number of input data points. Let $\delta$ denote the failure probability. 
 Let $\kappa$ be the variance of weights at initialization. We improve the exponent on $n$, and we also improve the dependency on $\delta$ from $1/\delta$ to $\log(1/\delta)$.}
\centering
\begin{tabular}{ | l| l| l| l| l|} 
\hline
{\bf Reference} & $m$ & $\lambda$ & $\alpha$ & $\theta$ \\\hline
\cite{adhlw19} & $\lambda^{-4} \kappa^{-2}n^7\poly( 1/\delta)$ & Yes & No & No \\ \hline
Theorem~\ref{thm:generalization_main} & $\lambda^{-4} \kappa^{-2}n^6\poly( \log n,\log(1/\delta))$ & Yes & No & No  \\ \hline
Theorem~\ref{thm:generalization_main_assumption_2} & $\lambda^{-4}\kappa^{-2} n^4\alpha^2\poly( \log n,\log(1/\delta))$ & Yes & Yes & No  \\ \hline
\end{tabular}
\end{table*}

We can also use over-parametrization theory to explain the difference between training with true labels and training with random labels.
Write the eigen-decomposition of $H^{\cts}$ as $H^{\cts}=\sum_{i=1}^n \lambda_i v_i v_i^\top$ where $v_i\in \mathbb{R}^n$ are the eigenvectors,
and $\lambda_i>0$ are the corresponding eigenvalues.
For labels $y\in \mathbb{R}^n$,
\cite{adhlw19} relate the training error with the quantity $\left( \sum_{i=1}^n(1-\lambda_i)(v_i^\top y)^2 \right)$,
and conjecture that the true labels align well with eigenvectors with large eigenvalues,
which explains the phenomenon that neutral networks converges faster with true labels in practice.
We improve the bound of over-parametrization in two ways: we lower the exponent on the number of samples $n$,
and we improve the dependency of $\delta$ from polynomially in $1/\delta$ to polynomially in $\log(1/\delta)$.
Informally,
our result is
\begin{theorem}[Informal of Theorem~\ref{thm:generalization_main}]\label{thm:generalization_main_1}
Assume Part 1 of Assumption~\ref{ass:data_dependent_assumption}. Let $m$ denote the width of neural network, let $n$ denote the number of input data points,
let $\eta$ be the step size,
and let $\kappa$ be the variance to initialize weights.
If $m = \Omega( \lambda^{-4} \kappa^{-2}n^6 \poly( \log (n/\delta) ) )$, then with probability $1-\delta$,
after training $k$ steps,
the training error is close to $\left( \sum_{i=1}^n(1-\eta\lambda_i)^k (v_i^\top y)^2 \right)^{1/2}$.
\end{theorem}
Similarly,
we can slightly improve Theorem \ref{thm:generalization_main_1} with stronger assumptions.
\begin{theorem}[Informal of Theorem~\ref{thm:generalization_main_assumption_2}]\label{thm:generalization_main_2}
Assume Part 1 and Part 2 of Assumption~\ref{ass:data_dependent_assumption}. Let $m$ denote the width of neural network, let $n$ denote the number of input data points,
let $\eta$ be the step size,
and let $\kappa$ be the variance of the initial weights.
If $m = \Omega( \lambda^{-4}\kappa^{-2} n^4\alpha^2  \poly( \log (n/\delta) ) )$, then with probability $1-\delta$,
after training $k$ steps,
the training error is close to $\left( \sum_{i=1}^n(1-\eta\lambda_i)^k (v_i^\top y)^2 \right)^{1/2}$.
\end{theorem}

We also improve the over-parametrization size bound in \cite{adhlw19} for generalization.
\begin{theorem}[Informal of Theorem \ref{thm:generalization}]
Assume the training data is sampled some distribution with good properties.
Let $\eta$ be the step size,
and let $\kappa$ be the variance of the initial weights.
If $m=\Omega(\kappa^{-2}(n^{14}\poly(\log m,\log(1/\delta),\lambda^{-1})))$,
then with probability $1-\delta$,
the neural network generalizes well.
\end{theorem}
Here, we give the explicit exponent of $n$,
and we improve the dependency on failure probability $\delta$ from $\poly(1/\delta)$ to $\poly(\log(1/\delta))$.

\subsection{Technical Overview}

We follow the exact same optimization framework as Du, Zhai, Poczos and Singh \cite{dzps19} and Arora, Du, Hu, Li and Wang \cite{adhlw19}. We improve the bound on $m$ by doing a careful concentration analysis for random variables without changing the high-level optimization framework.

We briefly summarize the optimization framework here: the minimal eigenvalue $\lambda$ of $H^{\cts}$, as introduced in \cite{dzps19},
turns out to be closely related with the convergence rate.
As time evolves,
the weights $w$ in the network may vary;
however if $w$ stay in a ball of radius $R$ that only depends on the number of data $n$ and $\lambda$,
and particularly does not depend on the number of neurons $m$,
then we are still able to lower bound the minimal eigenvalue of $H(w)$.
On the other hand,
we want to upper bound $D$,
the actual move of $w$,
 with high probability.
 It turns out $D$ is proportional to $\frac {1}{\sqrt{m}}$.
We require $D<R$ in order to control the convergence rate.
In this way we derive a lower bound of $m$.

Next we cover the concentration techniques we use in this work.
In order to bound $\| H \|$, \cite{dzps19} relax it to Frobenius norm and then relax it to entry-wise L1 norm, 
\begin{align*}
\| H \| \leq \| H \|_F \leq \| H \|_1.
\end{align*}
Then they can bound each term of $H_{i,j}$ individually via Markov inequality.

One key observation is that $\|H\|_1$ is a quite loose bound for $\|H\|_F$,
in the sense that $\|H\|_1=\|H\|_F$ holds only if $H$ contains at most 1 non-zero entry.
This means we can work on the Frobenius norm directly, and we shall be able to obtain a tighter estimation.
By definition of $H$, it can be written as a summation of $m$ independent matrices $A_1, A_2, \cdots, A_m \in \R^{n \times n}$,
\begin{align*}
H = \frac{1}{m} \sum_{r=1}^m A_r
\end{align*}
In order to bound $\|H\|_F$,
for each $i,j$, 
we regard each $H_{ij}$ as summation of  $m$ independent random variables, 
then apply Bernstein bound to obtain experiential tail bound on the concentration of $H_{ij}$.
Finally, by taking a union bound over all the $n^2$ pairs we obtain a tighter bound for $\| H \|_F$.

We shall mention that $\|H\|_F$ is also a loose upper bound of $\|H\|$,
i.e.,
$\|H\|_F=\|H\|$ only if $H$ is a rank-1 matrix.
Hence,
if the condition number of $H$ is small,
which may happen as a property of the data,
then we may benefit from bounding $\|H\|$ directly.
We achieve this by apply matrix Chernoff bound,
which states the spectral norm of summation of $m$ independent matrices concentrates under certain conditions.

We shall stress that mutually independence plays a very important role in our argument.
Throughout the whole paper we are dealing with summations of the form $\sum_{r=1}^m y_r$ where $\{y_m\}_{r=1}^{m}$ are independent random variables.
Previous argument mainly applies Markov inequality,
which pays a factor of $1/\delta$ around the mean for error probability $\delta$.
But we can obtain much tighter concentration bound by taking advantage of independence as in Bernstein inequality and Hoeffding inequality.
This allows us to improve the dependency on $\delta$ from $1/\delta$ to $\log (1/\delta)$.

We also make use of matrix spectral norm to deal with summation of the form $\|\sum_{i=1}^n a_i x_i\|_2$ where $\{a_i\}_{i=1}^n$ are scalars and $\{x_i\}_{i=1}^n$ are vectors.
Naively applying triangle inequality leads to an upper bound proportional to $\|a\|_1$,
which can be as large as $\sqrt{n}\|a\|_2$.
Instead,
we observe that the matrix formed by $\begin{pmatrix} x_1 &\cdots& x_n\end{pmatrix}:=X$ has good singular value property,
which allows us to obtain the bound $\|X\| \cdot \|a\|_2$.
Therefore,
this bound does not rely on number of inputs explicitly.

\subsection{Open Problems}
It is interesting whether our results can be further sharpened.
Here we list some open problems for future research, which are proposed by Yin Tat Lee~\cite{l18}.
We are the first to write them down explicitly.

\begin{open}
Is it possible to show over-parametrization result for Neural Network with ReLU activation when $m= \Omega^*(n\poly(\log (n/\delta)))$?
Here $\Omega^*(\cdot)$ hides data-dependent quantities like $\lambda$.
\end{open}
Note that the above statement is true for linear activation function \cite{dh19} in the sense that the over-parametrization bound is linear in $n$.

\begin{open}
Let $d$ be the dimension of input data.
Is it possible to prove over-parametrization result when $md= \Omega^*(n\poly(\log (n/\delta)))$?
\end{open}

\paragraph{Roadmap}
We provide some basic definitions in the next paragraph. 
We introduce the probability tools we use in Appendix \ref{sec:probability_tools}. We define the optimization problem in Section~\ref{sec:problem}. We present our quartic result in Section~\ref{sec:quartic_suffices}. We improve it to cubic and quadratic in Section~\ref{sec:cubic_suffices} and Section~\ref{sec:quadratic_suffices}.
We present our over-parameterization bound for the training speed in Appendix~\ref{sec:training_speed}.
We present our over-parameterization bound for the generalization in Appendix~\ref{sec:generalization}.
We present our result of training with regularization in Appendix ~\ref{sec:regularization}.

\paragraph{Notation}
We use $[n]$ to denote $\{1,2,\cdots, n\}$. We use $\phi$ to denote ReLU activation function, i.e., $\phi(x) = \max\{x,0\}$. For an event $f(x)$, we define ${\bf 1}_{f(x)}$ such that ${\bf 1}_{f(x)} = 1$ if $f(x)$ holds and ${\bf 1}_{f(x)} = 0$ otherwise. For a matrix $A$, we use $\| A \|$ to denote the spectral norm of $A$. We define $\| A \|_F = ( \sum_{i} \sum_{j} A_{i,j}^2 )^{1/2}$ and $\| A \|_1 = \sum_{i} \sum_{j} |A_{i,j}|$.   


\section{Problem Formulation}\label{sec:problem}

Our problem formulation is the same as \cite{dzps19}. We consider a two-layer ReLU activated neural network with $m$ neurons in the hidden layer:
\begin{align*}
f (W,x,a) = \frac{1}{ \sqrt{m} } \sum_{r=1}^m a_r \phi ( w_r^\top x ) ,
\end{align*}
where $x \in \R^d$ is the input, $w_1, \cdots, w_m \in \R^d$ are weight vectors in the first layer, $a_1, \cdots, a_m \in \R$ are weights in the second layer.  For simplicity, we only optimize $W$ but not optimize $a$ and $W$ at the same time. 

Recall that the ReLU function $\phi(x)=\max\{x,0\}$.
Therefore for $r\in [m]$,
we have
\begin{align}\label{eq:relu_derivative}
\frac{f (W,x,a)}{\partial w_r}=\frac{1}{ \sqrt{m} } a_r x{\bf 1}_{ w_r^\top x \geq 0 }.
\end{align}

We define objective function $L$ as follows
\begin{align*}
L (W) = \frac{1}{2} \sum_{i=1}^n ( y_i - f (W,x_i,a) )^2 .
\end{align*}

We apply the gradient descent to optimize the weight matrix $W$ in the following standard way,
\begin{align}\label{eq:w_update}
W(k+1) = W(k) - \eta \frac{ \partial L( W(k) ) }{ \partial W(k) } .
\end{align}

We can compute the gradient of $L$ in terms of $w_r$
\begin{align}\label{eq:gradient}
\frac{ \partial L(W) }{ \partial w_r } = \frac{1}{ \sqrt{m} } \sum_{i=1}^n ( f(W,x_i,a_r) - y_i ) a_r x_i {\bf 1}_{ w_r^\top x_i \geq 0 }.
\end{align}

We consider the ordinary differential equation defined by
\begin{align}\label{eq:wr_derivative}
\frac{\d w_r(t)}{\d t}=-\frac{ \partial L(W) }{ \partial w_r }.
\end{align}

At time $t$,
let $u(t)=(u_1(t),\cdots,u_n(t))\in \mathbb{R}^n$ be the prediction vector where each $u_i(t)$ is defined as
\begin{align}\label{eq:ut_def}
u_i(t)=f(W(t),a,x_i).
\end{align}

\begin{algorithm} 
\caption{Training neural network using gradient descent.}
\label{alg:main} 
\begin{algorithmic}[1]
	\Procedure{NNTraining}{$\{(x_i,y_i)\}_{i\in [n]}$}
	\State $w_r(0) \sim \N(0,I_d)$ for $r\in [m]$.
	\For{$t=1 \to T$}
		\State $u(t) \leftarrow \frac{1}{ \sqrt{m} } \sum_{r=1}^m a_r \sigma(w_r(t)^\top X) $ \Comment{$u(t) = f(W(t),x,a) \in \R^n$, it takes $O(mnd)$ time}
		\For{$r = 1 \to m$}
			\For{$i = 1 \to n$} 
				\State $Q_{i,:} \leftarrow \frac{1}{\sqrt{m}}a_r\sigma'(w_r(t)^\top x_i)x_i^\top$ \Comment{$Q_{i,:} = \frac{\partial f(W(t),x_i,a)}{\partial w_r}$, it takes $O(d)$ time}
			\EndFor
			\State $\text{grad}_r \leftarrow - Q^\top (y - u(t) ) $
			\Comment{ $Q = \frac{\partial f} {\partial w_r } \in \R^{n \times d}$, it takes $O(nd)$ time }
			\State $w_r(t+1) \leftarrow w_r(t) - \eta \cdot \text{grad}_r $
		\EndFor
	\EndFor
	\State \Return $W$
	\EndProcedure
\end{algorithmic}
\end{algorithm}


\section{Quartic Suffices}\label{sec:quartic_suffices}

\subsection{Bounding the difference between continuous and discrete }

In this section, we restate a result from \cite{dzps19}, showing that when the width $m$ is sufficiently large,
then the continuous version and discrete version of the gram matrix of input data is close in the spectral sense.
\begin{lemma}[Lemma 3.1 in \cite{dzps19}]\label{lem:3.1}
We define $H^{\cts}, H^{\dis} \in \R^{n \times n}$ as follows
\begin{align*}
H^{\cts}_{i,j} = & ~ \E_{w \sim \N(0,I)} \left[ x_i^\top x_j {\bf 1}_{ w^\top x_i \geq 0, w^\top x_j \geq 0 } \right] , \\ 
H^{\dis}_{i,j} = & ~ \frac{1}{m} \sum_{r=1}^m \left[ x_i^\top x_j {\bf 1}_{ w_r^\top x_i \geq 0, w_r^\top x_j \geq 0 } \right].
\end{align*}
Let $\lambda = \lambda_{\min} (H^{\cts}) $. If $m = \Omega( \lambda^{-2} n^2\log (n/\delta) )$, we have 
\begin{align*}
\| H^{\dis} - H^{\cts} \|_F \leq \frac{ \lambda }{4}, \mathrm{~and~} \lambda_{\min} ( H^{\dis} ) \geq \frac{3}{4} \lambda.
\end{align*}
hold with probability at least $1-\delta$.
\end{lemma}
The proof can be found in Appendix \ref{sec:missing_proof}.

We define the event
\begin{align*}
A_{i,r} = \left\{ \exists u : \| u - \wt{w}_r \|_2 \leq R, {\bf 1}_{ x_i^\top \wt{w}_r \geq 0 } \neq {\bf 1}_{ x_i^\top u \geq 0 } \right\}.
\end{align*}
Note this event happens if and only if $| \wt{w}_r^\top x_i | < R$. Recall that $\wt{w}_r \sim \N(0,I)$. By anti-concentration inequality of Gaussian (Lemma~\ref{lem:anti_gaussian}), we have
\begin{align}\label{eq:Air_bound}
\Pr[ A_{i,r} ] = \Pr_{ z \sim \N(0,1) } [ | z | < R ] \leq \frac{ 2 R }{ \sqrt{2\pi} }.
\end{align}

\subsection{Bounding changes of $H$ when $w$ is in a small ball}
We improve the Lemma 3.2 in \cite{dzps19} from the two perspective : one is the probability, and the other is upper bound on spectral norm.

\begin{lemma}[perturbed $w$]\label{lem:3.2}
Let $R \in (0,1)$. If $\wt{w}_1, \cdots, \wt{w}_m$ are i.i.d. generated ${\N}(0,I)$. For any set of weight vectors $w_1, \cdots, w_m \in \R^d$ that satisfy for any $r\in [m]$, $\| \wt{w}_r - w_r \|_2 \leq R$, then the $H : \R^{m \times d} \rightarrow \R^{n \times n}$ defined
\begin{align*}
    H(w)_{i,j} =  \frac{1}{m} x_i^\top x_j \sum_{r=1}^m {\bf 1}_{ w_r^\top x_i \geq 0, w_r^\top x_j \geq 0 } .
\end{align*}
Then we have
\begin{align*}
\| H (w) - H(\wt{w}) \|_F < 2 n R,
\end{align*}
holds with probability at least $1-n^2 \cdot \exp(-m R /10)$.
\end{lemma}
\begin{proof}

The random variable we care is

\begin{align*}
& ~ \sum_{i=1}^n \sum_{j=1}^n | H(\wt{w})_{i,j} - H(w)_{i,j} |^2 \\
\leq & ~ \frac{1}{m^2} \sum_{i=1}^n \sum_{j=1}^n \left( \sum_{r=1}^m {\bf 1}_{ \wt{w}_r^\top x_i \geq 0, \wt{w}_r^\top x_j \geq 0} - {\bf 1}_{ w_r^\top x_i \geq 0 , w_r^\top x_j \geq 0 } \right)^2 \\
= & ~ \frac{1}{m^2} \sum_{i=1}^n \sum_{j=1}^n  \Big( \sum_{r=1}^m s_{r,i,j} \Big)^2 ,
\end{align*}

where the last step follows from for each $r,i,j$, we define
\begin{align*}
s_{r,i,j} :=  {\bf 1}_{ \wt{w}_r^\top x_i \geq 0, \wt{w}_r^\top x_j \geq 0} - {\bf 1}_{ w_r^\top x_i \geq 0 , w_r^\top x_j \geq 0 } .
\end{align*} 

We consider $i,j$ are fixed. We simplify $s_{r,i,j}$ to $s_r$.

Then $s_r$ is a random variable that only depends on $\wt{w}_r$.
Since $\{\wt{w}_r\}_{r=1}^m$ are independent,
$\{s_r\}_{r=1}^m$ are also mutually independent.

If   $\neg A_{i,r}$ and $\neg A_{j,r}$ happen,
then 
\begin{align*}
\left| {\bf 1}_{ \wt{w}_r^\top x_i \geq 0, \wt{w}_r^\top x_j \geq 0} - {\bf 1}_{ w_r^\top x_i \geq 0 , w_r^\top x_j \geq 0 } \right|=0.
\end{align*}
If   $A_{i,r}$ or $A_{j,r}$ happen,
then 
\begin{align*}
\left| {\bf 1}_{ \wt{w}_r^\top x_i \geq 0, \wt{w}_r^\top x_j \geq 0} - {\bf 1}_{ w_r^\top x_i \geq 0 , w_r^\top x_j \geq 0 } \right|\leq 1.
\end{align*}
So we have 
\begin{align*}
 \E_{\wt{w}_r}[s_r]\leq \E_{\wt{w}_r} \left[ {\bf 1}_{A_{i,r}\vee A_{j,r}} \right] 
 \leq & ~ \Pr[A_{i,r}]+\Pr[A_{j,r}] \\
 \leq & ~ \frac {4 R}{\sqrt{2\pi}} \\
 \leq & ~ 2 R,
\end{align*}
and 
\begin{align*}
    \E_{\wt{w}_r} \left[ \left(s_r-\E_{\wt{w}_r}[s_r] \right)^2 \right]
    = & ~ \E_{\wt{w}_r}[s_r^2]-\E_{\wt{w}_r}[s_r]^2 \\
    \leq & ~ \E_{\wt{w}_r}[s_r^2]\\
    \leq & ~\E_{\wt{w}_r} \left[ \left( {\bf 1}_{A_{i,r}\vee A_{j,r}} \right)^2 \right] \\
     \leq & ~ \frac {4R}{\sqrt{2\pi}} \\
     \leq  &~ 2 R .
\end{align*}
We also have $|s_r|\leq 1$.
So we can apply Bernstein inequality (Lemma~\ref{lem:bernstein}) to get for all $t>0$,{\small
\begin{align*}
    \Pr \left[\sum_{r=1}^m s_r\geq 2m R +mt \right]
    \leq & ~ \Pr \left[\sum_{r=1}^m (s_r-\E[s_r])\geq mt \right]\\
    \leq & ~ \exp \left( - \frac{ m^2t^2/2 }{ 2m R   + mt/3 } \right).
\end{align*}}
Choosing $t = R$, we get
\begin{align*}
    \Pr \left[\sum_{r=1}^m s_r\geq 3mR  \right]
    \leq & ~ \exp \left( -\frac{ m^2  R^2 /2 }{ 2 m  R + m  R /3 } \right) \\
     \leq & ~ \exp \left( - m R / 10 \right) .
\end{align*}
Thus, we can have
\begin{align*}
\Pr \left[ \frac{1}{m} \sum_{r=1}^m s_r \geq 3  R \right] \leq \exp( - m R /10 ).
\end{align*}
Therefore, we complete the proof.
\end{proof}

\begin{table*}\caption{Table of Parameters for the $m = \wt{\Omega}(n^4)$ result in Section~\ref{sec:quartic_suffices}. {\bf Nt.} stands for notations. 
 $m$ is the width of neural network. $n$ is the number of input data points. $\delta$ is the failure probability.}
\centering
\begin{tabular}{ | l| l| l| l| } 
\hline
{\bf Nt.} & {\bf Choice} & {\bf Place} & {\bf Comment} \\\hline
$\lambda$ & $:= \lambda_{\min}(H^{\cts}) $ & Assumption~\ref{ass:data_dependent_assumption} & Data-dependent \\ \hline
$R$ & $\lambda/n$ & Eq.~\eqref{eq:choice_of_eta_R} & Maximal allowed movement of weight \\ \hline
$D_{\cts}$ & $\frac{ \sqrt{n} \| y - u(0) \|_2 }{ \sqrt{m} \lambda }$ & Lemma~\ref{lem:3.3} & Actual moving distance of weight, continuous case  \\ \hline
$D$ & $\frac{ 4\sqrt{n} \| y - u(0) \|_2 }{ \sqrt{m} \lambda }$ & Lemma~\ref{lem:4.1} & Actual moving distance of weight, discrete case  \\ \hline
$\eta$ & $\lambda/n^2$ & Eq.~\eqref{eq:choice_of_eta_R} & Step size of gradient descent \\ \hline
$m$ & $\lambda^{-2} n^2 \log(n/\delta)$ & Lemma~\ref{lem:3.1} & Bounding discrete and continuous \\ \hline
$m$ & $\lambda^{-4} n^4 \log^3(n/\delta)$  & Lemma~\ref{lem:3.4} and Claim~\ref{cla:yu0} & $D < R$ and $\| y - u(0) \|_2^2 = \wt{O}(n)$ \\ \hline
\end{tabular}
\end{table*}

\subsection{Loss is decreasing while weights are not changing much}

For simplicity of notation, we provide the following definition.
\begin{definition}
For any $s \in [0,t]$, we define matrix $H(s) \in \R^{n \times n}$ as follows
\begin{align*}
H(s)_{i,j} = \frac{1}{m} \sum_{r=1}^m x_i^\top x_j {\bf 1}_{ w_r(s)^\top x_i \geq 0, w_r(s)^\top x_j \geq 0 }.
\end{align*} 
\end{definition}
With $H$ defined,
it becomes more convenient to write the dynamics of predictions (proof can be found in Appendix \ref{sec:missing_proof}).
\begin{fact}\label{fact:dudt}
$
\frac{\d}{\d t} u(t)= H(t) \cdot (y-u(t)) .
$
\end{fact}

We state two tools from previous work(delayed the proof into Appendix \ref{sec:missing_proof})

\begin{lemma}[Lemma 3.3 in \cite{dzps19}]\label{lem:3.3}
Suppose for $0 \leq s \leq t$, $\lambda_{\min} ( H( w(s) ) ) \geq \lambda / 2$. Let $D_{\cts}$ be defined as
$
D_{\cts} := \frac{ \sqrt{n} \| y - u(0) \|_2 }{ \sqrt{m} \lambda }.
$
Then we have 
\begin{align*}
1. & ~ & \| w_r(t) - w_r(0) \|_2 \leq & ~ D_{\cts} , \forall r \in [m], \\
2. & ~ & \| y - u(t) \|_2^2 \leq  &~ \exp( - \lambda t ) \cdot \| y - u(0) \|_2^2.
\end{align*}
\end{lemma}

\begin{lemma}[Lemma 3.4 in \cite{dzps19}]\label{lem:3.4}
If $D_{\cts}<R$.
then for all $t\geq 0$,
$\lambda_{\min}(H(t))\geq \frac{1}{2}\lambda$.
Moreover,
\begin{align*}
1. & ~ & \|w_r(t)-w_r(0)\|_2 \leq & ~ D_{\cts}, \forall r \in [m], \\
2. & ~ & \|y-u(t)\|_2^2 \leq & ~ \exp(-\lambda t) \cdot \|y-u(0)\|_2^2.
\end{align*}
\end{lemma}

\subsection{Convergence}
In this section we show that when the neural network is over-parametrized,
the training error converges to 0 at linear rate.
Our main result is Theorem \ref{thm:quartic}.

\begin{theorem}\label{thm:quartic}
Recall that $\lambda=\lambda_{\min}(H^{\cts})>0$.
Let $m = \Omega( \lambda^{-4} n^4 \log (n/\delta) )$, we i.i.d. initialize $w_r \in {\N}(0,I)$, $a_r$ sampled from $\{-1,+1\}$ uniformly at random for $r\in [m]$, and we set the step size $\eta = O( \lambda / n^2 )$ then with probability at least $1-\delta$ over the random initialization we have for $k = 0,1,2,\cdots$
\begin{align}\label{eq:quartic_condition}
\| u (k) - y \|_2^2 \leq ( 1 - \eta \lambda / 2 )^k \cdot \| u (0) - y \|_2^2.
\end{align}
\end{theorem}

\paragraph{Correctness}
We prove Theorem \ref{thm:quartic} by induction.
The base case is $i=0$ and it is trivially true.
Assume for $i=0,\cdots,k$ we have proved Eq. \eqref{eq:quartic_condition} to be true.
We want to show Eq. \eqref{eq:quartic_condition} holds for $i=k+1$.

From the induction hypothesis,
we have the following Lemma (see proof in Appendix \ref{sec:missing_proof}) stating that the weights should not change too much.
\begin{lemma}[Corollary 4.1 in \cite{dzps19}]\label{lem:4.1}
If Eq. \eqref{eq:quartic_condition} holds for $i = 0, \cdots, k$, then we have for all $r\in [m]$
\begin{align*}
\| w_r(k+1) - w_r(0) \|_2 \leq \frac{ 4 \sqrt{n} \| y - u (0) \|_2 }{ \sqrt{m} \lambda } := D.
\end{align*}
\end{lemma}

Next, we calculate the different of predictions between two consecutive iterations, analogue to $\frac{\d u_i(t)}{ \d t }$ term in Fact \ref{fact:dudt}.
For each $i \in [n]$, we have

\begin{align*}
& ~ u_i(k+1) - u_i(k) \\
= & ~ \frac{1}{ \sqrt{m} } \sum_{r=1}^m a_r \cdot \left( \phi( w_r(k+1)^\top x_i ) - \phi(w_r(k)^\top x_i ) \right) \\
= & ~ \frac{1}{\sqrt{m}} \sum_{r=1}^m a_r \cdot \left( \phi \left( \Big( w_r(k) - \eta \frac{ \partial L( W(k) ) }{ \partial w_r(k) } \Big)^\top x_i \right) - \phi ( w_r(k)^\top x_i ) \right).
\end{align*}

Here we divide the right hand side into two parts. $v_{1,i}$ represents the terms that the pattern does not change and $v_{2,i}$ represents the term that pattern may changes. For each $i \in [n]$,
we define the set $S_i\subset [m]$ as
\begin{align*}
    S_i:=\{r\in [m]:\forall 
    w\in \mathbb{R}^d \text{ s.t. } & ~ \|w-w_r(0)\|_2\leq R,\\
    & ~ \mathbf{1}_{w_r(0)^\top x_i\geq 0}=\mathbf{1}_{w^\top x_i\geq 0}\}.
\end{align*}
Then we define $v_{1,i}$ and $v_{2,i}$ as follows

\begin{align*}
v_{1,i} : = & ~ \frac{1}{ \sqrt{m} } \sum_{r \in S_i} a_r \left( \phi \left( \Big( w_r(k) - \eta \frac{ \partial L(W(k)) }{ \partial w_r(k) } \Big)^\top x_i \right) - \phi( w_r(k)^\top x_i ) \right), \\
v_{2,i} : = & ~ \frac{1}{ \sqrt{m} } \sum_{r \in \ov{S}_i} a_r \left( \phi \left( \Big( w_r(k) - \eta \frac{ \partial L(W(k)) }{ \partial w_r(k) } \Big)^\top x_i \right) - \phi( w_r(k)^\top x_i ) \right) .
\end{align*}

Define $H$ and $H^{\bot} \in \R^{n \times n}$ as
\begin{align*}
H(k)_{i,j} = & ~ \frac{1}{m} \sum_{r=1}^m x_i^\top x_j {\bf 1}_{ w_r(k)^\top x_i \geq 0, w_r(k)^\top x_j \geq 0 } , \\
H(k)^{\bot}_{i,j} = & ~ \frac{1}{m} \sum_{r\in \ov{S}_i} x_i^\top x_j {\bf 1}_{ w_r(k)^\top x_i \geq 0, w_r(k)^\top x_j \geq 0 } .
\end{align*}
and
\begin{align*}
C_1 = & ~ -2 \eta (y - u(k))^\top H(k) ( y - u(k) ) , \\
C_2 = & ~ 2 \eta ( y - u(k) )^\top H(k)^{\bot} ( y - u(k) ) , \\
C_3 = & ~ - 2 ( y - u(k) )^\top v_2 , \\
C_4 = & ~ \| u (k+1) - u(k) \|_2^2 . 
\end{align*}

Then we have (delayed the proof into Appendix \ref{sec:missing_proof})
\begin{claim}\label{cla:inductive_claim}
\begin{align*}
\| y - u(k+1) \|_2^2 = \| y - u(k) \|_2^2 + C_1 + C_2 + C_3 + C_4.
\end{align*}
\end{claim}

Applying Claim~\ref{cla:C1}, \ref{cla:C2}, \ref{cla:C3} and \ref{cla:C4} gives
\begin{align*}
\| y - u(k+1) \|_2^2 
\leq & ~ \| y - u(k) \|_2^2 \\
 & ~ \cdot ( 1 - \eta \lambda + 8 \eta n R  + 8 \eta n R  + \eta^2 n^2 ).
\end{align*}

\paragraph{Choice of $\eta$ and $R$.}

Next, we want to choose $\eta$ and $R$ such that
\begin{equation}\label{eq:choice_of_eta_R}
( 1 - \eta \lambda + 8 \eta n R  + 8 \eta n R  + \eta^2 n^2 ) \leq (1-\eta\lambda/2) .
\end{equation}

If we set $\eta=\frac{\lambda }{4n^2}$ and $R=\frac{\lambda}{64n}$, we have 
\begin{align*}
8 \eta n R  + 8 \eta n R =16\eta nR \leq  \eta \lambda /4 ,
\mathrm{~~~and~~~} \eta^2 n^2 \leq  \eta \lambda / 4.
\end{align*}
This implies
\begin{align*}
\| y - u(k+1) \|_2^2 \leq & ~ \| y - u(k) \|_2^2 \cdot ( 1 - \eta \lambda / 2 )
\end{align*}
holds with probability at least $1-3n^2\exp(-mR/10)$.

\paragraph{Over-parameterization size, lower bound on $m$.}

We require 
\begin{align*}
 D=  \frac{4\sqrt{n}\|y-u(0)\|_2}{\sqrt{m}\lambda} < R = \frac{\lambda}{64n} ,
\text{~and~}  3n^2\exp(-mR/10)\leq  \delta .
\end{align*}
By Claim \ref{cla:yu0},
 it is sufficient to choose $m = \Omega( \lambda^{-4} n^4 \log(m/\delta)\log^2(n/\delta) )$.

\subsection{Technical Claims}
 \begin{claim}\label{cla:yu0}
For $0<\delta<1$,
with probability at least $1-\delta$,
\begin{align*}
\|y-u(0)\|_2^2=O(n\log(m/\delta)\log^2(n/\delta)).
\end{align*}
\end{claim}
The proof of Claim \ref{cla:yu0} is deferred to Appendix \ref{sec:proof_yu0}.

\begin{claim}\label{cla:C1}
Let $C_1 = -2 \eta (y - u(k))^\top H(k) ( y - u(k) )$. We have
\begin{align*}
C_1 \leq - \eta \lambda\cdot \| y - u(k) \|_2^2 
\end{align*}
holds with probability at least $1-n^2 \cdot \exp(-m R /10)$.
\end{claim}
The proof is in Appendix \ref{sec:proof_c1}.

\begin{claim}\label{cla:C2}
Let $C_2 = 2 \eta ( y - u(k) )^\top H(k)^{\bot} ( y - u(k) )$. We have
\begin{align*}
C_2 \leq 8\eta nR\cdot \| y - u(k) \|_2^2
\end{align*}
holds with probability $1-n\cdot \exp(-mR)$.
\end{claim}

The proof is in Appendix \ref{sec:proof_c2}.

\begin{claim}\label{cla:C3}
Let $C_3 = - 2 (y - u(k))^\top v_2$. Then we have
\begin{align*}
C_3 \leq 8 \eta nR\cdot \| y - u(k) \|_2^2  .
\end{align*}
with probability at least $1-n\cdot \exp(-mR)$.
\end{claim}
The proof is in Appendix \ref{sec:proof_c3}

\begin{claim}\label{cla:C4}
Let $C_4  = \| u (k+1) - u(k) \|_2^2$. Then we have
\begin{align*}
C_4 \leq \eta^2 n^2 \cdot \| y - u(k) \|_2^2.
\end{align*}
\end{claim}
The proof is in Appendix \ref{sec:proof_c4}
\section{Conclusion}

In this paper we improve the over-parametrization bound for two-layer neural networks trained by gradient descent with random initialization from two aspects:
first we improve the dependency of failure probability $\delta$ in the size bound from $\poly(1/\delta)$ to $\poly(\log (1/\delta))$;
second we lower the exponent on number of input data $n$,
showing that it can be as small as $n^2$ when input data have good properties.
We also study the training speed and generalization of two-layer neural networks,
and improve the exponent on $n$ and the dependency of failure probability $\delta$.



\ifdefined\isarxivversion
\section*{Acknowledgments}
The authors would like to thank Sanjeev Arora, Zeyuan Allen-Zhu, Simon S. Du, Rasmus Kyng, Jason D. Lee, Yin Tat Lee, Xingguo Li, Yuanzhi Li, Yingyu Liang, Zheng Yu, and Yi Zhang for useful discussions.
\else

\fi

\bibliographystyle{alpha}
\bibliography{ref}

\newcommand{\etalchar}[1]{$^{#1}$}
\begin{thebibliography}{AZLS19b}

\bibitem[AAA{\etalchar{+}}16]{aaab+16}
Dario Amodei, Sundaram Ananthanarayanan, Rishita Anubhai, Jingliang Bai, Eric
  Battenberg, Carl Case, Jared Casper, Bryan Catanzaro, Qiang Cheng, Guoliang
  Chen, et~al.
\newblock Deep speech 2: End-to-end speech recognition in english and mandarin.
\newblock In {\em ICML}, pages 173--182, 2016.

\bibitem[ADH{\etalchar{+}}19]{adhlw19}
Sanjeev Arora, Simon Du, Wei Hu, Zhiyuan Li, and Ruosong Wang.
\newblock Fine-grained analysis of optimization and generalization for
  overparameterized two-layer neural networks.
\newblock In {\em International Conference on Machine Learning}, pages
  322--332, 2019.

\bibitem[AW02]{aw02}
Rudolf Ahlswede and Andreas Winter.
\newblock Strong converse for identification via quantum channels.
\newblock {\em ITIT}, 48(3):569--579, 2002.

\bibitem[AZLL19]{all18}
Zeyuan Allen-Zhu, Yuanzhi Li, and Yingyu Liang.
\newblock Learning and generalization in overparameterized neural networks,
  going beyond two layers.
\newblock In {\em NeurIPS}. \url{https://arxiv.org/pdf/1811.04918.pdf}, 2019.

\bibitem[AZLS19a]{als19}
Zeyuan Allen-Zhu, Yuanzhi Li, and Zhao Song.
\newblock A convergence theory for deep learning via over-parameterization.
\newblock In {\em ICML}. \url{https://arxiv.org/pdf/1811.03962}, 2019.

\bibitem[AZLS19b]{als18}
Zeyuan Allen-Zhu, Yuanzhi Li, and Zhao Song.
\newblock On the convergence rate of training recurrent neural networks.
\newblock In {\em NeurIPS}. \url{https://arxiv.org/pdf/1810.12065}, 2019.

\bibitem[Ber24]{b24}
Sergei Bernstein.
\newblock On a modification of chebyshev's inequality and of the error formula
  of laplace.
\newblock {\em Ann. Sci. Inst. Sav. Ukraine, Sect. Math}, 1(4):38--49, 1924.

\bibitem[BG17]{bg17}
Alon Brutzkus and Amir Globerson.
\newblock Globally optimal gradient descent for a convnet with gaussian inputs.
\newblock In {\em ICML}, 2017.

\bibitem[BGMS18]{bgms18}
Alon Brutzkus, Amir Globerson, Eran Malach, and Shai Shalev{-}Shwartz.
\newblock {SGD} learns over-parameterized networks that provably generalize on
  linearly separable data.
\newblock In {\em 6th International Conference on Learning Representations,
  {ICLR} 2018, Vancouver, BC, Canada, April 30 - May 3, 2018, Conference Track
  Proceedings}. OpenReview.net, 2018.

\bibitem[BJW19]{brw19}
Ainesh Bakshi, Rajesh Jayaram, and David~P Woodruff.
\newblock Learning two layer rectified neural networks in polynomial time.
\newblock In {\em COLT}. \url{http://arxiv.org/pdf/:1811.01885}, 2019.

\bibitem[CB18]{cb18b}
Lenaic Chizat and Francis Bach.
\newblock A note on lazy training in supervised differentiable programming.
\newblock {\em arXiv preprint arXiv:1812.07956}, 2018.

\bibitem[Che52]{c52}
Herman Chernoff.
\newblock A measure of asymptotic efficiency for tests of a hypothesis based on
  the sum of observations.
\newblock {\em The Annals of Mathematical Statistics}, pages 493--507, 1952.

\bibitem[Dan17]{d17}
Amit Daniely.
\newblock Sgd learns the conjugate kernel class of the network.
\newblock In {\em Advances in Neural Information Processing Systems}, pages
  2422--2430, 2017.

\bibitem[DDS{\etalchar{+}}09]{ddsllf09}
Jia Deng, Wei Dong, Richard Socher, Li-Jia Li, Kai Li, and Li~Fei-Fei.
\newblock Imagenet: A large-scale hierarchical image database.
\newblock In {\em 2009 IEEE conference on computer vision and pattern
  recognition}, pages 248--255. Ieee, 2009.

\bibitem[DH19]{dh19}
Simon~S Du and Wei Hu.
\newblock Width provably matters in optimization for deep linear neural
  networks.
\newblock {\em arXiv preprint arXiv:1901.08572}, 2019.

\bibitem[DLL{\etalchar{+}}19]{dllwz19}
Simon~S Du, Jason~D Lee, Haochuan Li, Liwei Wang, and Xiyu Zhai.
\newblock Gradient descent finds global minima of deep neural networks.
\newblock In {\em ICML}. \url{https://arxiv.org/pdf/1811.03804}, 2019.

\bibitem[DLT{\etalchar{+}}18]{dltps18}
Simon~S. Du, Jason~D. Lee, Yuandong Tian, Barnab{\'{a}}s P{\'{o}}czos, and
  Aarti Singh.
\newblock Gradient descent learns one-hidden-layer {CNN:} don't be afraid of
  spurious local minima.
\newblock In {\em ICML}. \url{http://arxiv.org/pdf/1712.00779}, 2018.

\bibitem[DZPS19]{dzps19}
Simon~S Du, Xiyu Zhai, Barnabas Poczos, and Aarti Singh.
\newblock Gradient descent provably optimizes over-parameterized neural
  networks.
\newblock In {\em ICLR}. \url{https://arxiv.org/pdf/1810.02054}, 2019.

\bibitem[GLM18]{glm18}
Rong Ge, Jason~D. Lee, and Tengyu Ma.
\newblock Learning one-hidden-layer neural networks with landscape design.
\newblock In {\em ICLR}, 2018.

\bibitem[GLSS18]{glss18}
Ankit Garg, Yin-Tat Lee, Zhao Song, and Nikhil Srivastava.
\newblock A matrix expander chernoff bound.
\newblock In {\em STOC}. \url{https://arxiv.org/pdf/1704.03864}, 2018.

\bibitem[GMH13]{gmh13}
Alex Graves, Abdel-rahman Mohamed, and Geoffrey Hinton.
\newblock Speech recognition with deep recurrent neural networks.
\newblock In {\em 2013 IEEE international conference on acoustics, speech and
  signal processing}, pages 6645--6649. IEEE, 2013.

\bibitem[Hoe63]{h63}
Wassily Hoeffding.
\newblock Probability inequalities for sums of bounded random variables.
\newblock {\em Journal of the American Statistical Association},
  58(301):13--30, 1963.

\bibitem[HZRS16]{hzrs13}
Kaiming He, Xiangyu Zhang, Shaoqing Ren, and Jian Sun.
\newblock Deep residual learning for image recognition.
\newblock In {\em CVPR}, pages 770--778, 2016.

\bibitem[JGH18]{jgh18}
Arthur Jacot, Franck Gabriel, and Cl{\'e}ment Hongler.
\newblock Neural tangent kernel: Convergence and generalization in neural
  networks.
\newblock In {\em Advances in neural information processing systems}, pages
  8571--8580, 2018.

\bibitem[KS18]{ks18}
Rasmus Kyng and Zhao Song.
\newblock A matrix chernoff bound for strongly rayleigh distributions and
  spectral sparsifiers from a few random spanning trees.
\newblock In {\em FOCS}. \url{https://arxiv.org/pdf/1810.08345}, 2018.

\bibitem[KSH12]{ksh12}
Alex Krizhevsky, Ilya Sutskever, and Geoffrey~E Hinton.
\newblock Imagenet classification with deep convolutional neural networks.
\newblock In {\em NeurIPS}, pages 1097--1105, 2012.

\bibitem[LBBH98]{lbbh98}
Yann LeCun, L{\'e}on Bottou, Yoshua Bengio, and Patrick Haffner.
\newblock Gradient-based learning applied to document recognition.
\newblock {\em Proceedings of the IEEE}, 86(11):2278--2324, 1998.

\bibitem[Lee18]{l18}
Yin~Tat Lee.
\newblock Personal communication.
\newblock {\em .}, 2018.

\bibitem[LL18]{ll18}
Yuanzhi Li and Yingyu Liang.
\newblock Learning overparameterized neural networks via stochastic gradient
  descent on structured data.
\newblock In {\em NeurIPS}, 2018.

\bibitem[LSSS14]{lss14}
Roi Livni, Shai Shalev-Shwartz, and Ohad Shamir.
\newblock On the computational efficiency of training neural networks.
\newblock In {\em Advances in neural information processing systems}, pages
  855--863, 2014.

\bibitem[LY17]{ly17}
Yuanzhi Li and Yang Yuan.
\newblock Convergence analysis of two-layer neural networks with {R}e{LU}
  activation.
\newblock In {\em NeurIPS}. \url{http://arxiv.org/pdf/1705.09886}, 2017.

\bibitem[MMN18]{mmn18}
Song Mei, Andrea Montanari, and Phan-Minh Nguyen.
\newblock A mean field view of the landscape of two-layer neural networks.
\newblock {\em Proceedings of the National Academy of Sciences},
  115(33):E7665--E7671, 2018.

\bibitem[Rud99]{r99}
Mark Rudelson.
\newblock Random vectors in the isotropic position.
\newblock {\em Journal of Functional Analysis}, 164(1):60--72, 1999.

\bibitem[SGS15]{sgs15}
Rupesh~K Srivastava, Klaus Greff, and J{\"u}rgen Schmidhuber.
\newblock Training very deep networks.
\newblock In {\em Advances in neural information processing systems}, pages
  2377--2385, 2015.

\bibitem[SHM{\etalchar{+}}16]{shmc+16}
David Silver, Aja Huang, Chris~J Maddison, Arthur Guez, Laurent Sifre, George
  Van Den~Driessche, Julian Schrittwieser, Ioannis Antonoglou, Veda
  Panneershelvam, Marc Lanctot, et~al.
\newblock Mastering the game of go with deep neural networks and tree search.
\newblock {\em nature}, 529(7587):484, 2016.

\bibitem[Sol17]{s17}
Mahdi Soltanolkotabi.
\newblock Learning {R}e{LU}s via gradient descent.
\newblock In {\em arXiv preprint}. \url{http://arxiv.org/pdf/1705.04591}, 2017.

\bibitem[SSS{\etalchar{+}}17]{sssa+17}
David Silver, Julian Schrittwieser, Karen Simonyan, Ioannis Antonoglou, Aja
  Huang, Arthur Guez, Thomas Hubert, Lucas Baker, Matthew Lai, Adrian Bolton,
  et~al.
\newblock Mastering the game of go without human knowledge.
\newblock {\em Nature}, 550(7676):354, 2017.

\bibitem[Tia17]{t17}
Yuandong Tian.
\newblock An analytical formula of population gradient for two-layered {R}e{LU}
  network and its applications in convergence and critical point analysis.
\newblock In {\em ICML}. \url{http://arxiv.org/pdf/1703.00560}, 2017.

\bibitem[Tro12]{t12}
Joel~A Tropp.
\newblock User-friendly tail bounds for sums of random matrices.
\newblock {\em Foundations of computational mathematics}, 12(4):389--434, 2012.

\bibitem[Tro15]{t15}
Joel~A Tropp.
\newblock An introduction to matrix concentration inequalities.
\newblock {\em Foundations and Trends{\textregistered} in Machine Learning},
  8(1-2):1--230, 2015.

\bibitem[ZBH{\etalchar{+}}17]{zbhrv16}
Chiyuan Zhang, Samy Bengio, Moritz Hardt, Benjamin Recht, and Oriol Vinyals.
\newblock Understanding deep learning requires rethinking generalization.
\newblock {\em ICLR}, 2017.

\bibitem[ZK16]{zk16}
Sergey Zagoruyko and Nikos Komodakis.
\newblock Wide residual networks.
\newblock {\em arXiv preprint arXiv:1605.07146}, 2016.

\bibitem[ZSD17]{zsd17}
Kai Zhong, Zhao Song, and Inderjit~S Dhillon.
\newblock Learning non-overlapping convolutional neural networks with multiple
  kernels.
\newblock In {\em arXiv preprint}. \url{https://arxiv.org/pdf/1711.03440},
  2017.

\bibitem[ZSJ{\etalchar{+}}17]{zsjbd17}
Kai Zhong, Zhao Song, Prateek Jain, Peter~L. Bartlett, and Inderjit~S. Dhillon.
\newblock Recovery guarantees for one-hidden-layer neural networks.
\newblock In {\em ICML}, 2017.

\end{thebibliography}

\newpage
\onecolumn
\appendix
\section{Probability Tools}\label{sec:probability_tools}


In this section we introduce the probability tools we use in the proof.
Lemma \ref{lem:chernoff}, \ref{lem:hoeffding} and \ref{lem:bernstein} are about tail bounds for random scalar variables.
Lemma \ref{lem:anti_gaussian} is about cdf of Gaussian distributions.
Finally,
Lemma \ref{lem:matrix_bernstein} is a concentration result on random matrices.

\begin{lemma}[Chernoff bound \cite{c52}]\label{lem:chernoff}
Let $X = \sum_{i=1}^n X_i$, where $X_i=1$ with probability $p_i$ and $X_i = 0$ with probability $1-p_i$, and all $X_i$ are independent. Let $\mu = \E[X] = \sum_{i=1}^n p_i$. Then \\
1. $ \Pr[ X \geq (1+\delta) \mu ] \leq \exp ( - \delta^2 \mu / 3 ) $, $\forall \delta > 0$ ; \\
2. $ \Pr[ X \leq (1-\delta) \mu ] \leq \exp ( - \delta^2 \mu / 2 ) $, $\forall 0 < \delta < 1$. 
\end{lemma}

\begin{lemma}[Hoeffding bound \cite{h63}]\label{lem:hoeffding}
Let $X_1, \cdots, X_n$ denote $n$ independent bounded variables in $[a_i,b_i]$. Let $X= \sum_{i=1}^n X_i$, then we have
\begin{align*}
\Pr[ | X - \E[X] | \geq t ] \leq 2\exp \left( - \frac{2t^2}{ \sum_{i=1}^n (b_i - a_i)^2 } \right).
\end{align*}
\end{lemma}

\begin{lemma}[Bernstein inequality \cite{b24}]\label{lem:bernstein}
Let $X_1, \cdots, X_n$ be independent zero-mean random variables. Suppose that $|X_i| \leq M$ almost surely, for all $i$. Then, for all positive $t$,
\begin{align*}
\Pr \left[ \sum_{i=1}^n X_i > t \right] \leq \exp \left( - \frac{ t^2/2 }{ \sum_{j=1}^n \E[X_j^2]  + M t /3 } \right).
\end{align*}
\end{lemma}

\begin{lemma}[Anti-concentration of Gaussian distribution]\label{lem:anti_gaussian}
Let $X\sim {\N}(0,\sigma^2)$,
that is,
the probability density function of $X$ is given by $\phi(x)=\frac 1 {\sqrt{2\pi\sigma^2}}e^{-\frac {x^2} {2\sigma^2} }$.
Then
\begin{align*}
    \Pr[|X|\leq t]\in \left( \frac 2 3\frac t \sigma, \frac 4 5\frac t \sigma \right).
\end{align*}
\end{lemma}

\begin{lemma}[Matrix Bernstein, Theorem 6.1.1 in \cite{t15}]\label{lem:matrix_bernstein}
Consider a finite sequence $\{ X_1, \cdots, X_m \} \subset \R^{n_1 \times n_2}$ of independent, random matrices with common dimension $n_1 \times n_2$. Assume that
\begin{align*}
\E[ X_i ] = 0, \forall i \in [m] ~~~ \mathrm{and}~~~ \| X_i \| \leq M, \forall i \in [m] .
\end{align*}
Let $Z = \sum_{i=1}^m X_i$. Let $\mathrm{Var}[Z]$ be the matrix variance statistic of sum:
\begin{align*}
\mathrm{Var} [Z] = \max \left\{ \Big\| \sum_{i=1}^m \E[ X_i X_i^\top ] \Big\| , \Big\| \sum_{i=1}^m \E [ X_i^\top X_i ] \Big\| \right\}.
\end{align*}
Then 

\begin{align*}
\E[ \| Z \| ] \leq ( 2 \mathrm{Var} [Z] \cdot \log (n_1 + n_2) )^{1/2} +  M \cdot \log (n_1 + n_2) / 3.
\end{align*}

Furthermore, for all $t \geq 0$,
\begin{align*}
\Pr[ \| Z \| \geq t ] \leq (n_1 + n_2) \cdot \exp \left( - \frac{t^2/2}{ \mathrm{Var} [Z] + M t /3 }  \right)  .
\end{align*}
\end{lemma}
\section{Synthetic Examples }\label{sec:example}

In this section we check some synthetic examples to validate Assumption \ref{ass:data_dependent_assumption}.

Our first example is a very trivial one, where all the data points are unit vectors and are orthogonal to each other.
This is the best separable case we can hope for.
Notice that in this case we must have $d\geq n$.
In this case,
we have
$H^{\cts}=\frac {1}{2} I_n$.
Therefore,
we have\\
1. $\lambda=\lambda_{\min}(H^{\cts})=1/2$.\\
2. For $i\in [n]$, let $y_i=\mathbf{1}_{x_i^\top w\geq 0}$. Then 
\begin{align*}
\Big\| H(w) - \E_{w}[H(w)] \Big\|=\max_{i\in [n]} y_i-1/2.
\end{align*}
So we can set $\alpha=1/2$ and $\gamma=0$.\\
3. Since 
\begin{align*}
(H(w) - \E_{w}[H(w)] )(H(w) - \E_{w}[H(w)] )^\top=\begin{pmatrix}
(y_1-1/2)^2 & & &\\
&(y_2-1/2)^2 & &\\
& &\ddots & \\
& & &(y_n-1/2)^2
\end{pmatrix},
\end{align*}
we can set $\beta=1/4$.\\
4. Since data points are mutually orthogonal, we can set $\theta=0$.

Our second example is all the data points are i.i.d normalized random Gaussian vectors in $\mathbb{R}^d$.
That is, for $i\in [n]$, $x_i\sim {\N}(0,I_d)$.
Therefore the $(i,j)$-th entry of $\E_w[H(w)]$ is simply 
\begin{align*}
x_i^\top x_j \cdot \frac{\pi-\arccos(x_i^\top x_j)}{2\pi}.
\end{align*}
We perform 2 numerical experiments to validate Assumption \ref{ass:data_dependent_assumption}.
For part 1 and part 4 of Assumption \ref{ass:data_dependent_assumption},
we set $d=500$,
and for $i=1,\cdots,20$,
we set $n=50i$ and compute the corresonding $\lambda$ and $\theta$.
The experimental result can be found in Figure \ref{fig:exp_1}.
We can see that though $\lambda$ decreases as $n$ increases, $\lambda$ is indeed positive.
Also,
when $n$ is not too large compared to $d$,
$\theta$ is relatively small compared to the maximal possible value $\sqrt{n}$.

\begin{figure*}[!t]
  \centering
  	\begin{minipage}{0.45\textwidth}
  		\includegraphics[width=\textwidth]{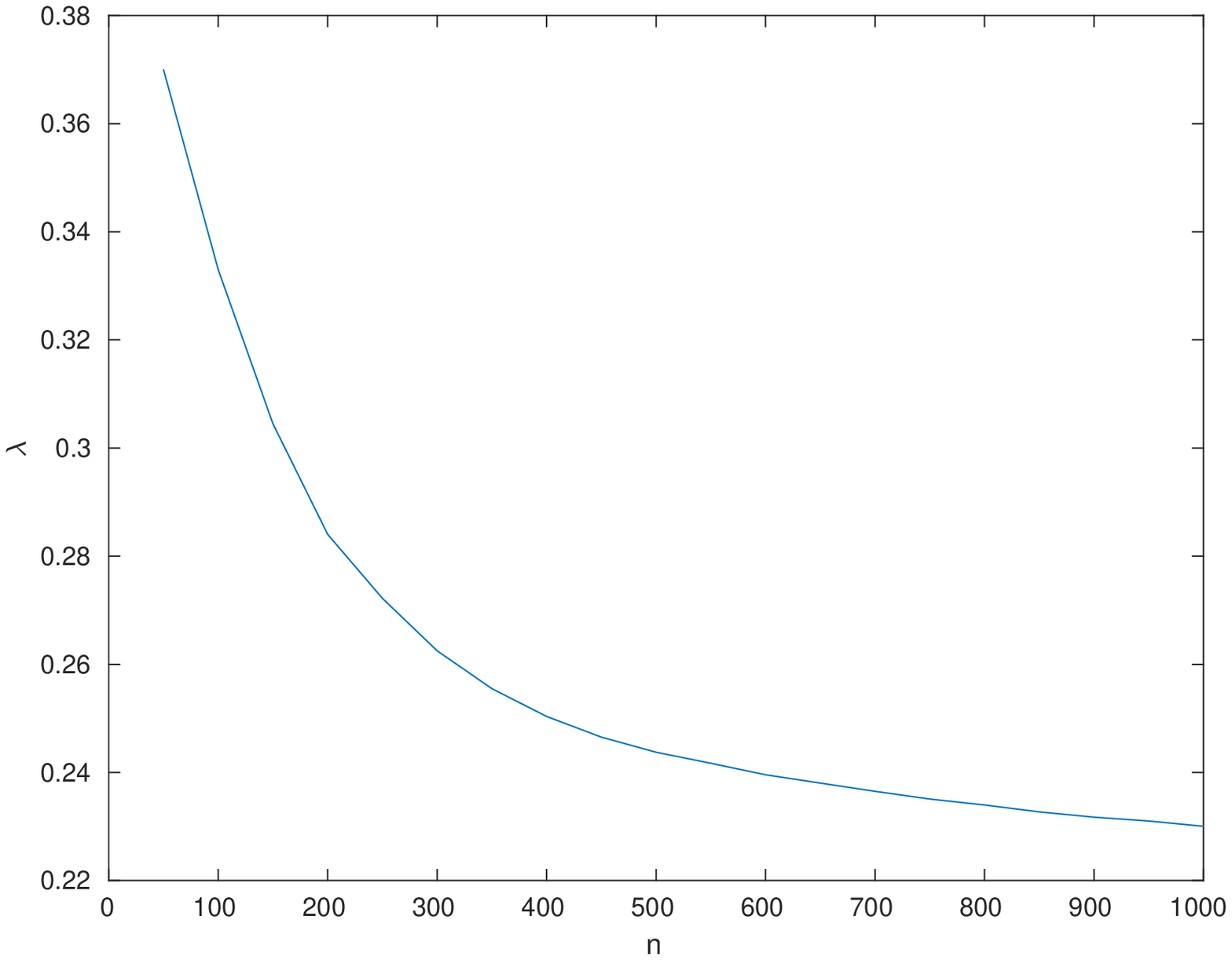}
  	\end{minipage}
  	\hspace{0.05\textwidth}
  	\begin{minipage}{0.45\textwidth}
  		\includegraphics[width=\textwidth]{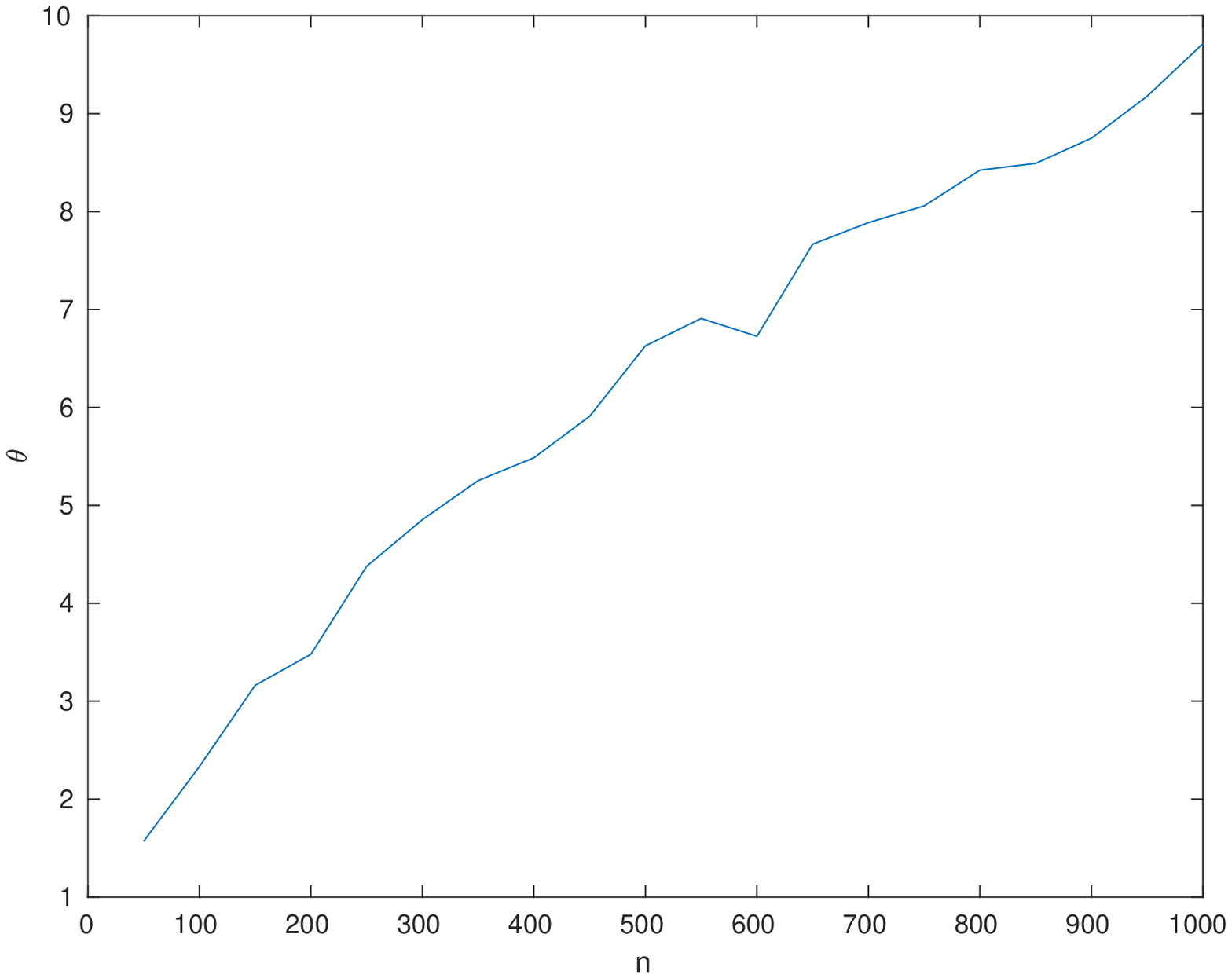}
  	\end{minipage}

    \caption{\small Minimal eigenvalue and $\theta$ of the training data. In all experiments $d=500$. For $i\in [n]$, $x_i$ is drawn i.i.d from $\N(0,I_d)$ and then normalized.
    (Left) The $y$-axis $\lambda=\lambda_{\min}(\E_w[H(w)])$; the $x$-axis is the number samples $n$. (Right) The $y$-axis $\theta=\sqrt{n}\cdot \max_{i\neq j,i,j\in [n]}x_i^\top x_j$; the $x$-axis is the number samples $n$.}
    \label{fig:exp_1}
\end{figure*}

For part 2 and part 3 of Assumption \ref{ass:data_dependent_assumption},
we set $n=100$ and $d=20$,
and take 1000 random Gaussian weights $w$ to plot the distribution of $\|H(w) - \E_w[H(w)]\|$ and $\|(H(w) - \E_w[H(w)])(H(w) - \E_w[H(w)])^\top\|$.
The result can be found in Figure \ref{fig:exp_2}.
We can see that the distribution of both quantities are concentrated;
moreover,
the maximal value of $\|H(w) - \E_w[H(w)]\|$ is no more than $6$, which is much smaller than the maximal possible value $n=100$.
Similarly the maximal value of $\|(H(w) - \E_w[H(w)])(H(w) - \E_w[H(w)])^\top\|$ is also much smaller than the maximal allowed value $n^2$.
Hence in this case we shall expect $\alpha\ll n$ and $\beta\ll n^2$.

\begin{figure*}[!t]
  \centering
  	\begin{minipage}{0.45\textwidth}
  		\includegraphics[width=\textwidth]{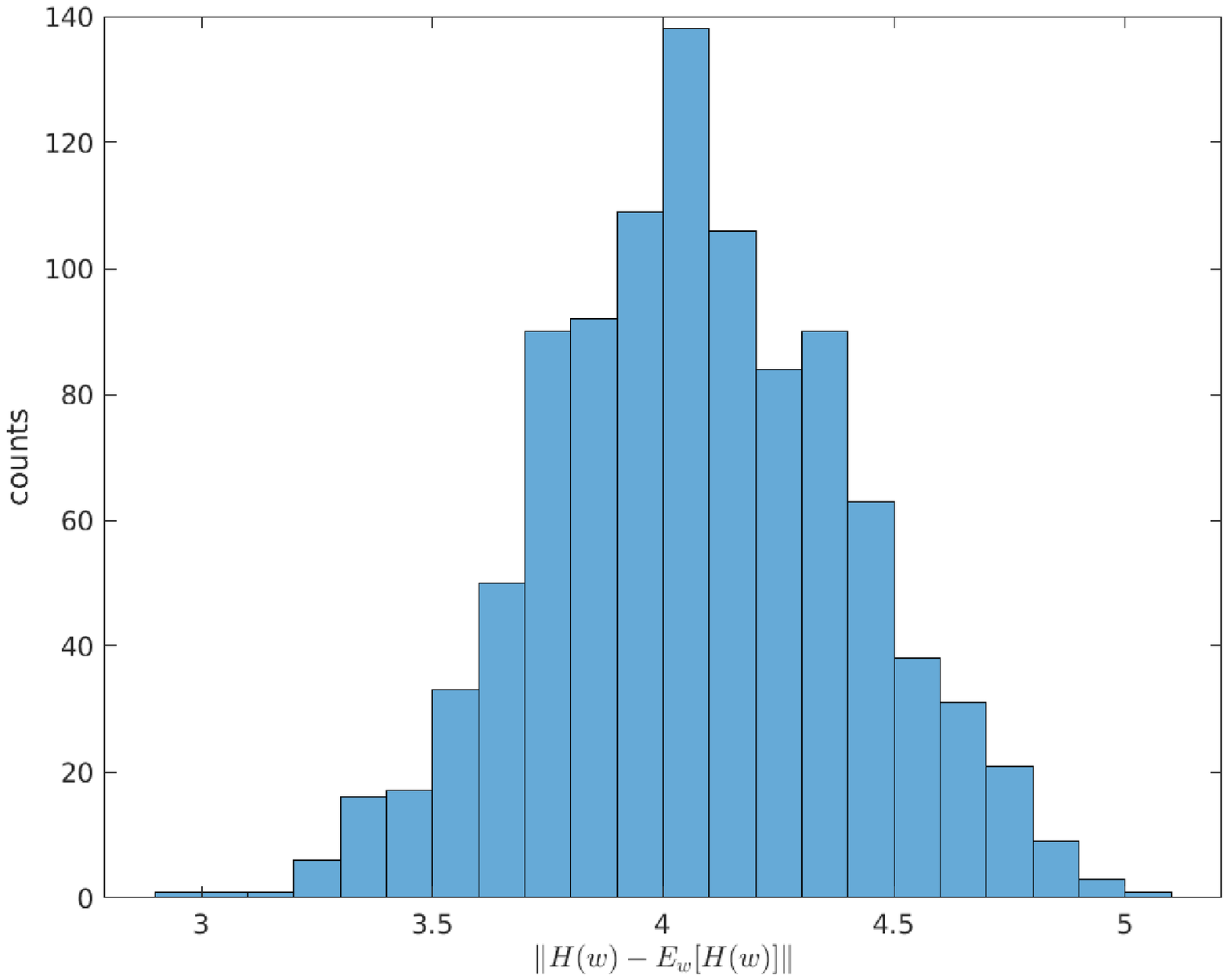}
  	\end{minipage}
  	\hspace{0.05\textwidth}
  	\begin{minipage}{0.45\textwidth}
  		\includegraphics[width=\textwidth]{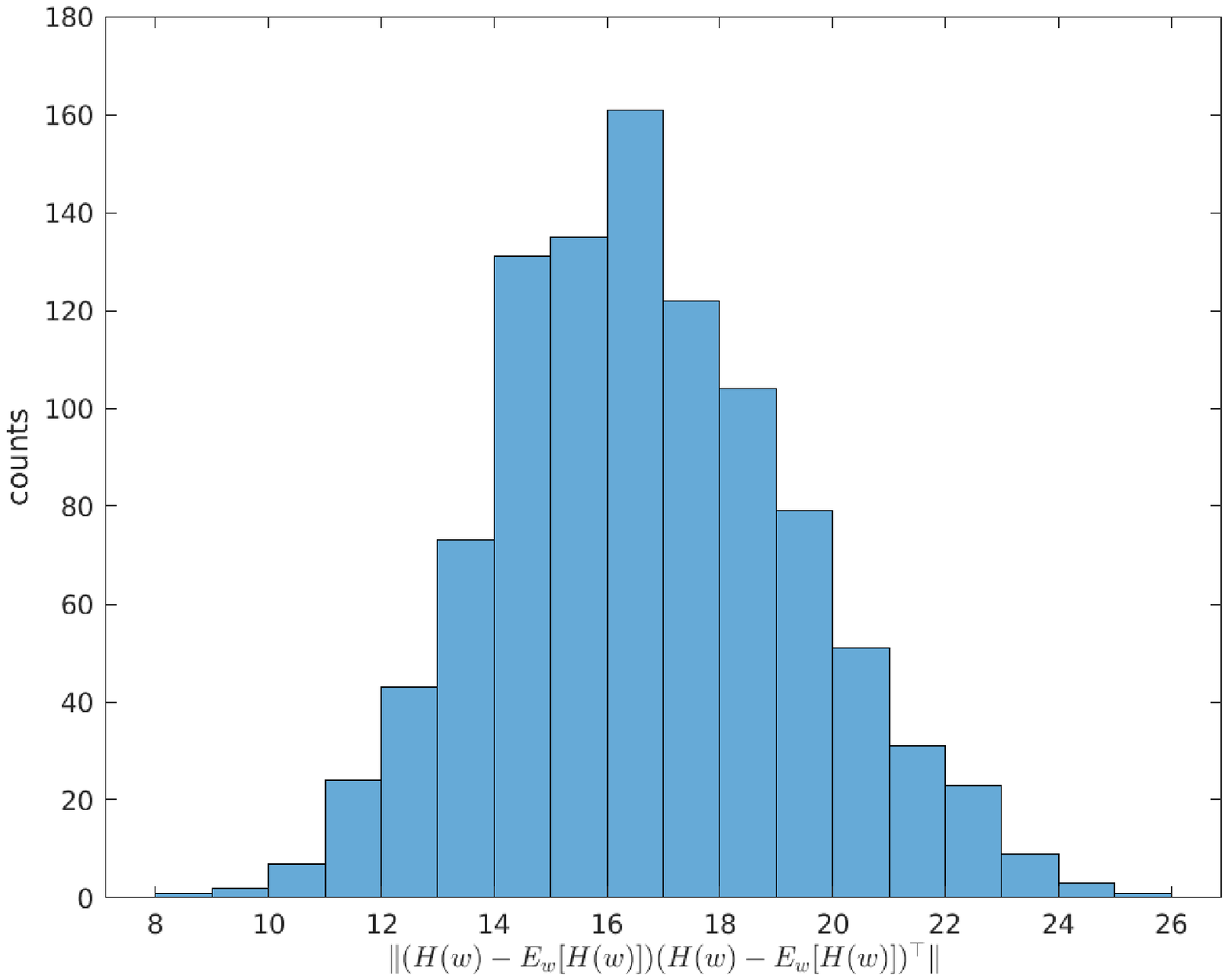}
  	\end{minipage}

    \caption{\small Distributions of $\|H(w) - \E_w[H(w)]\|$ and $\|(H(w) - \E_w[H(w)])(H(w) - \E_w[H(w)])^\top\|$. In all experiments $n=100$ and $d=20$. For $i\in [n]$, $x_i$ is drawn i.i.d from $\N(0,I_d)$ and then normalized.
    For $r=1,\cdots,1000$, $w_r$ is drawn i.i.d from $\N(0,I_d)$.
    (Left) The $x$-axis is the value of $\|H(w) - \E_w[H(w)]\|$. The $y$-axis is the counts for each value of $\|H(w) - \E_w[H(w)]\|$. (Right) The $x$-axis is the value of $\|(H(w) - \E_w[H(w)])(H(w) - \E_w[H(w)])^\top\|$. The $y$-axis is the counts for each value of $\|(H(w) - \E_w[H(w)])(H(w) - \E_w[H(w)])^\top\|$.}
    \label{fig:exp_2}
\end{figure*}

\section{Technical claims (Missing proofs from Section~\ref{sec:quartic_suffices})}\label{sec:missing_proof}

\subsection{Proof of Lemma \ref{lem:3.1}}
For the completeness, we provide a proof of Lemma \ref{lem:3.1} here.
\begin{proof}[Proof of Lemma \ref{lem:3.1}]
For every fixed pair $(i,j)$,
$H_{i,j}^{\dis}$ is an average of independent random variables,
i.e.
\begin{align*}
H_{i,j}^{\dis}=~\frac {1}{m}\sum_{r=1}^m x_i^\top x_j\mathbf{1}_{w_r^\top x_i\geq 0,w_r^\top x_j\geq 0}.
\end{align*}
Then the expectation of $H_{i,j}^{\dis}$ is
\begin{align*}
\E [ H_{i,j}^{\dis} ]
= & ~\frac {1}{m}\sum_{r=1}^m \E_{w_r\sim {\N}(0,I_d)} \left[ x_i^\top x_j\mathbf{1}_{w_r^\top x_i\geq 0,w_r^\top x_j\geq 0} \right]\\
= & ~\E_{w\sim {\N}(0,I_d)} \left[ x_i^\top x_j\mathbf{1}_{w^\top x_i\geq 0,w^\top x_j\geq 0} \right]\\
= & ~ H_{i,j}^{\cts}.
\end{align*}
For $r\in [m]$,
let $z_r=\frac {1}{m}x_i^\top x_j \mathbf{1}_{w_r ^\top x_i\geq 0,w_r^\top x_j\geq 0}$.
Then $z_r$ is a random function of $w_r$,
hence $\{z_r\}_{r\in [m]}$ are mutually independent.
Moreover,
$-\frac {1}{m}\leq z_r\leq \frac {1}{m}$.
So by Hoeffding inequality(Lemma \ref{lem:hoeffding}) we have for all $t>0$,
\begin{align*}
\Pr \left[ | H_{i,j}^{\dis} - H_{i,j}^{\cts} | \geq t \right]
\leq & ~ 2\exp \Big( -\frac{2t^2}{4/m} \Big) \\
 = & ~ 2\exp(-mt^2/2).
\end{align*}
Setting $t=( \frac{1}{m} 2 \log (2n^2/\delta) )^{1/2}$,
we can apply union bound on all pairs $(i,j)$ to get with probability at least $1-\delta$,
for all $i,j\in [n]$,
\begin{align*}
|H_{i,j}^{\dis} - H_{i,j}^{\cts}|
\leq \Big( \frac{2}{m}\log (2n^2/\delta) \Big)^{1/2}
\leq 4 \Big( \frac{\log ( n/\delta ) }{m} \Big)^{1/2}.
\end{align*}
Thus we have
\begin{align*}
\|H^{\dis} - H^{\cts}\|^2 
\leq & ~ \|H^{\dis} - H^{\cts}\|_F^2 \\
 = & ~ \sum_{i=1}^n\sum_{j=1}^n |H_{i,j}^{\dis} - H_{i,j}^{\cts}|^2 \\
 \leq & ~ \frac{1}{m} 16n^2\log (n/\delta).
\end{align*}
Hence if $m=\Omega( \lambda^{-2} n^2\log (n/\delta) )$ we have the desired result.
 \end{proof}

\subsection{Proof of Lemma \ref{lem:3.3}}
 \begin{proof}

Recall we can write the dynamics of predictions as 
\begin{align*}
\frac{ \d }{ \d t} u(t) = H(t)\cdot ( y - u(t) ) .
\end{align*}
We can calculate the loss function dynamics
\begin{align*}
\frac{\d }{ \d t } \| y - u(t) \|_2^2
= & ~ - 2 ( y - u(t) )^\top \cdot H(t) \cdot ( y - u(t) ) \\
\leq & ~ - \lambda \| y - u(t) \|_2^2 .
\end{align*}

Thus we have $\frac{\d}{ \d t} ( \exp(\lambda t) \| y - u(t) \|_2^2 ) \leq 0$ and $\exp( \lambda t ) \| y - u(t) \|_2^2$ is a decreasing function with respect to $t$.

Using this fact we can bound the loss
\begin{align}\label{eq:yut}
\| y - u(t) \|_2^2 \leq \exp( - \lambda t ) \| y - u(0) \|_2^2.
\end{align}


Now, we can bound the gradient norm. For $0 \leq s \leq t$,
\begin{align}\label{eq:gradient_bound}
& ~ \left\| \frac{ \d }{ \d s } w_r(s) \right\|_2 \notag\\
= & ~ \left\| \sum_{i=1}^n (y_i - u_i) \frac{1}{\sqrt{m}} a_r x_i \cdot {\bf 1}_{ w_r(s)^\top x_i \geq 0 } \right\|_2 \notag\\
\leq & ~ \frac{1}{ \sqrt{m} } \sum_{i=1}^n | y_i - u_i(s) | \notag\\
\leq & ~ \frac{ \sqrt{n} }{ \sqrt{m} } \| y - u(s) \|_2 \\
\leq & ~ \frac{ \sqrt{n} }{ \sqrt{m} } \exp( - \lambda s ) \| y - u(0) \|_2.\notag 
\end{align}
where the first step follows from Eq. \eqref{eq:gradient}, the second step follows from triangle inequality and $a_r=\pm 1$ for $r\in [m]$ and $\|x_i\|_2=1$ for $i\in [n]$, the third step follows from Cauchy-Schwartz inequality, and the last step follows from Eq. \eqref{eq:yut}.

Integrating the gradient, we can bound the distance from the initialization
\begin{align*}
\| w_r(t) - w_r(0) \|_2 \leq & ~ \int_0^t \left\| \frac{\d}{ \d s} w_r(s) \right\|_2 \d s \\
\leq & ~ \frac{ \sqrt{n} \| y - u(0) \|_2 }{ \sqrt{m} \lambda } .
\end{align*}
\end{proof}

\subsection{Proof of Lemma \ref{lem:3.4}}
\begin{proof}
Assume the conclusion does not hold at time $t$.
We argue there must be some $s\leq t$ so that $\lambda_{\min}( H ( s ) )<\frac {1}{2}\lambda$.

If $\lambda_{\min}( H ( t  ) )<\frac {1}{2}\lambda$, then we can simply take $s=t$.

Otherwise since the conclusion does not hold, there exists $r$ so that 
\begin{align*}
\|w_r(t)-w_r(0)\|\geq D_{\cts}
\end{align*}
or
\begin{align*}
\|y-u(t)\|_2^2>\exp(-\lambda t)\|y-u(0)\|_2^2.
\end{align*}

Then by Lemma \ref{lem:3.3}, there exists $s\leq t$ such that 
\begin{align*}
\lambda_{\min}( H ( s  ) )<\frac {1}{2}\lambda.
\end{align*}

By Lemma \ref{lem:3.2}, there exists $t_0> 0$ defined as
\begin{align*}
t_0=\inf \left\{ t>0 : \max_{r\in [m]} \|w_r(t)-w_r(0)\|_2^2\geq R \right\}.
\end{align*}

Thus at time $t_0$, there exists $r\in [m]$ satisfying $\|w_r(t_0)-w_r(0)\|_2^2=R$.

By Lemma \ref{lem:3.2},
\begin{align*}
\lambda_{\min}(H(t'))\geq \frac {1}{2}\lambda , \forall t'\leq t_0.
\end{align*}

However, by Lemma \ref{lem:3.3}, this implies 
\begin{align*}
\|w_r(t_0)-w_r(0)\|_2\leq D_{\cts}<R,
\end{align*}
which is a contradiction.
\end{proof}

\subsection{Proof of Lemma \ref{lem:4.1}}
\begin{proof}
We use the norm of gradient to bound this distance,
\begin{align*}
& ~ \| w_r(k+1) - w_r(0) \|_2 \\
\leq & ~ \eta \sum_{i=0}^k \left\| \frac{ \partial L( W(i) ) }{ \partial w_r(i) } \right\|_2 \\
\leq & ~ \eta \sum_{i=0}^k \frac{ \sqrt{n} \| y - u(i) \|_2 }{ \sqrt{m} } \\
\leq & ~ \eta \sum_{i=0}^k \frac{ \sqrt{n} ( 1 - \eta \lambda /2 )^{i/2} }{ \sqrt{m} } \| y - u(0) \|_2 \\
\leq & ~ \eta \sum_{i=0}^{\infty} \frac{ \sqrt{n} (1-\eta\lambda/2)^{i/2} }{ \sqrt{m} } \| y - u(0) \|_2  \\
= & ~ \frac{ 4 \sqrt{n} \| y - u(0) \|_2 }{ \sqrt{m} \lambda },
\end{align*}
where the first step follows from Eq. \eqref{eq:w_update}, the second step follows from Eq. \eqref{eq:gradient_bound}, the third step follows from the induction hypothesis, the fourth step relaxes the summation to an infinite summation, and the last step follows from $\sum_{i=0}^{\infty}(1-\eta\lambda/2)^{i/2}=\frac {2}{\eta\lambda}$.

Thus, we complete the proof.
\end{proof}

\subsection{Upper Bound of $\| H(k)^{\bot} \|_2 $}

\begin{fact}\label{fact:bound_H_k_bot}
\begin{align*}
\| H(k)^{\bot} \|_2 \leq \frac{n}{m^2} \sum_{i=1}^n y_i^2 .
\end{align*}
\end{fact}

\begin{proof}
We have
\begin{align*}
\| H(k)^{\bot} \|_F^2
= & ~ \sum_{i=1}^n\sum_{j=1}^n (H(k)^{\bot}_{i,j})^2\\
= & ~ \sum_{i=1}^n\sum_{j=1}^n \Big( \frac {1} {m}\sum_{r\in \ov{S}_i} x_i^\top x_j\mathbf{1}_{w_r(k)^\top x_i\geq 0,w_r(k)^\top x_j\geq 0} \Big)^2\\
= & ~ \sum_{i=1}^n\sum_{j=1}^n \Big( \frac {1} {m}\sum_{r=1}^m x_i^\top x_j\mathbf{1}_{w_r(k)^\top x_i\geq 0,w_r(k)^\top x_j\geq 0} \cdot \mathbf{1}_{r\in \ov{S}_i} \Big)^2\\
= & ~ \sum_{i=1}^n\sum_{j=1}^n ( \frac {x_i^\top x_j} {m} )^2 \Big( \sum_{r=1}^m \mathbf{1}_{w_r(k)^\top x_i\geq 0,w_r(k)^\top x_j\geq 0} \cdot \mathbf{1}_{r\in \ov{S}_i} \Big)^2 \\
\leq & ~ \frac{1}{m^2} \sum_{i=1}^n\sum_{j=1}^n \Big( \sum_{r=1}^m \mathbf{1}_{w_r(k)^\top x_i\geq 0,w_r(k)^\top x_j\geq 0} \cdot \mathbf{1}_{r\in \ov{S}_i} \Big)^2 \\
= & ~ \frac{n}{m^2} \sum_{i=1}^n \Big( \sum_{r=1}^m \mathbf{1}_{r\in \ov{S}_i} \Big)^2 \\
= & ~ \frac{n}{m^2} \sum_{i=1}^n y_i^2 .
\end{align*}
\end{proof}

\subsection{Proof of Fact \ref{fact:dudt}}
\begin{proof}
For each $i\in [n]$,
we have

\begin{align*}
&~\frac{\d}{\d t} u_i(t)\\
=&~\sum_{r=1}^m \left\langle \frac{\partial f(W(t),a,x_i)}{\partial w_r(t)},\frac{\d w_r(t)}{\d t} \right\rangle\\
=&~\sum_{r=1}^m \left\langle \frac{\partial f(W(t),a,x_i)}{\partial w_r(t)},-\frac{\partial L(w(t),a)}{\partial w_r(t)} \right\rangle\\
 =&~\sum_{r=1}^m \Big\langle \frac{\partial f(W(t),a,x_i)}{\partial w_r(t)},
 -\frac{1}{ \sqrt{m} } \sum_{i=1}^n ( f(W,x_i,a_r) - y_i ) a_r x_i {\bf 1}_{ w_r^\top x_i \geq 0 } \Big\rangle\\
=&~\sum_{j=1}^n (y_j-u_j(t)) \left\langle \frac{\partial f(W(t),a,x_i)}{\partial w_r(t)},\frac{\partial f(W(t),a,x_j)}{\partial w_r(t)} \right\rangle\\
=&~\sum_{j=1}^n (y_j-u_j(t))H(t)_{i,j}
\end{align*}
where the first step follows from Eq. \eqref{eq:ut_def} and the chain rule of derivatives,
the second step uses Eq. \eqref{eq:gradient},
the third step uses Eq. \eqref{eq:wr_derivative},
the fourth step uses Eq. \eqref{eq:relu_derivative} and Eq. \eqref{eq:ut_def},
and the last step uses the definition of the matrix $H$.
\end{proof}

\subsection{Proof of Claim \ref{cla:inductive_claim}}
\begin{proof}
We can rewrite $u(k+1) - u(k) \in \R^n$ in the following sense
\begin{align*}
u(k+1) - u(k) = v_1 + v_2 .
\end{align*}

Then, we can rewrite $v_{1,i} \in \R$ with the notation of $H$ and $H^{\bot}$

\begin{align*}
v_{1,i} 
= & ~ - \frac{\eta}{ m } \sum_{j=1}^n x_i^\top x_j (u_j - y_j) \sum_{r \in S_i} {\bf 1}_{ w_r(k)^\top x_i \geq 0 , w_r(k)^\top x_j \geq 0 } \\
= & ~ - \eta \sum_{j=1}^n (u_j - y_j) ( H_{i,j}(k) - H_{i,j}^{\bot}(k) ) ,
\end{align*}

which means vector $v_1 \in \R^n$ can be written as
\begin{align}\label{eq:rewrite_v1}
v_1 = \eta ( y - u(k) )^\top ( H( k ) - H^{\bot}( k ) ) .
\end{align}

We can rewrite $\| y - u(k+1) \|_2^2$ as follows:
\begin{align*}
& ~\| y - u(k+1) \|_2^2\\
= & ~ \| y - u(k) - ( u(k+1) - u(k) ) \|_2^2 \\
= & ~ \| y - u(k) \|_2^2 - 2 ( y - u(k) )^\top  ( u(k+1) - u(k) )\\
 + & ~ \| u (k+1) - u(k) \|_2^2 .
\end{align*}

We can rewrite the second term in the above Equation in the following sense,
\begin{align*}
 & ~ ( y - u(k) )^\top ( u(k+1) - u(k) ) \\
= & ~ ( y - u(k) )^\top ( v_1 + v_2 )  \\
= & ~ ( y - u(k) )^\top v_1 + ( y - u(k) )^\top v_2  \\
= & ~ \eta ( y - u(k) )^\top H(k) ( y - u (k) ) \\
- &~ \eta ( y - u(k) )^\top H(k)^{\bot} ( y - u(k) ) + ( y - u(k) )^\top v_2 ,
\end{align*}
where the third step follows from Eq.~\eqref{eq:rewrite_v1}.

Thus, we have
\begin{align*}
&~\| y - u(k+1) \|_2^2 \\
= & ~ \| y - u(k) \|_2^2 + C_1 + C_2 + C_3 + C_4 \\
\leq & ~ \| y - u(k) \|_2^2 ( 1 - \eta \lambda + 8 \eta n R  + 8 \eta n R  + \eta^2 n^2 ) ,
\end{align*}
where the last step follows from Claim~\ref{cla:C1}, \ref{cla:C2}, \ref{cla:C3} and \ref{cla:C4},
whose proof is given later. 
\end{proof}

\subsection{Proof of Claim \ref{cla:yu0}}\label{sec:proof_yu0}
\begin{proof}

\begin{align*}
\|y-u(0)\|_2^2
= & ~ \sum_{i=1}^n(y_i-f(W(0),a,x_i))^2\\
= & ~ \sum_{i=1}^n \Big( y_i-\frac {1} {\sqrt{m}}\sum_{r=1}^{m} a_r\phi(w_r^\top x_i) \Big)^2\\
= & ~ \sum_{i=1}^n y_i^2-2\sum_{i=1}^n \frac{y_i}{\sqrt{m}}\sum_{r=1}^{m} a_r\phi(w_r^\top x_i) +\sum_{i=1}^n \frac {1}{m} \Big( \sum_{r=1}^{m} a_r\phi(w_r^\top x_i) \Big)^2.
\end{align*}

Fix $r\in [m]$ and $i\in [n]$.
Since $w_r\sim \N(0,I)$ and $\|x_i\|_2=1$,
$w_r^\top x_i$ follows distribution $\N(0,1)$.
From concentration of Gaussian distribution,
we have
\begin{align*}
\Pr_{w_r}[w_r^\top x_i\geq \sqrt{2\log (2mn / \delta) }]\leq \frac{\delta}{2mn}.
\end{align*}
Let $E_1$ be the event that
for all $r\in [m]$ and $i\in [n]$ we have
$
\phi(w_r^\top x_i)\leq \sqrt{2\log ( 2mn/ \delta) }.
$ 
Then by union bound,
$\Pr[E_1]\geq 1-\frac {\delta}{2}$,

Fix $i\in [n]$.
For every $r\in [m]$,
we define random variable $z_{i,r}$ as
\begin{align*}
z_{i,r}:=\frac {1}{\sqrt{m}} \cdot a_r \cdot \phi(w_r^\top x_i) \cdot \mathbf{1}_{w_r^\top x_i\leq \sqrt{2\log ( 2mn / \delta ) }}.
\end{align*}
Then $z_{i,r}$ only depends on $a_r\in \{-1,1\}$ and $w_r\sim \N(0,I)$.
Notice that $\E_{a_r,w_r}[z_{i,r}]=0$,
and $|z_{i,r}|\leq \sqrt{2\log ( 2mn / \delta ) }$.
Moreover,
\begin{align*}
& ~ \E_{a_r,w_r}[z_{i,r}^2] \\
= & ~\E_{a_r,w_r}\left[\frac {1}{m}a_r^2\phi^2(w_r^\top x_i)\mathbf{1}^2_{w_r^\top x_i\leq \sqrt{2\log ( 2mn / \delta) }}\right]\\
= & ~\frac {1}{m}\E_{a_r}[a_r^2] \cdot \E_{w_r} \Big[\phi^2(w_r^\top x_i)\mathbf{1}^2_{w_r^\top x_i\leq \sqrt{2\log ( 2mn / \delta) }} \Big]\\
\leq & ~\frac {1}{m}\cdot 1 \cdot \E_{w_r}[(w_r^\top x_i)^2] \\
= & ~ \frac {1} {m},
\end{align*}
where the second step uses independence between $a_r$ and $w_r$,
the third step uses $a_r\in \{-1,1\}$ and $\phi(t) = \max \{ t,0\}$,
and the last step follows from $w_r^\top x_i\sim \N(0,1)$.

Now we are ready to apply Bernstein inequality~(Lemma \ref{lem:bernstein}) to get for all $t>0$,
\begin{align*}
\Pr \left[ \sum_{r=1}^m z_{i,r}>t \right] \leq \exp\left(-\frac{t^2/2}{m\cdot \frac{1}{m}+\sqrt{2\log (2mn/\delta)} \cdot t/3} \right).
\end{align*}
Setting $t=\sqrt{2\log ( 2mn / \delta) }\cdot \log ( 4n / \delta )$,
we have with probability at least $1-\frac {\delta}{4n}$,
\begin{align*}
\sum_{r=1}^m z_{i,r}\leq \sqrt{2\log ( 2mn / \delta) }\cdot \log ( 4n / \delta ).
\end{align*}

Notice that we can also apply Bernstein inequality (Lemma~\ref{lem:bernstein}) on $-z_{i,r}$ to get
\begin{align*}
\Pr \left[ \sum_{r=1}^m z_{i,r}<-t \right] \leq \exp\left(-\frac{t^2/2}{m\cdot \frac{1}{m}+\sqrt{2\log (2mn/\delta)} \cdot t/3} \right).
\end{align*}
Let $E_2$ be the event that
for all $i\in[n]$,
\begin{align*}
\left| \sum_{r=1}^m z_{i,r} \right| \leq \sqrt{2\log ( 2mn / \delta) }\cdot \log ( 4n / \delta ).
\end{align*}
By applying union bound on all $i\in [n]$,
we have $\Pr[E_2]\geq 1-\frac {\delta}{2}$.

If both $E_1$ and $E_2$ happen,
we have
\begin{align*}
& ~ \|y-u(0)\|_2^2 \\
= & ~ \sum_{i=1}^n y_i^2-2\sum_{i=1}^n y_i\sum_{r=1}^{m} z_{i,r}+\sum_{i=1}^n \Big( \sum_{r=1}^{m}z_{i,r} \Big)^2\\
\leq & ~\sum_{i=1}^n y_i^2+2\sum_{i=1}^n |y_i|\sqrt{2\log ( 2mn / \delta) }\cdot \log ( 4n / \delta ) + \sum_{i=1}^n \Big( \sqrt{2\log ( 2mn / \delta) }\cdot \log ( 4n / \delta ) \Big)^2\\
= & ~ O(n\log(m/\delta)\log^2(n/\delta)) ,
\end{align*}
where the first step uses $E_1$, the second step uses $E_2$, and the last step follows from $|y_i| = O(1), \forall i \in [n]$.

By union bound, this will happen with probability at least $1-\delta$.
\end{proof}

\subsection{Proof of Claim \ref{cla:C1}}\label{sec:proof_c1}

\begin{proof}
By Lemma \ref{lem:3.2} and our choice of $R<\frac{\lambda}{8n}$,
We have $\|H(0)-H(k)\|_F\leq 2n\cdot \frac{\lambda}{8n}=\frac {\lambda}{4}$.
Recall that $\lambda=\lambda_{\min}(H(0))$.
Therefore
\begin{align*}
\lambda_{\min}(H(k)) \geq \lambda_{\min}(H(0))- \|H(0)-H(k)\|\geq \lambda /2.
\end{align*}
Then we have
\begin{align*}
  (y - u(k))^\top H(k) ( y - u(k) ) \geq \| y - u(k) \|_2^2 \cdot \lambda / 2.
\end{align*}
Thus, we complete the proof.
\end{proof}

\subsection{Proof of Claim \ref{cla:C2}}\label{sec:proof_c2}

\begin{proof}
Note that
\begin{align*}
C_2 \leq 2 \eta \| y - u(k) \|_2^2 \| H(k)^{\bot} \|.
\end{align*}

It suffices to upper bound $\| H(k)^{\bot} \|$. Since $\| \cdot \| \leq \| \cdot \|_F$, then it suffices to upper bound $\| \cdot \|_F$.

For each $i \in [n]$, we define $y_i$ as follows
\begin{align*}
y_i=\sum_{r=1}^m\mathbf{1}_{r\in \ov{S}_i} .
\end{align*}

Using Fact~\ref{fact:bound_H_k_bot}, we have $\| H(k)^{\bot} \|_2 \leq \frac{n}{m^2} \sum_{i=1}^n y_i^2 $.

Fix $i \in [n]$. The plan is to use Bernstein inequality to upper bound $y_i$ with high probability.

First by Eq.~\eqref{eq:Air_bound} we have 
$
\E[\mathbf{1}_{r\in \ov{S}_i}]\leq R 
$. 
We also have
\begin{align*}
\E \left[(\mathbf{1}_{r\in \ov{S}_i}-\E[\mathbf{1}_{r\in \ov{S}_i}])^2 \right]
 = & ~ \E[\mathbf{1}_{r\in \ov{S}_i}^2]-\E[\mathbf{1}_{r\in \ov{S}_i}]^2\\
\leq & ~ \E[\mathbf{1}_{r\in \ov{S}_i}^2] \\
\leq & ~ R .
\end{align*}
Finally we have $|\mathbf{1}_{r\in \ov{S}_i}-\E[\mathbf{1}_{r\in \ov{S}_i}]|\leq 1$.

Notice that $\{\mathbf{1}_{r\in \ov{S}_i}\}_{r=1}^m$ are mutually independent,
since $\mathbf{1}_{r\in \ov{S}_i}$ only depends on $w_r(0)$.
Hence from Bernstein inequality (Lemma \ref{lem:bernstein}) we have for all $t>0$,
\begin{align*}
\Pr \left[ y_i > m\cdot R+t \right] \leq \exp \left(-\frac{t^2/2}{m\cdot R+t/3} \right).
\end{align*}
By setting $t=3mR$, we have
\begin{align}\label{eq:Si_size_bound}
\Pr \left[ y_i > 4mR \right] \leq \exp(-mR).
\end{align}
Hence by union bound,
with probability at least $1-n\exp(-mR)$,
\begin{align*}
\| H(k)^{\bot} \|_F^2\leq \frac{n}{m^2}\cdot n\cdot (4mR)^2=16n^2R^2 .
\end{align*}
Putting all together we have
\begin{align*}
\| H(k)^{\bot} \|\leq \| H(k)^{\bot} \|_F\leq 4nR
\end{align*}
with probability at least $1-n\exp(-mR)$.

\end{proof}

\subsection{Proof of Claim \ref{cla:C3}}\label{sec:proof_c3}

\begin{proof}
Using Cauchy-Schwarz inequality, we have
$
C_3 \leq 2 \| y - u(k) \|_2 \cdot \| v_2 \|_2
$. 
We can upper bound $\| v_2 \|_2$ in the following sense
\begin{align*}
\| v_2 \|_2^2
\leq &~ \sum_{i=1}^n \left(\frac{\eta}{ \sqrt{m} } \sum_{ r \in \ov{S}_i } \left| ( \frac{ \partial L(W(k)) }{ \partial w_r(k) } )^\top x_i \right|\right)^2\\
= &~ \frac{\eta^2}{ m }\sum_{i=1}^n \left(\sum_{r=1}^m \mathbf{1}_{r\in \ov{S}_i}\left| ( \frac{ \partial L(W(k)) }{ \partial w_r(k) } )^\top x_i \right|\right)^2\\
\leq &~ \frac{\eta^2}{ m }\cdot \max_{r \in [m]} \left|  \frac{ \partial L(W(k)) }{ \partial w_r(k) } \right|^2\cdot\sum_{i=1}^n \left(\sum_{r=1}^m \mathbf{1}_{r\in \ov{S}_i}\right)^2\\
 \leq & ~ \frac{\eta^2}{ m }\cdot (\frac{ \sqrt{n} }{ \sqrt{m} } \| u(k) - y\|_2 )^2 \cdot \sum_{i=1}^n \left(\sum_{r=1}^m \mathbf{1}_{r\in \ov{S}_i}\right)^2\\
  \leq & ~ \frac{\eta^2}{ m }\cdot (\frac{ \sqrt{n} }{ \sqrt{m} } \| u(k) - y\|_2 )^2 \cdot \sum_{i=1}^n (4mR)^2\\
  = & ~ 16n^2R^2\eta^2\| u(k) - y\|_2^2,
\end{align*}
where the first step follows from definition of $v_2$, the fourth step follows from $\max_{r \in [m]} | \frac{ \partial L (W(k)) }{ \partial w_r(k) } | \leq \frac{ \sqrt{n} }{ \sqrt{m} } \cdot \| u(k) - y \|_2$, the fifth step follows from $\sum_{r=1}^m {\bf 1}_{r \in \ov{S}_i } \leq 4 m R$ with probability at least $1-\exp(-m R)$.
\end{proof}

\subsection{Proof of Claim \ref{cla:C4}}\label{sec:proof_c4}
\begin{proof}
We have
\begin{align*}
\LHS \leq & ~ \eta^2 \sum_{i=1}^n \frac{1}{m} \left( \sum_{r=1}^m \Big\| \frac{ \partial L( W(k) ) }{ \partial w_r(k) } \Big\|_2 \right)^2 \\
\leq & ~ \eta^2 n^2 \| y - u(k) \|_2^2.
\end{align*}
where the first step follows from \eqref{eq:w_update}
and the last step follows from \eqref{eq:gradient_bound}.
\end{proof}


\section{Cubic Suffices}\label{sec:cubic_suffices}

We prove a more general version of Lemma~\ref{lem:3.1} in this section.
\begin{theorem}[Data-dependent version, bounding the difference between discrete and continuous]\label{thm:3.1_matrix_chernoff}
Let $H^{\cts}$ and $H^{\dis}$ be defined as Definition~\ref{def:data_dependent_function}. Let $\lambda, \alpha, \beta$ be satisfied Assumption~\ref{ass:data_dependent_assumption}. If 
\begin{align*}
m = \Omega ( (\lambda^{-2} \beta + \lambda^{-1} \alpha ) \log (n/\delta)  ),
\end{align*}
we have
\begin{align*}
\| H^{\dis} - H^{\cts} \|_2 \leq \lambda/ 4, \mathrm{~and~} \lambda_{\min} ( H^{\dis} ) \geq \frac{3}{4} \lambda
\end{align*}
holds with probability at least $1-\exp(-\Omega(\log (n/\delta)))$.
\end{theorem}
\begin{proof}

Recall the definition, we know
\begin{align*}
H^{\cts} = \E_{w} [ H(w) ], \mathrm{~~~and~~~} H^{\dis} = \frac{1}{m} \sum_{r=1}^m H(w_r) .
\end{align*}
We define matrix $Y_r = H(w_r) - \E_w[ H(w) ]$. We know that, $Y_r$ are all independent, 
\begin{align*}
\E[ Y_r ] = 0, ~~~
\| Y_r \| \leq \alpha, ~~~
\Big\| \sum_{r=1}^m \E [ Y_r Y_r^\top ] \Big\| \leq m \beta.
\end{align*}
Let $Y = \sum_{r=1}^m Y_r$.
We apply Matrix Bernstein inequality (Lemma~\ref{lem:matrix_bernstein}) with $t = \sqrt{ m \beta \log (n/\delta)} + \alpha \log(n/\delta)$,
\begin{align*}
\Pr[ \| Y \| \geq t ]
\leq & ~ 2 n \exp \Big( - \frac{t^2/2}{ m \beta + \alpha t/3 } \Big) \\
\leq & ~ 2 n \exp ( - \log(n/\delta) ) \\
\leq & ~ \exp ( - \Omega(\log (n/\delta)) ) .
\end{align*}
Thus, we have
\begin{align*}
\Pr \left[ \Big\| \frac{1}{m} \sum_{r=1}^m Y_r \Big\| \geq \frac{1}{m} ( \sqrt{m \beta\log (n/\delta)} + \alpha \log (n/\delta) )  \right] \leq \exp(-\Omega(\log (n/\delta))).
\end{align*}
In order to guarantee that $\frac{1}{m} ( \sqrt{m \beta\log (n/\delta)} + \alpha \log (n/\delta) ) \leq \lambda$, we need
\begin{align*}
\sqrt{m} \geq \lambda^{-1} \sqrt{\beta \log (n/\delta)}
\end{align*}
when the first term is the dominated one; we need
\begin{align*}
m \geq \lambda^{-1} \alpha \log(n/\delta).
\end{align*}
Overall, we need 
\begin{align*}
m \geq \Omega( ( \lambda^{-2} \beta + \lambda^{-1} \alpha ) \log (n/\delta) ).
\end{align*} 
Thus, we complete the proof.
\end{proof}

\begin{table}\caption{Table of Parameters for the $m = \wt{\Omega}(n^3)$ result in Section~\ref{sec:cubic_suffices}. {\bf Nt.} stands for notations.}
\centering
\begin{tabular}{ | l| l| l| l| } 
\hline
{\bf Nt.} & {\bf Choice} & {\bf Place} & {\bf Comment} \\\hline
$\lambda$ & $:= \lambda_{\min}(H^{\cts}) $ & Part 1 of Assumption~\ref{ass:data_dependent_assumption} & Data-dependent \\ \hline
$\alpha$ & Absolute & Part 2 of Assumption~\ref{ass:data_dependent_assumption} & Data-dependent \\ \hline
$\beta$ & Variance & Part 3 of Assumption~\ref{ass:data_dependent_assumption} & Data-dependent \\ \hline
$R$ & $\lambda/n$ & Eq.~\eqref{eq:choice_of_eta_R} & Maximal allowed movement of weight \\ \hline
$D_{\cts}$ & $\frac{ \sqrt{\alpha} \| y - u(0) \|_2 }{ \sqrt{m} \lambda }$ & Lemma~\ref{lem:3.3_chernoff} & Actual moving distance, continuous case  \\ \hline
$D$ & $\frac{ 4\sqrt{\alpha} \| y - u(0) \|_2 }{ \sqrt{m} \lambda }$ & Theorem~\ref{thm:cubic} & Actual moving distance, discrete case  \\ \hline
$\eta$ & $\lambda/ ( \alpha n)$ & Eq.~\eqref{eq:choice_of_eta_R} & Step size of gradient descent \\ \hline
$m$ & $ ( \lambda^{-2} \beta + \lambda^{-1} \alpha ) \log(n/\delta)$ & Theorem~\ref{thm:3.1_matrix_chernoff} & Bounding discrete and continuous \\ \hline
$m$ & $\lambda^{-4} \alpha n^3 \log^3(n/\delta)$  & Lemma~\ref{lem:3.4} and Claim~\ref{cla:yu0} & $D < R$ and $\| y - u(0) \|_2^2 = \wt{O}(n)$ \\ \hline
\end{tabular}
\end{table}

\begin{lemma}[Stronger version of Lemma 3.3 in \cite{dzps19}]\label{lem:3.3_chernoff}
Let Part 4 in Assumption \ref{ass:data_dependent_assumption} hold. Let $D_{\cts} = \frac{ \sqrt{\alpha} \| y - u(0) \|_2 }{ \sqrt{m} \lambda }$. 
Suppose for $0 \leq s \leq t$, $\lambda_{\min} ( H( s ) ) \geq \lambda / 2$. Then we have 
\begin{align*}
\| y - u(t) \|_2^2 \leq \exp( - \lambda t ) \cdot \| y - u(0) \|_2^2,
\end{align*}
and
\begin{align*}
\| w_r(t) - w_r(0) \|_2 \leq D_{\cts}.
\end{align*}
\end{lemma}

\begin{proof}

Recall we can write the dynamics of predictions as 
\begin{align*}
\frac{ \d }{ \d t} u(t) = H(t) \cdot ( y - u(t) ) .
\end{align*}
We can calculate the loss function dynamics
\begin{align*}
\frac{\d }{ \d t } \| y - u(t) \|_2^2
= & ~ - 2 ( y - u(t) )^\top \cdot H(t) \cdot ( y - u(t) ) \\
\leq & ~ - \lambda \| y - u(t) \|_2^2.
\end{align*}

Thus we have $\frac{\d}{ \d t} ( \exp(\lambda t) \| y - u(t) \|_2^2 ) \leq 0$ and $\exp( \lambda t ) \| y - u(t) \|_2^2$ is a decreasing function with respect to $t$.

Using this fact we can bound the loss
\begin{align*}
\| y - u(t) \|_2^2 \leq \exp( - \lambda t ) \| y - u(0) \|_2^2.
\end{align*}

Therefore, $u(t) \rightarrow y$ exponentially fast.

Now, we can bound the gradient norm. Recall for $0 \leq s \leq t$,
\begin{align*}
\left\| \frac{ \d }{ \d s } w_r(s) \right\|_2
= & ~ \left\| \sum_{i=1}^n (y_i - u_i) \frac{1}{\sqrt{m}} a_r x_i \cdot {\bf 1}_{ w_r(s)^\top x_i \geq 0 } \right\|_2 .
\end{align*}
Define matrix $X_r\in \mathbb{R}^{d\times n}$ by setting the $i$-th column to be ${\bf 1}_{ w_r(s)^\top x_i \geq 0 }\cdot x_i$,
then $X_r^\top X_r=H(w_r(s))$,
where $H(\cdot )$ is the matrix defined in Definition \ref{def:data_dependent_function}.
Then we have $\|X_r^\top X_r\|_2\leq \alpha$ by Part 2 in Assumption \ref{ass:data_dependent_assumption},
which leads to $\|X_r\|_2\leq \sqrt{\alpha}$.
So we have
\begin{align}\label{eq:yu0-chernoff}
\left\| \frac{ \d }{ \d s } w_r(s) \right\|_2
= & ~ \frac{1}{ \sqrt{m} } \|X_r(y-u(s))\|_2 \notag\\
\leq & ~ \frac{1}{ \sqrt{m} } \|X_r\|_2\|(y-u(s))\|_2 \notag\\
\leq & ~ \frac{ \sqrt{\alpha} }{ \sqrt{m} } \| y - u(s) \|_2 \\
\leq & ~ \frac{ \sqrt{\alpha} }{ \sqrt{m} } \exp( - \lambda s ) \| y - u(0) \|_2. \notag
\end{align}
Integrating the gradient, we can bound the distance from the initialization
\begin{align*}
\| w_r(t) - w_r(0) \|_2 
\leq & ~ \int_0^t \left\| \frac{\d}{ \d s} w_r(s) \right\|_2 \d s \\
\leq & ~ \frac{ \sqrt{\alpha} \| y - u(0) \|_2 }{ \sqrt{m} \lambda } .
\end{align*}
\end{proof}

\subsection{Technical claims}

\begin{claim}\label{cla:C3-chernoff}
Let $C_3 = - 2 (y - u(k))^\top v_2$. Then we have
\begin{align*}
C_3 \leq \| y - u(k) \|_2^2 \cdot 8 \eta (\alpha n)^{1/2}R .
\end{align*}
with probability at least $1-n\exp(-mR)$.
\end{claim}
\begin{proof}
We have
\begin{align*}
\LHS \leq 2 \| y - u(k) \|_2 \cdot \| v_2 \|_2 .
\end{align*}
We can upper bound $\| v_2 \|_2$ in the following sense
\begin{align*}
\| v_2 \|_2^2
\leq & ~ \sum_{i=1}^n \left(\frac{\eta}{ \sqrt{m} } \sum_{ r \in \ov{S}_i } \left| \Big( \frac{ \partial L(W(k)) }{ \partial w_r(k) } \Big)^\top x_i \right| \right)^2\\
= & ~ \frac{\eta^2}{ m }\sum_{i=1}^n \left(\sum_{r=1}^m \mathbf{1}_{r\in \ov{S}_i}\left| \Big( \frac{ \partial L(W(k)) }{ \partial w_r(k) } \Big)^\top x_i \right|\right)^2\\
\leq & ~ \frac{\eta^2}{ m }\cdot \max_{r \in [m]} \left|  \frac{ \partial L(W(k)) }{ \partial w_r(k) } \right|^2\cdot \sum_{i=1}^n \left(\sum_{r=1}^m \mathbf{1}_{r\in \ov{S}_i}\right)^2\\
 \leq & ~ \frac{\eta^2}{ m }\cdot \Big( \frac{ \sqrt{\alpha} }{ \sqrt{m} } \| u(k) - y\|_2 \Big)^2 \cdot \sum_{i=1}^n \left(\sum_{r=1}^m \mathbf{1}_{r\in \ov{S}_i}\right)^2\\
  \leq & ~ \frac{\eta^2}{ m } \cdot \Big( \frac{ \sqrt{\alpha} }{ \sqrt{m} } \| u(k) - y\|_2 \Big)^2 \cdot \sum_{i=1}^n (4mR)^2\\
  = & ~ 16 \alpha n R^2 \eta^2 \| u(k) - y\|_2^2,
\end{align*}
where the first step follows from definition of $v_2$, the fourth step follows from Eq. \eqref{eq:yu0-chernoff} and
\begin{align*}
\max_{r \in [m]} \left| \frac{ \partial L (W(k)) }{ \partial w_r(k) } \right| 
= & ~\max_{r \in [m]} \left| \frac{ \d w_r(k) }{ \d k } \right| \\
\leq & ~ \frac{ \sqrt{\alpha} }{ \sqrt{m} } \| y - u(k) \|_2,
\end{align*} 
the fifth step follows from $\sum_{r=1}^m {\bf 1}_{r \in \ov{S}_i } \leq 4 m R$ with probability at least $1-\exp(-m R)$.
\end{proof}

\begin{claim}\label{cla:C4-chernoff}
Let $C_4  = \| u (k+1) - u(k) \|_2^2$. Then we have
\begin{align*}
C_4 \leq \eta^2 \alpha n \| y - u(k) \|_2^2.
\end{align*}
\end{claim}
\begin{proof}
We have
\begin{align*}
\LHS \leq & ~ \eta^2 \sum_{i=1}^n \frac{1}{m} \left( \sum_{r=1}^m \Big\| \frac{ \partial L( W(k) ) }{ \partial w_r(k) } \Big\|_2 \right)^2 \\
\leq & ~ \eta^2 \sum_{i=1}^n \frac{1}{m} \left( \sum_{r=1}^m \frac{ \sqrt{\alpha} }{ \sqrt{m} } \| u(k) - y\|_2 \right)^2 \\
\leq & ~ \eta^2 \alpha n \| y - u(k) \|_2^2 .
\end{align*}
\end{proof}

\subsection{Main result}

\begin{theorem}\label{thm:cubic}
Assume Part 1 and 2 of Assumption~\ref{ass:data_dependent_assumption}. 
Recall that $\lambda=\lambda_{\min}(H^{\cts})>0$.
Let $m = \Omega( \lambda^{-4} n^3\alpha \log^3 (n/\delta) )$, we i.i.d. initialize $w_r \in {\cal N}(0,I)$, $a_r$ sampled from $\{-1,+1\}$ uniformly at random for $r\in [m]$, and we set the step size $\eta = O( \lambda / (\alpha n) )$ then with probability at least $1-\delta$ over the random initialization we have for $k = 0,1,2,\cdots$
\begin{align}\label{eq:cubic_condition}
\| u (k) - y \|_2^2 \leq ( 1 - \eta \lambda / 2 )^k \cdot \| u (0) - y \|_2^2.
\end{align}
\end{theorem}

\begin{proof}

This proof, similar to the proof of Theorem \ref{thm:quartic},
 is again by induction.
Eq. \eqref{eq:cubic_condition} trivially holds when $k=0$,
which is the base case.

If Eq. \eqref{eq:cubic_condition} holds for $k' = 0, \cdots, k$, then we claim that for all $r\in [m]$
\begin{align}
\| w_r(k+1) - w_r(0) \|_2 \leq \frac{ 4 \sqrt{\alpha} \| y - u (0) \|_2 }{ \sqrt{m} \lambda } := D
\end{align}

To see this, we use the norm of gradient to bound this distance,
\begin{align*}
\| w_r(k+1) - w_r(0) \|_2 
\leq & ~ \eta \sum_{k'=0}^k \left\| \frac{ \partial L( W(k') ) }{ \partial w_r(k') } \right\|_2 \\
\leq & ~ \eta \sum_{k'=0}^k \frac{ \sqrt{\alpha} \| y - u(k') \|_2 }{ \sqrt{m} } \\
\leq & ~ \eta \sum_{k'=0}^k \frac{ \sqrt{\alpha} ( 1 - \eta \lambda /2 )^{k'/2} }{ \sqrt{m} } \| y - u(0) \|_2 \\
\leq & ~ \eta \sum_{k'=0}^{\infty} \frac{ \sqrt{\alpha} (1-\eta\lambda/2)^{k'/2} }{ \sqrt{m} } \| y - u(0) \|_2  \\
= & ~ \frac{ 4 \sqrt{\alpha} \| y - u(0) \|_2 }{ \sqrt{m} \lambda },
\end{align*}
where the first step follows from Eq. \eqref{eq:w_update}, the second step follows from Eq. \eqref{eq:yu0-chernoff}, the third step follows from the induction hypothesis, the fourth step relaxes the summation to an infinite summation, and the last step follows from $\sum_{k'=0}^{\infty}(1-\eta\lambda/2)^{k'/2}=\frac {2}{\eta\lambda}$.

Then from Claim \ref{cla:C4-chernoff},
it is sufficient to choose $\eta=\frac{\lambda}{4\alpha n}$ so that Eq. \eqref{eq:cubic_condition} holds for $k'=k+1$. This completes the induction step.

{\bf Over-parameterization size, lower bound on $m$.}

We require 
\begin{align*}
D = \frac{4\sqrt{\alpha}\|y-u(0)\|_2}{\sqrt{m}\lambda} < R = \frac{\lambda}{64n},
\end{align*}
and 
\begin{align*}
3n^2\exp(-mR/10)\leq \delta .
\end{align*}

This implies that
\begin{align*}
m 
= & ~ \Omega ( \lambda^{-4} n^2 \alpha \| y - u(0) \|_2^2  ) \\
= & ~ \Omega ( \lambda^{-4} n^3 \alpha \log(m/\delta) \log^2(n/\delta) ),
\end{align*}
where the last step follows from Claim~\ref{cla:yu0}.
\end{proof} 
\section{Quadratic Suffices}\label{sec:quadratic_suffices}

\begin{lemma}[perturbed $w$]\label{lem:3.2-quadratic}
Let $R \in (0,1)$. Let  Assumption 4 in \ref{ass:data_dependent_assumption} hold,
i.e. for all $i\neq j$,
$|x_i^\top x_j| \leq \theta / \sqrt{n}$.
If $\wt{w}_1, \cdots, \wt{w}_m$ are i.i.d. generated ${\cal N}(0,I)$. For any set of weight vectors $w_1, \cdots, w_m \in \R^d$ that satisfy for any $r\in [m]$, $\| \wt{w}_r - w_r \|_2 \leq R$, then the $H : \R^{m \times d} \rightarrow \R^{n \times n}$ defined
\begin{align*}
    H(w)_{i,j} =  \frac{1}{m} x_i^\top x_j \sum_{r=1}^m {\bf 1}_{ w_r^\top x_i \geq 0, w_r^\top x_j \geq 0 } .
\end{align*}
Then we have
\begin{align*}
\| H (w) - H(\wt{w}) \|_F < 2 \left(n(1+\theta^2)\right)^{1/2} R,
\end{align*}
holds with probability at least $1-n^2 \cdot \exp(-m R /10)$.
\end{lemma}
\begin{proof}

The random variable we care is
\begin{align*}
\sum_{i=1}^n \sum_{j=1}^n | H(\wt{w})_{i,j} - H(w)_{i,j} |^2
= & ~ \frac{1}{m^2} \sum_{i=1}^n \sum_{j=1}^n \left| x_i^\top x_j \sum_{r=1}^m ( {\bf 1}_{ \wt{w}_r^\top x_i \geq 0, \wt{w}_r^\top x_j \geq 0} - {\bf 1}_{ w_r^\top x_i \geq 0 , w_r^\top x_j \geq 0 } ) \right|^2 \\
= & ~ B_1+B_2 ,
\end{align*} 
where $B_1,B_2$ are defined as
\begin{align*}
B_1 =&~\frac{1}{m^2} \sum_{i=1}^n \left| \sum_{r=1}^m ( {\bf 1}_{ \wt{w}_r^\top x_i \geq 0} - {\bf 1}_{ w_r^\top x_i \geq 0} ) \right|^2 , \\
B_2 =&~\frac{1}{m^2} \sum_{i=1}^n \sum_{j\in [n]\backslash\{i\}} \left| x_i^\top x_j \sum_{r=1}^m ( {\bf 1}_{ \wt{w}_r^\top x_i \geq 0, \wt{w}_r^\top x_j \geq 0} - {\bf 1}_{ w_r^\top x_i \geq 0 , w_r^\top x_j \geq 0 } ) \right|^2 .
\end{align*}

We can further bound $B_2$ as 
\begin{align*}
B_2 
\leq & ~ \frac{1}{m^2} \sum_{i=1}^n \sum_{j\in [n]\backslash\{i\}} \frac{\theta^2}{n}\left| \sum_{r=1}^m ( {\bf 1}_{ \wt{w}_r^\top x_i \geq 0, \wt{w}_r^\top x_j \geq 0} - {\bf 1}_{ w_r^\top x_i \geq 0 , w_r^\top x_j \geq 0 } ) \right|^2 \\
= & ~ \frac{\theta^2}{nm^2} \sum_{i=1}^n \sum_{j\in [n]\backslash\{i\}} \left| \sum_{r=1}^m ( {\bf 1}_{ \wt{w}_r^\top x_i \geq 0, \wt{w}_r^\top x_j \geq 0} - {\bf 1}_{ w_r^\top x_i \geq 0 , w_r^\top x_j \geq 0 } ) \right|^2 .
\end{align*}

For each $r,i,j$, we define
\begin{align*}
s_{r,i,j} :=  {\bf 1}_{ \wt{w}_r^\top x_i \geq 0, \wt{w}_r^\top x_j \geq 0} - {\bf 1}_{ w_r^\top x_i \geq 0 , w_r^\top x_j \geq 0 } .
\end{align*} 
Then we can rewrite $B_1$ and $B_2$ as
\begin{align*}
B_1 = & ~ \frac{1}{m^2} \sum_{i=1}^n \Big( \sum_{r=1}^m s_{r,i,i} \Big)^2 , \\
B_2 = & ~ \frac{\theta^2}{nm^2}  \sum_{i=1}^n \sum_{j\in [n]\backslash\{i\}} \Big( \sum_{r=1}^m s_{r,i,j} \Big)^2.
\end{align*}
Therefore it is sufficient to bound $\sum_{r=1}^m s_{r,i,j}$ simutaneously for all pair $i,j$.
Using same technique in the proof of Theorem \ref{lem:3.2},
we have
\begin{align*}
\Pr \left[ \frac{1}{m} \sum_{r=1}^m s_{r,i,j} \geq 2  R \right] \leq \exp( - m R /10 ).
\end{align*}
By applying union bound on all $i,j$ pairs,
we get with probability at least $1-\exp( - m R /10 )$,

\begin{align*}
\| H (w) - H(\wt{w}) \|_F^2\leq B_1+B_2 \leq 4nR^2(1+\theta)^2.
\end{align*}
which is precisely what we need.
\end{proof}

\begin{table}\caption{Table of Parameters for the $m = \wt{\Omega}(n^2)$ result in Section~\ref{sec:quadratic_suffices}. {\bf Nt.} stands for notations.}
\centering
\begin{tabular}{ | l| l| l| l| } 
\hline
{\bf Nt.} & {\bf Choice} & {\bf Place} & {\bf Comment} \\\hline
$\lambda$ & $:= \lambda_{\min}(H^{\cts}) $ & Part 1 of Assumption~\ref{ass:data_dependent_assumption} & Data-dependent \\ \hline
$\alpha$ & Absolute & Part 2 of Assumption~\ref{ass:data_dependent_assumption} & Data-dependent \\ \hline
$\beta$ & Variance & Part 3 of Assumption~\ref{ass:data_dependent_assumption} & Data-dependent \\ \hline
$\theta$ & Inner product & Part 4 of Assumption~\ref{ass:data_dependent_assumption} & Data-dependent \\ \hline
$R$ & $ \frac{\lambda}{  \sqrt{n} } \cdot \min \{ \frac{1}{\sqrt{\alpha}} , \frac{1}{\sqrt{1+\theta^2}} \} $ & Eq.~\eqref{eq:choice_of_eta_R_quadratic} & Maximal allowed movement of weight \\ \hline
$D$ & $\frac{ 4\sqrt{\alpha} \| y - u(0) \|_2 }{ \sqrt{m} \lambda }$ & Theorem~\ref{thm:quadratic} & Actual moving distance, discrete case  \\ \hline
$\eta$ & $\lambda/ ( \alpha n)$ & Eq.~\eqref{eq:choice_of_eta_R} & Step size of gradient descent \\ \hline
$m$ & $ ( \lambda^{-2} \beta + \lambda^{-1} \alpha ) \log(n/\delta)$ & Theorem~\ref{thm:3.1_matrix_chernoff} & Bounding discrete and continuous \\ \hline
$m$ & $\lambda^{-4} \alpha (\alpha + \theta^2) n^2 \log^3(n/\delta)$  & Lemma~\ref{lem:3.4} and Claim~\ref{cla:yu0} & $D< R$ and $\| y - u(0) \|_2^2 = \wt{O}(n)$ \\ \hline
\end{tabular}
\end{table}

\begin{claim}\label{cla:C1-quadratic}
Assume $R\leq \frac{\lambda}{64\sqrt{n}}\cdot \frac {1}{\sqrt{1+\theta^2}}$.
Let $C_1 = -2 \eta (y - u(k))^\top H(k) ( y - u(k) )$ . We have
\begin{align*}
C_1 \leq - \| y - u(k) \|_2^2 \cdot \eta \lambda
\end{align*}
\end{claim}
\begin{proof}
By Lemma \ref{lem:3.2-quadratic} and our choice of $R\leq \frac{\lambda}{64\sqrt{n}}\cdot \frac {1}{\sqrt{1+\theta^2}}$,
We have 
\begin{align*}\|H(0)-H(k)\|_F\leq 2 \left(n(1+\theta^2)\right)^{1/2}\cdot \frac{\lambda}{64\sqrt{n}}\cdot \frac {1}{\sqrt{1+\theta^2}}\leq \frac {\lambda}{4}.
\end{align*}
Recall that $\lambda=\lambda_{\min}(H(0))$.
Therefore
\begin{align*}
\lambda_{\min}(H(k)) \geq \lambda_{\min}(H(0))- \|H(0)-H(k)\|\geq \lambda /2.
\end{align*}
Then we have
\begin{align*}
  (y - u(k))^\top H(k) ( y - u(k) ) \geq \| y - u(k) \|_2^2 \cdot \lambda / 2.
\end{align*}
Thus, we complete the proof.
\end{proof}

\begin{claim}\label{cla:C2-quadratic}
Let $C_2 = 2 \eta ( y - u(k) )^\top H(k)^{\bot} ( y - u(k) )$. We have
\begin{align*}
C_2 \leq \| y - u(k) \|_2^2 \cdot 8\eta R\left(n(1+\theta^2)\right)^{1/2}.
\end{align*}
holds with probability $1-n\exp(-mR)$.
\end{claim}
\begin{proof}
Note that
\begin{align*}
C_2 \leq 2 \eta \cdot \| y - u(k) \|_2^2 \cdot \| H(k)^{\bot} \|.
\end{align*}
It suffices to upper bound $\| H(k)^{\bot} \|$. Since $\| \cdot \| \leq \| \cdot \|_F$, then it suffices to upper bound $\| \cdot \|_F$.

For each $i \in [n]$, we define $y_i$ as follows
\begin{align*}
y_i=\sum_{r=1}^m\mathbf{1}_{r\in \ov{S}_i} .
\end{align*}

Then we have
\begin{align*}
\| H(k)^{\bot} \|_F^2
= & ~ \sum_{i=1}^n\sum_{j=1}^n (H(k)^{\bot}_{i,j})^2\\
= & ~ \sum_{i=1}^n\sum_{j=1}^n \Big( \frac {1} {m}\sum_{r\in \ov{S}_i} x_i^\top x_j\mathbf{1}_{w_r(k)^\top x_i\geq 0,w_r(k)^\top x_j\geq 0} \Big)^2\\
= & ~ \sum_{i=1}^n\sum_{j=1}^n \Big( \frac {1} {m}\sum_{r=1}^m x_i^\top x_j\mathbf{1}_{w_r(k)^\top x_i\geq 0,w_r(k)^\top x_j\geq 0} \cdot \mathbf{1}_{r\in \ov{S}_i} \Big)^2\\
= & ~ \sum_{i=1}^n\sum_{j=1}^n \frac{ | x_i^\top x_j |^2 }{ m^2 } \Big( \sum_{r=1}^m \mathbf{1}_{w_r(k)^\top x_i\geq 0,w_r(k)^\top x_j\geq 0} \cdot \mathbf{1}_{r\in \ov{S}_i} \Big)^2 \\
= & ~ B_1+B_2,
\end{align*}
where $B_1$ and $B_2$ are defined as:
\begin{align*}
B_1:= & ~\sum_{i=1}^n \frac{1}{m^2}  \Big( \sum_{r=1}^m \mathbf{1}_{w_r(k)^\top x_i\geq 0} \cdot \mathbf{1}_{r\in \ov{S}_i} \Big)^2 , \\
B_2:= & ~\sum_{i=1}^n\sum_{j\in [n]\backslash\{i\}} \frac{ |x_i^\top x_j|^2 } {m^2}  \Big( \sum_{r=1}^m \mathbf{1}_{w_r(k)^\top x_i\geq 0,w_r(k)^\top x_j\geq 0} \cdot \mathbf{1}_{r\in \ov{S}_i} \Big)^2 .
\end{align*}
We bound $B_1$ and $B_2$ separately.

We first bound $B_1$.
\begin{align*}
B_1 
= & ~ \sum_{i=1}^n \frac {1}{m^2} \Big( \sum_{r=1}^m \mathbf{1}_{w_r(k)^\top x_i\geq 0} \cdot \mathbf{1}_{r\in \ov{S}_i} \Big)^2\\
\leq & ~\frac {1} {m^2}\sum_{i=1}^n\Big( \sum_{r=1}^m  \mathbf{1}_{r\in \ov{S}_i} \Big)^2\\
= & ~ \frac {1} {m^2}\sum_{i=1}^n y_i^2.
\end{align*}

Fix $i \in [n]$. The plan is to use Bernstein inequality to upper bound $y_i$ with high probability.

First by Eq.~\eqref{eq:Air_bound} we have 
\begin{align*}
\E \left[ \mathbf{1}_{r\in \ov{S}_i} \right] \leq R .
\end{align*}
We also have
\begin{align*}
\E \left[ \left( \mathbf{1}_{r\in \ov{S}_i}-\E \left[ \mathbf{1}_{r\in \ov{S}_i} \right] \right)^2 \right]
 = & ~ \E \left[ \mathbf{1}_{r\in \ov{S}_i}^2 \right]-\E \left[\mathbf{1}_{r\in \ov{S}_i} \right]^2\\
\leq & ~ \E \left[ \mathbf{1}_{r\in \ov{S}_i}^2 \right]\\
\leq & ~ R .
\end{align*}
Finally we have $|\mathbf{1}_{r\in \ov{S}_i}-\E[\mathbf{1}_{r\in \ov{S}_i}]|\leq 1$.

Notice that $\{\mathbf{1}_{r\in \ov{S}_i}\}_{r=1}^m$ are mutually independent,
since $\mathbf{1}_{r\in \ov{S}_i}$ only depends on $w_r(0)$.
Hence from Bernstein inequality (Lemma \ref{lem:bernstein}) we have for all $t>0$,
\begin{align*}
\Pr \left[ y_i > m\cdot R+t \right] \leq \exp \left(-\frac{t^2/2}{m\cdot R+t/3} \right).
\end{align*}
By setting $t=3mR$,
we have
\begin{align*}
\Pr \left[ y_i > 4mR \right] \leq \exp(-mR).
\end{align*}
Hence by union bound,
with probability at least $1-n\exp(-mR)$,
for all $i\in [n]$,
\begin{align*}
y_i \leq 4mR.
\end{align*}
If this happens, we have
\begin{align*}
B_1\leq 16nR^2.
\end{align*}

Next we bound $B_2$.
We have
\begin{align*}
B_2 = & ~\sum_{i=1}^n\sum_{j\in [n]\backslash\{i\}} \frac { | x_i^\top x_j |^2 }{ m^2 } \Big( \sum_{r=1}^m \mathbf{1}_{w_r(k)^\top x_i\geq 0,w_r(k)^\top x_j\geq 0} \cdot \mathbf{1}_{r\in \ov{S}_i} \Big)^2\\
\leq  & ~ \frac{1}{m^2}\sum_{i=1}^n \left( \sum_{j\in [n]\backslash\{i\}} (  x_i^\top x_j  )^4 \right)^{1/2} \cdot \left( \sum_{j\in [n]\backslash\{i\}}\Big( \sum_{r=1}^m \mathbf{1}_{w_r(k)^\top x_i\geq 0,w_r(k)^\top x_j\geq 0} \cdot \mathbf{1}_{r\in \ov{S}_i} \Big)^4 \right)^{1/2} \\
\leq & ~\frac{1}{m^2}\sum_{i=1}^n \left(\sum_{j\in [n]\backslash\{i\}} (  x_i^\top x_j )^4 \right)^{1/2} \left(\sum_{j\in [n]\backslash\{i\}} y_i^4 \right)^{1/2} \\
= & ~\frac{\sqrt{n-1}}{m^2}\sum_{i=1}^n \left(\sum_{j\in [n]\backslash\{i\}}(  x_i^\top x_j  )^4 \right)^{1/2} y_i^2 \\
\leq & ~16R^2\sqrt{n}\sum_{i=1}^n \left(\sum_{j\in [n]\backslash\{i\}} (  x_i^\top x_j  )^4 \right)^{1/2} ,
\end{align*}
where the last step happens when $y_i \leq 4mR$ for all $i\in [n]$.

Now, using the assumption $x_i^\top x_j\leq \frac{\theta}{\sqrt{n}}$ (Part 4 of Assumption~\ref{ass:data_dependent_assumption}),
we have  
\begin{align*}
B_2\leq 16nR^2\theta^2 .
\end{align*}

Putting things together, we have with probability at least $1-n\exp(-mR)$,
\begin{align*}
\| H(k)^{\bot} \|_F^2
\leq & ~ B_1+B_2\\
\leq & ~16nR^2(1+\theta^2).
\end{align*}
This gives us $\| H(k)^{\bot} \|\leq 4R\left(n(1+\theta^2)\right)^{1/2}$, which is precisely what we need.

\end{proof}

\subsection{Main result}

\begin{theorem}\label{thm:quadratic}
Let $\lambda,\alpha,\beta,\theta$ be defined as Assumption~\ref{ass:data_dependent_assumption}. Let 
\begin{align*}
m = \Omega \left( \lambda^{-4} n^2 \alpha \max\{1+\theta^2,\alpha\} \log^3(n/\delta) \right).
\end{align*} 
We i.i.d. initialize $w_r \in {\cal N}(0,I)$, $a_r$ sampled from $\{-1,+1\}$ uniformly at random for $r\in [m]$, and we set the step size $\eta = O( \lambda / (\alpha n) )$ then with probability at least $1-\delta$ over the random initialization we have for $k = 0,1,2,\cdots$
\begin{align*}
\| u (k) - y \|_2^2 \leq ( 1 - \eta \lambda / 2 )^k \cdot \| u (0) - y \|_2^2.
\end{align*}
\end{theorem}

\begin{proof}

{\bf Choice of $\eta$ and $R$.}
We want to choose $\eta$ and $R$ such that
\begin{align}\label{eq:choice_of_eta_R_quadratic}
( 1 - \eta \lambda + 8\eta R\left(n(1+\theta^2)\right)^{1/2}  + 8 \eta (\alpha n)^{1/2}R + \eta^2 \alpha n ) \leq (1-\eta\lambda/2) .
\end{align}

Now, if we set $\eta=\frac{\lambda }{4\alpha n}$ and $R=\frac{\lambda}{64\sqrt{n}}\cdot \min\{\frac {1}{\sqrt{1+\theta^2}},\frac {1}{\sqrt{\alpha}}\}$, we have 
\begin{align*}
8\eta R\left(n(1+\theta^2)\right)^{1/2}  + 8 \eta (\alpha n)^{1/2}R \leq \frac {1}{4}\eta \lambda ,
\end{align*}
and $\eta^2 n^2 \leq \frac {1}{4}\eta \lambda$.
This gives us 
\begin{align*}
\| y - u(k+1) \|_2^2 \leq & ~ \| y - u(k) \|_2^2 ( 1 - \eta \lambda / 2 )
\end{align*}
with probability at least $1-2n\exp(-mR)$.

{\bf Over-parameterization size, lower bound on $m$.}
By same analysis as in the proof of Theorem \ref{thm:cubic},
we still have
\begin{align*}
\| w_r(k+1) - w_r(0) \|_2 \leq \frac{ 4 \sqrt{\alpha} \| y - u (0) \|_2 }{ \sqrt{m} \lambda } := D
\end{align*}
We require 
\begin{align*}
D = \frac{4\sqrt{\alpha}\|y-u(0)\|_2}{\sqrt{m}\lambda} < R = \frac{\lambda}{64\sqrt{n}}\cdot \min \left\{ \frac{1}{ \sqrt{1+\theta^2} },\frac{1}{ \sqrt{\alpha} } \right\}
\end{align*}
and 
\begin{align*}
3n^2\exp(-mR/10)\leq \delta .
\end{align*}
This implies that
\begin{align*}
m 
= & ~ \Omega ( \lambda^{-4} n \alpha \| y - u(0) \|_2^2 \max\{1+\theta^2,\alpha\} ) \\
= & ~ \Omega ( \lambda^{-4} n^2 \alpha \cdot \max\{ 1+\theta^2,\alpha\} \cdot \log(m/\delta) \log^2(n/\delta) ) ,
\end{align*}
where the last step follows from Claim~\ref{cla:yu0}.
\end{proof} 

\section{Training Speed}\label{sec:training_speed}

In this section we change the initialization scheme as initialize each $w_r(0)$ as ${\N}(0,\kappa^2 I)$.
It is not hard to see that this just introduce an extra $\kappa^{-2}$ term for every occurrence of $m$.
\begin{algorithm} 
\caption{Training neural network using gradient descent with small variance.}
\label{alg:variance} 
\begin{algorithmic}[1]
    \Procedure{NNTraining}{$\{(x_i,y_i)\}_{i\in [n]}$}
    \State $w_r(0) \sim \N(0,\kappa^2I_d)$ for $r\in [m]$.
    \For{$t=1 \to T$}
        \State $u(t) \leftarrow \frac{1}{ \sqrt{m} } \sum_{r=1}^m a_r \sigma(w_r(t)^\top X) $ \Comment{$u(t) = f(W(t),x,a) \in \R^n$, it takes $O(mnd)$ time}
        \For{$r = 1 \to m$}
            \For{$i = 1 \to n$} 
                \State $Q_{i,:} \leftarrow \frac{1}{\sqrt{m}}a_r\sigma'(w_r(t)^\top x_i)x_i^\top$ \Comment{$Q_{i,:} = \frac{\partial f(W(t),x_i,a)}{\partial w_r}$, it takes $O(d)$ time}
            \EndFor
            \State $\text{grad}_r \leftarrow - Q^\top (y - u(t) ) $
            \Comment{ $Q = \frac{\partial f} {\partial w_r } \in \R^{n \times d}$, it takes $O(nd)$ time }
            \State $w_r(t+1) \leftarrow w_r(t) - \eta \cdot \text{grad}_r $
        \EndFor
    \EndFor
    \State \Return $W$
    \EndProcedure
\end{algorithmic}
\end{algorithm}

\begin{lemma}[Restatement of Lemma \ref{lem:4.1}, improved Version of Lemma C.1 in \cite{adhlw19}]\label{lem:4.1_restatement}
Recall that $\lambda=\lambda_{\min}(H^{\cts})>0$.
Fix $\kappa>0$.
Let $m = \Omega( \lambda^{-4} \kappa^{-2} n^4 \log (n/\delta) )$, we i.i.d. initialize $w_r \in {\N}(0,\kappa^2 I)$, $a_r$ sampled from $\{-1,+1\}$ uniformly at random for $r\in [m]$, and we set the step size $\eta = O( \lambda / n^2 )$ then with probability at least $1-\delta$ over the random initialization we have for $k = 0,1,2,\cdots$.
\begin{align*}
\| w_r(k+1) - w_r(0) \|_2 \leq \frac{ 4 \sqrt{n} \| y - u (0) \|_2 }{ \sqrt{m} \lambda }.
\end{align*}
\end{lemma}

\begin{lemma}[Improved version of Lemma C.2 in \cite{adhlw19}]\label{lem:HZ_norm_bound}
For integer $k\geq 0$,
define $Z(k)\in \mathbb{R}^{md\times n}$ as the matrix
\begin{align*}
Z(k)=\frac{1}{\sqrt{m}} \begin{pmatrix} a_1x_1 {\bf 1}_{ w_1(k)^\top x_1 \geq 0} & \cdots & a_1x_n {\bf 1}_{ w_1(k)^\top x_n \geq 0}\\
\vdots & \ddots & \vdots\\
a_mx_1 {\bf 1}_{ w_m(k)^\top x_1 \geq 0} & \cdots & a_mx_n {\bf 1}_{ w_m(k)^\top x_n \geq 0}\\
\end{pmatrix}.
\end{align*}
Under the same setting as Lemma \ref{lem:4.1_restatement},
with probability at least $1-4\delta$ over the random initialization, for $k = 0,1,2,\cdots$ we have
\begin{align*}
\|H(k)-H(0)\|_F = & ~ O(1) \cdot n \cdot \left( \delta+\frac{ n\sqrt{\log(m/\delta)\log^2(n/\delta)}} { \kappa\lambda\sqrt{m}  } \right) ,\\
\|Z(k)-Z(0)\|_F = & ~ O(1) \cdot  \left( n \cdot \Big(\delta+\frac{  n\sqrt{\log(m/\delta)\log^2(n/\delta)}} { \kappa\lambda\sqrt{m} } \Big) \right)^{1/2} .\\
\end{align*}
\end{lemma}
\begin{remark}
$H(0)$ is just $H^{\dis}$ defined in Lemma \ref{lem:3.1}.
\end{remark}

\begin{proof}
This proof is analogous to the proof of Lemma \ref{lem:3.2}.
Let 
\begin{align*}
R=\frac{ C n\sqrt{\log(m/\delta)\log^2(n/\delta)}} { \sqrt{m} \lambda }
\end{align*}
for some sufficiently small constant $C>0$. 

By Claim \ref{cla:yu0},
plugging 
\begin{align*}
\|y-u(0)\|_2 = O \Big( \sqrt{ n \log(m/\delta) \log^2(n/\delta) } \Big)
\end{align*}
into Lemma \ref{lem:4.1_restatement},
we have with probability at least $1-\delta$,
$\|w_r(k)-w_r(0)\|_2\leq R$ for all $k\geq 0$ and $r\in [m]$.

We start with bounding $\|H(k)-H(0)\|_F$ as 
\begin{align*}
\|H(k)-H(0)\|_F^2= & ~ \sum_{i=1}^n \sum_{j=1}^n | H(k)_{i,j} - H(0)_{i,j} |^2 \\
\leq & ~ \frac{1}{m^2} \sum_{i=1}^n \sum_{j=1}^n \left( \sum_{r=1}^m {\bf 1}_{ w_r(k)^\top x_i \geq 0, w_r(k)^\top x_j \geq 0} - {\bf 1}_{ w_r(0)^\top x_i \geq 0 , w_r(0)^\top x_j \geq 0 } \right)^2 \\
= & ~ \frac{1}{m^2} \sum_{i=1}^n \sum_{j=1}^n  \Big( \sum_{r=1}^m s_{r,i,j} \Big)^2 ,
\end{align*}

where the last step follows from for each $r,i,j$, we define
\begin{align*}
s_{r,i,j} :=  {\bf 1}_{ w_r(k)^\top x_i \geq 0, w_r(k)^\top x_j \geq 0} - {\bf 1}_{ w_r(0)^\top x_i \geq 0 , w_r(0)^\top x_j \geq 0 } .
\end{align*} 

We consider $i,j$ to be fixed. We simplify $s_{r,i,j}$ to $s_r$.
Then $s_r$ is a random variable that only depends on $w_r(0)$.
Since $\{w_r(0)\}_{r=1}^m$ are independent,
$\{s_r\}_{r=1}^m$ are also mutually independent.

We define the event
\begin{align*}
A_{i,r} = \left\{ \exists u : \| u - w_r(0) \|_2 \leq R, {\bf 1}_{ x_i^\top w_r(0) \geq 0 } \neq {\bf 1}_{ x_i^\top u \geq 0 } \right\}.
\end{align*}
Note this event happens if and only if $| w_r(0)^\top x_i | < R$. Recall that $w_r(0) \sim \N(0,\kappa^2 I)$. By anti-concentration inequality of Gaussian (Lemma~\ref{lem:anti_gaussian}), we have
\begin{align}\label{eq:Air_bound_variance}
\Pr[ A_{i,r} ] = \Pr_{ z \sim \N(0,\kappa^2) } [ | z | < R ] \leq \frac{ 2 R }{ \sqrt{2\pi}\kappa }.
\end{align}

Assume $ \|w_r(0)- w_r(k)\|_2\leq R$.
If   $\neg A_{i,r}$ and $\neg A_{j,r}$ happen,
then 
\begin{align*}
\left| {\bf 1}_{  w_r(k)^\top x_i \geq 0,  w_r(k)^\top x_j \geq 0} - {\bf 1}_{  w_r(0)^\top x_i \geq 0 ,  w_r(0)^\top x_j \geq 0 } \right|=0.
\end{align*}
If   $A_{i,r}$ or $A_{j,r}$ happen,
then 
\begin{align*}
\left| {\bf 1}_{  w_r(k)^\top x_i \geq 0,  w_r(k)^\top x_j \geq 0} - {\bf 1}_{  w_r(0)^\top x_i \geq 0 ,  w_r(0)^\top x_j \geq 0 } \right|\leq 1.
\end{align*}
Finally, if $\|w_r(0)- w_r(k)\|_2> R$, we still have
\begin{align*}
\left| {\bf 1}_{  w_r(k)^\top x_i \geq 0,  w_r(k)^\top x_j \geq 0} - {\bf 1}_{  w_r(0)^\top x_i \geq 0 ,  w_r(0)^\top x_j \geq 0 } \right|\leq 1.
\end{align*}
So we have 
\begin{align*}
 \E_{w_r(0)}[s_r]\leq & ~\E_{w_r(0)} \left[ {\bf 1}_{A_{i,r}\vee A_{j,r}} \right] +  \Pr[\|w_r(0)- w_r(k)\|_2> R]\\
 \leq & ~ \Pr[A_{i,r}]+\Pr[A_{j,r}] +  \Pr[\|w_r(0)- w_r(k)\|_2> R]\\
 \leq & ~ \frac {4 R\kappa^{-1}}{\sqrt{2\pi}}+\delta \\
 \leq & ~ 2 R\kappa^{-1}+\delta ,
\end{align*}
and
\begin{align*}
    \E_{w_r(0)} \left[ \left(s_r-\E_{w_r(0)}[s_r] \right)^2 \right]
    = & ~ \E_{w_r(0)}[s_r^2]-\E_{w_r(0)}[s_r]^2 \\
    \leq & ~ \E_{w_r(0)}[s_r^2]\\
    \leq & ~\E_{w_r(0)} \left[ \left( {\bf 1}_{A_{i,r}\vee A_{j,r}} \right)^2 \right]+\delta \\
     \leq & ~ \frac {4R\kappa^{-1}}{\sqrt{2\pi}}+\delta \\
     \leq & ~ 2 R\kappa^{-1} + \delta .
\end{align*}
We also have $|s_r|\leq 1$.
So we can apply Bernstein inequality (Lemma~\ref{lem:bernstein}) to get for all $t>0$,
\begin{align*}
    \Pr \left[\sum_{r=1}^m s_r\geq 2m R\kappa^{-1}+ m \delta +mt \right]
    \leq & ~ \Pr \left[\sum_{r=1}^m (s_r-\E[s_r])\geq mt \right]\\
    \leq & ~ \exp \left( - \frac{ m^2t^2/2 }{ 2m R\kappa^{-1} +m \delta   + mt/3 } \right).
\end{align*}
Choosing $t = R\kappa^{-1}+2\delta$, we get
\begin{align*}
    \Pr \left[\sum_{r=1}^m s_r\geq 3m(R\kappa^{-1}+\delta)  \right]
    \leq & ~ \exp \left( -\frac{ m^2  (R\kappa^{-1}+\delta)^2 /2 }{ 2 m  R\kappa^{-1}  +m \delta +m  (R\kappa^{-1}+2\delta) /3 } \right) \\
     \leq & ~ \exp \left( - m (R\kappa^{-1}+\delta) / 10 \right) .
\end{align*}
Thus, we can have
\begin{align*}
\Pr \left[ \frac{1}{m} \sum_{r=1}^m s_r \geq 3  R\kappa^{-1}+3\delta \right] \leq \exp( - m (R\kappa^{-1}+\delta) /10 ).
\end{align*}
Therefore, by allying union bound all $(i,j)\in [n]\times [n]$, we have
with probability at least $1-n^2\exp( - m (R\kappa^{-1}+\delta) /10 )$,
\begin{align*}
\|H(k)-H(0)\|_F^2\leq n^2(3  R\kappa^{-1}+3\delta)^2
\end{align*}

Similarly, to bound $\|Z(k)-Z(0)\|_F$,
we have
\begin{align*}
\|Z(k)-Z(0)\|_F^2
= & ~ \frac{1}{m} \sum_{r=1}^m \sum_{i=1}^n \left( \|x_i\|_2 \cdot  a_i({\bf 1}_{ w_r(k)^\top x_i \geq 0}-{\bf 1}_{ w_r(0)^\top x_i \geq 0})\right)^2 \\
= & ~ \frac{1}{m} \sum_{r=1}^m \sum_{i=1}^n  \Big({\bf 1}_{ w_r(k)^\top x_i \geq 0}-{\bf 1}_{ w_r(0)^\top x_i \geq 0} \Big)^2\\
= & ~ \frac{1}{m} \sum_{r=1}^m \sum_{i=1}^n  {\bf 1}_{{\bf 1}_{ w_r(k)^\top x_i \geq 0}\neq {\bf 1}_{ w_r(0)^\top x_i \geq 0}} \\
= & ~ \frac{1}{m}  \sum_{i=1}^n \sum_{r=1}^m {\bf 1}_{{\bf 1}_{ w_r(k)^\top x_i \geq 0}\neq {\bf 1}_{ w_r(0)^\top x_i \geq 0}}.
\end{align*}

Fix $i\in [n]$ and for $r\in [m]$,
define $t_r= {\bf 1}_{{\bf 1}_{ w_r(k)^\top x_i \geq 0}\neq {\bf 1}_{ w_r(0)^\top x_i \geq 0}}$.
If $A_{i,r}$ does not happen and $\|w_r(0)- w_r(k)\|_2\leq R$, 
we must have $t_r=0$.
Equivalently,
if $t_r=1$,
then either $A_{i,r}$ happens or $\|w_r(0)- w_r(k)\|_2> R$.
Therefore
\begin{align*}
\E_{w_r(0)}[t_r]\leq \Pr[A_{i,r}]+\Pr[\|w_r(0)- w_r(k)\|_2> R]\leq R\kappa^{-1} + \delta.
\end{align*}
Similarly,
\begin{align*}
\E_{w_r(0)}[(t_r-\E_{w_r(0)}[t_r])^2]
= & ~\E_{w_r(0)}[t_r^2]-\E_{w_r(0)}[t_r]^2\\
\leq & ~\E_{w_r(0)}[t_r^2]\\
= & ~\E_{w_r(0)}[t_r]\\
\leq & ~R\kappa^{-1} + \delta.
\end{align*}
So we can apply Bernstein inequality (Lemma~\ref{lem:bernstein}) to get for all $t>0$,
\begin{align*}
    \Pr \left[\sum_{r=1}^m t_r\geq m R\kappa^{-1}+ m \delta +mt \right]
    \leq & ~ \Pr \left[\sum_{r=1}^m (s_r-\E[s_r])\geq mt \right]\\
    \leq & ~ \exp \left( - \frac{ m^2t^2/2 }{ m R\kappa^{-1} +m \delta   + mt/3 } \right).
\end{align*}
Choosing $t = R\kappa^{-1}+\delta$, we get
\begin{align*}
    \Pr \left[\sum_{r=1}^m t_r\geq 2m(R\kappa^{-1}+\delta)  \right]
    \leq & ~ \exp \left( -\frac{ m^2  (R\kappa^{-1}+\delta)^2 /2 }{  m  R\kappa^{-1}  +m \delta +m  (R\kappa^{-1}+\delta) /3 } \right) \\
     \leq & ~ \exp \left( - m (R\kappa^{-1}+\delta) / 10 \right) .
\end{align*}
By applying union bound over $i\in [n]$,
we have with probability at least $1-n\exp \left( - m (R\kappa^{-1}+\delta) / 10 \right) $,
\begin{align*}
\|Z(k)-Z(0)\|_F^2\leq 2n(R\kappa^{-1}+\delta),
\end{align*}
which is exactly what we need.
\end{proof}

We also need the following parametric version of Lemma \ref{lem:3.1}.
\begin{lemma}\label{lem:3.1_para}
Let $H^{\cts}, H^{\dis} \in \R^{n \times n}$ be defined as in Lemma \ref{lem:3.1}.
Then with probability at least $1-\delta$,
\begin{align*}
\| H^{\dis} - H^{\cts} \|_F =O(n\sqrt{\log(n/\delta)/m}).
\end{align*}
\end{lemma}

The main result in this section is the following theorem.
\begin{theorem}[Improvement of Theorem 4.1 in \cite{adhlw19}]\label{thm:generalization_main}
Recall that $\lambda=\lambda_{\min}(H^{\cts})>0$.
Write the eigen-decomposition of $H^{\cts}$ as $H^{\cts}=\sum_{i=1}^n \lambda_i v_i v_i^\top$ where $v_i\in \mathbb{R}^n$ are the eigenvectors,
and $\lambda_i>0$ are the corresponding eigenvalues.
Let 
\begin{align*}
\kappa = & ~ O( \epsilon / ( \sqrt{2n\log ( 2mn / \delta) }\cdot \log ( 4n / \delta ) ) ) \\
m = & ~ \Omega( \lambda^{-4}\kappa^{-2} n^6 \log (n/\delta) ),
\end{align*}
we i.i.d. initialize $w_r \in {\N}(0,\kappa^2I)$, $a_r$ sampled from $\{-1,+1\}$ uniformly at random for $r\in [m]$, and we set the step size $\eta = O( \lambda / n^2 )$,  then with probability at least $1-\delta$ over the random initialization we have for $k = 0,1,2,\cdots$
\begin{align}
\| u (k) - y \|_2 = \left( \sum_{i=1}^n(1-\eta\lambda_i)^{2k}(v_i^\top y)^2 \right)^{1/2} \pm \epsilon.
\end{align}
\end{theorem}
\begin{proof}
Throughout the proof we assume for all $k=0,1,\cdots$ and all $r\in[m]$,
\begin{align*}
\| w_r(k+1) - w_r(0) \|_2 \leq \frac{ 4 \sqrt{n} \| y - u (0) \|_2 }{ \sqrt{m} \lambda }:=R.
\end{align*}
By Lemma \ref{lem:4.1_restatement},
this holds with probability at least $1-\delta$.

By the update rule of gradient descent,
we have for all $i\in [n]$,
\begin{align}\label{eq:gd_update}
u_i(k+1) - u_i(k)
= & ~ \frac{1}{ \sqrt{m} } \sum_{r=1}^m a_r \cdot \left( \phi( w_r(k+1)^\top x_i ) - \phi(w_r(k)^\top x_i ) \right).
\end{align}
Our target is to relate $u(k+1)-u(k)$ with $u(k)-y$.
Recall that for $i\in [n]$,
the subset $S_i$ is defined as
\begin{align*}
    S_i:=\{r\in [m]:\forall 
    w\in \mathbb{R}^d \text{ s.t. } \|w-w_r(0)\|_2\leq R, 
    \mathbf{1}_{w_r(0)^\top x_i\geq 0}=\mathbf{1}_{w^\top x_i\geq 0}\}.
\end{align*}
Hence we split the summation in \ref{eq:gd_update} by $S_i$,
which is
\begin{align*}
u_i(k+1) - u_i(k)=U_1+U_2
\end{align*}
where 
\begin{align*}
U_1:= & ~\frac{1}{ \sqrt{m} } \sum_{r\in S_i} a_r \cdot \left( \phi( w_r(k+1)^\top x_i ) - \phi(w_r(k)^\top x_i ) \right),\\
U_2:= & ~\frac{1}{ \sqrt{m} } \sum_{r \in \ov{S}_i} a_r \cdot \left( \phi( w_r(k+1)^\top x_i ) - \phi(w_r(k)^\top x_i ) \right).
\end{align*}
We start with upper bounding $U_2$ by considering $U_2$ as a perturbation term.
\begin{align*}
|U_2|= & ~|\frac{1}{ \sqrt{m} } \sum_{r \in \ov{S}_i} a_r \cdot \left( \phi( w_r(k+1)^\top x_i ) - \phi(w_r(k)^\top x_i ) \right)|\\
\leq & ~\frac{1}{ \sqrt{m} } \sum_{r \in \ov{S}_i} \left |w_r(k+1)^\top x_i  - w_r(k)^\top x_i  \right|\\
\leq & ~\frac{1}{ \sqrt{m} } \sum_{r \in \ov{S}_i} \left \|w_r(k+1)  - w_r(k) \right\|_2\\
=& ~\frac{1}{ \sqrt{m} } \sum_{r \in \ov{S}_i} \left \|\eta \frac{ \partial L(W(k)) }{ \partial w_r(k) }\right\|_2\\
 =& ~\frac{1}{ \sqrt{m} } \sum_{r \in \ov{S}_i} \left \|\eta\frac{1}{ \sqrt{m} } \sum_{j=1}^n ( u_j(k) - y_j ) a_r x_j {\bf 1}_{ w_r(k)^\top x_j \geq 0 }\right\|_2\\
 \leq & ~\frac{\eta}{ m} \sum_{r \in \ov{S}_i} \sum_{j=1}^n\left | u_j(k) - y_j \right|\\
  \leq & ~\frac{\eta\sqrt{n}|\ov{S}_i|}{ m} \left\| u(k) - y \right\|_2\\
\end{align*}

where the second step follows from $\phi(\cdot)$ is 1-Lipschitz,
the third step follows from $\|x_i\|_2=1$,
the fifth step follows from Eq. \eqref{eq:gradient},
and the sixth step follows from triangle inequality.

We then bound $U_1$ as follows:
\begin{align*}
& ~ \frac{1}{ \sqrt{m} } \sum_{r\in S_i} a_r \cdot \left( \phi( w_r(k+1)^\top x_i ) - \phi(w_r(k)^\top x_i ) \right)\\
= & ~\frac{1}{ \sqrt{m} } \sum_{r\in S_i} a_r{\bf 1}_{w_r(k)^\top x_i\geq 0} \cdot \left(  w_r(k+1)^\top x_i  - w_r(k)^\top x_i \right)\\
= & ~\frac{1}{ \sqrt{m} } \sum_{r\in S_i} a_r{\bf 1}_{w_r(k)^\top x_i\geq 0} \cdot \left\langle  -\eta\frac{1}{ \sqrt{m} } \sum_{j=1}^n ( u_j(k) - y_j ) a_r x_j {\bf 1}_{ w_r(k)^\top x_j \geq 0 } ,x_i \right\rangle\\
= & ~ - \frac{\eta}{m}\sum_{j=1}^n ( u_j(k) - y_j )x_j^\top x_i \sum_{r\in S_i}{\bf 1}_{w_r(k)^\top x_i\geq 0}{\bf 1}_{w_r(k)^\top x_j\geq 0}\\
= & ~ - \frac{\eta}{m}\sum_{j=1}^n ( u_j(k) - y_j )x_j^\top x_i \sum_{r=1}^{m}{\bf 1}_{w_r(k)^\top x_i\geq 0}{\bf 1}_{w_r(k)^\top x_j\geq 0} \\
 & ~ + \frac{\eta}{m}\sum_{j=1}^n ( u_j(k) - y_j )x_j^\top x_i \sum_{r\in \ov{S}_i}{\bf 1}_{w_r(k)^\top x_i\geq 0}{\bf 1}_{w_r(k)^\top x_j\geq 0}\\
=: & ~-\eta \sum_{j=1}^{n}( u_j(k) - y_j )H_{ij}(k)+\epsilon(k)'
\end{align*}
where the first step follows from the definition of $S_i$,
the fifth step follows from the construction of the matrix $H$.
Moreover,
analogs to upper bounding $U_2$,
we also have
\begin{align*}
|\epsilon(k)'|\leq \frac{\eta\sqrt{n}|\ov{S}_i|}{ m} \left\| u(k) - y \right\|_2.
\end{align*}
Hence by considering all $i\in [n]$,
we have
\begin{align}\label{eq:uk_diff}
u(k+1)-u(k)=-\eta H(k)(u(k)-y)+\epsilon(k)
\end{align}
where $\epsilon(k)\in \mathbb{R}^n$ and the norm of $\epsilon(k)$ can be upper bounded by
\begin{align}\label{eq:eps_k_norm_bound}
\|\epsilon(k)\|_2\leq \|\epsilon(k)\|_1\leq \sum_{i=1}^n(U_2+\epsilon(k)')\leq \sum_{i=1}^n\frac{2\eta\sqrt{n}|\ov{S}_i|}{ m} \left\| u(k) - y \right\|_2.
\end{align}
Notice that
\begin{align*}
\sum_{i=1}^n |\ov{S}_i|=\sum_{r=1}^m \sum_{i=1}^n{\bf 1}_{r\in \ov{S}_i}=\sum_{i=1}^n(\sum_{r=1}^m {\bf 1}_{r\in \ov{S}_i}).
\end{align*}
Hence by Eq. \eqref{eq:Si_size_bound},
with probability at least $1-n\exp(-mR\kappa^{-1})$,
\begin{align*}
\sum_{i=1}^n |\ov{S}_i|\leq 4mnR\kappa^{-1}.
\end{align*}
Plugging the choice of $R=\frac{ 4 \sqrt{n} \| y - u (0) \|_2 }{ \sqrt{m} \lambda }$ and $\|y-u(0)\|_2=O\left(\sqrt{n\log(m/\delta)\log^2(n/\delta)}\right)$,
we have
\begin{align}\label{eq:epsk_norm}
\|\epsilon(k)\|_2 = O \Big( \frac{\eta n^{5/2}\sqrt{\log(m/\delta)\log^2(n/\delta)}}{\lambda \kappa \sqrt{m}} \Big) \cdot \left\| u(k) - y \right\|_2.
\end{align}

Next we relate $H(k)$ with $H^{\cts}$.
We can rewrite Eq. \eqref{eq:uk_diff} as
\begin{align}\label{eq:uk_diff_cts}
u(k+1)-u(k)=-\eta H^{\cts}(u(k)-y)+\xi(k),
\end{align}
where 
\begin{align*}
\xi(k)=\eta(H^{\cts}-H(k))(u(k)-y)+\epsilon(k).
\end{align*}
Notice that
\begin{align}\label{eq:cts_Hk_diff}
\|H^{\cts}-H(k)\|\leq & ~ \|H^{\cts}-H(k)\|_F\notag\\
\leq & ~ \|H^{\cts}-H^{\dis}\|_F+\|H^{\dis}-H(k)\|_F\notag\\
= & ~ O\left(n\Big(\delta+\frac{ n\sqrt{\log(m/\delta)\log^2(n/\delta)}} { \kappa\lambda\sqrt{m}  }\Big)\right)+O(n\sqrt{\log(n/\delta)/m})\notag\\
= & ~O\left(\frac{ n^2\sqrt{\log(m/\delta)\log^2(n/\delta)}} { \kappa\lambda\sqrt{m}  }\right)
\end{align}
where the second step follows from Lemma \ref{lem:HZ_norm_bound} and Lemma \ref{lem:3.1_para}.
Hence we can bound $\xi(k)$ as
\begin{align}\label{eq:xi_bound}
\|\xi(k)\|_2\leq & ~\eta\|(H^{\cts}-H(k))(u(k)-y)\|_2+\|\epsilon(k)\|_2\notag\\
= & ~ O \Big( \frac{ \eta n^2\sqrt{\log(m/\delta)\log^2(n/\delta)}} { \kappa\lambda\sqrt{m} } \Big)\cdot \|u(k)-y\|_2 \notag\\
& ~ + O \Big( \frac{\eta n^{5/2}\sqrt{\log(m/\delta)\log^2(n/\delta)}}{\lambda \kappa \sqrt{m}} \Big) \cdot \left\| u(k) - y \right\|_2\notag\\
= & ~ O \Big( \frac{\eta n^{5/2}\sqrt{\log(m/\delta)\log^2(n/\delta)}}{\lambda \kappa \sqrt{m}} \Big)\cdot \left\| u(k) - y \right\|_2,
\end{align}
where the second step follows from Eq. \eqref{eq:epsk_norm} and Eq. \eqref{eq:cts_Hk_diff}.

Therefore we have
\begin{align}\label{eq:uk_y}
u(k)-y= & ~u(k)-u(k-1)+u(k-1)-y\notag \\
= & ~-\eta H^{\cts}(u(k-1)-y)+\xi(k-1)+u(k-1)-y\notag\\
= & ~ (I-\eta H^{\cts})(u(k-1)-y)+\xi(k-1)\notag\\
= & ~(I-\eta H^{\cts})^k(u(0)-y)+\sum_{t=0}^{k-1}(I-\eta H^{\cts})^t\xi(k-1-t)\notag\\
= & ~(I-\eta H^{\cts})^ku(0)-(I-\eta H^{\cts})^ky+\sum_{t=0}^{k-1}(I-\eta H^{\cts})^t\xi(k-1-t).
\end{align}
where the second step follows from Eq. \eqref{eq:uk_diff_cts}.
We bound terms on RHS respectively.

Recall that $H^{\cts}=\sum_{i=1}^n \lambda_i v_i v_i^\top$.
Hence $v_1,\cdots,v_n$ are also the eigenvectors of $(I-\eta H^{\cts})^k$ with corresponding eigenvalues $(1-\eta\lambda_1)^k,\cdots,(1-\eta\lambda_n)^k$.
With the choice of $\eta=O(\lambda/n^2)$,
we have 
\begin{align}\label{eq:I-cts_spectrum}
\|(I-\eta H^{\cts})^k\|\leq (1-\eta\lambda)^k.
\end{align}
Therefore
\begin{align}\label{eq:y_bound}
\|-(I-\eta H^{\cts})^ky\|_2^2=\sum_{i=1}^n (1-\eta\lambda_i)^{2k}\langle v_i,y\rangle^2.
\end{align}

We then bound the term $(I-\eta H^{\cts})^ku(0)$.
In the proof of Claim \ref{cla:yu0},
we implicitly prove that with probability at least $1-\delta$,
for all $i\in[n]$,
\begin{align*}
|u_i(0)|\leq \kappa \cdot \sqrt{2\log ( 2mn / \delta) }\cdot \log ( 4n / \delta ).
\end{align*}
And this implies
\begin{align*}
\|u(0)\|_2^2 \leq n\kappa^2 \cdot 2\log ( 2mn / \delta) \cdot \log^2 ( 4n / \delta ).
\end{align*}
Therefore
\begin{align}\label{eq:u0_bound}
\|(I-\eta H^{\cts})^ku(0)\|_2
\leq & ~ \|(I-\eta H^{\cts})^k\|\cdot \|u(0)\|_2 \notag \\
\leq & ~ (1-\eta\lambda)^k \cdot \sqrt{n\kappa^2 }\cdot \sqrt{2\log ( 2mn / \delta) }\cdot \log ( 4n / \delta ).
\end{align}
where the second step follows from Eq. \eqref{eq:I-cts_spectrum}.

Finally, we have
\begin{align}\label{eq:sum_xi_bound}
&~\|\sum_{t=0}^{k-1}(I-\eta H^{\cts})^t\xi(k-1-t)\|_2\notag\\
\leq & ~ \sum_{t=0}^{k-1}\|(I-\eta H^{\cts})^t\| \cdot \|\xi(k-1-t)\|_2\notag\\
\leq & ~ \sum_{t=0}^{k-1}(1-\eta \lambda)^t \cdot O \Big( \frac{\eta n^{5/2}\sqrt{\log(m/\delta)\log^2(n/\delta)}}{\lambda \kappa \sqrt{m}} \Big) \cdot \left\| u(k-1-t) - y \right\|_2\notag\\
\leq & ~ \sum_{t=0}^{k-1}(1-\eta \lambda)^t \cdot O \Big( \frac{\eta n^{5/2}\sqrt{\log(m/\delta)\log^2(n/\delta)}}{\lambda \kappa \sqrt{m}} \Big) \cdot (1-\eta\lambda/2)^{(k-1-t)/2}\cdot \|u(0)-y\|_2 \notag\\
\leq & ~ \sum_{t=0}^{k-1}(1-\eta \lambda)^t \cdot O \Big( \frac{\eta n^{5/2}\sqrt{\log(m/\delta)\log^2(n/\delta)}}{\lambda \kappa \sqrt{m}} \Big) \cdot (1-\eta\lambda/4)^{k-1-t}\cdot \sqrt{n\log(m/\delta)\log^2(n/\delta)}\notag\\
\leq & ~ \sum_{t=0}^{k-1}(1-\eta\lambda/4)^{k-1}\cdot O \Big(\frac{\eta n^{3}\log(m/\delta)\log^2(n/\delta)}{\lambda \kappa \sqrt{m}} \Big) \notag\\
= & ~  O \Big(\frac{n^{3}\log(m/\delta)\log^2(n/\delta)}{\lambda^2 \kappa \sqrt{m}}\Big)
\end{align}
where the second step follows from Eq. \eqref{eq:I-cts_spectrum} and Eq. \eqref{eq:xi_bound},
the third step follows from Theorem \ref{thm:quartic},
the fourth step follows from Claim \ref{cla:yu0} and the inequality $(1-x)^{1/2}\leq 1-x/2$,
and the last step follows from $\max_{k\geq 0} k(1-\eta \lambda/4)^{k-1}=O(1/(\eta\lambda))$.

Combining Eq. \eqref{eq:u0_bound}, \eqref{eq:y_bound} and \eqref{eq:sum_xi_bound},
we have
\begin{align*}
\|u(k)-y\|_2= & ~\Big( \sum_{i=1}^n (1-\eta\lambda_i)^{2k}\langle v_i,y\rangle^2 \Big)^{1/2} \\
& ~ \pm O \Big(\sqrt{n\kappa^2 }\cdot \sqrt{2\log ( 2mn / \delta) }\cdot \log ( 4n / \delta )+\frac{n^{3}\log(m/\delta)\log^2(n/\delta)}{\lambda^2 \kappa \sqrt{m}} \Big)^{1/2}\\
= & ~ \left( \sum_{i=1}^n (1-\eta\lambda_i)^{2k}\langle v_i,y\rangle^2 \right)^{1/2} \\
& ~ \pm O\left(\epsilon\right).
\end{align*}
which completes the proof of Theorem \ref{thm:generalization_main}.
\end{proof}

\begin{remark}
The conclusiion in Theorem \ref{thm:generalization_main} can be further strengthened under stronger assumptions on the training data. For example, if part 2 in Assumption \ref{ass:data_dependent_assumption} holds,
notice that
\begin{align*}
\left\|\sum_{j=1}^n ( u_j(k) - y_j ) a_r x_j {\bf 1}_{ w_r(k)^\top x_j \geq 0 }\right\|_2^2
= & ~\sum_{j_1=1}^n \sum_{j_2=1}^n ( u_{j_1}(k) - y_{j_1} )( u_{j_2}(k) - y_{j_2} )x_{j_1}^\top x_{j_2} {\bf 1}_{ w_r(k)^\top x_{j_1} \geq 0,w_r(k)^\top x_{j_2} \geq 0 }\\
= & ~(u(k)-y)^\top H(k)(u(k)-y).
\end{align*}
By part 2 in Assumption \ref{ass:data_dependent_assumption},
we have $\|H(k)\|\leq \alpha$,
therefore
\begin{align*}
\left\|\sum_{j=1}^n ( u_j(k) - y_j ) a_r x_j {\bf 1}_{ w_r(k)^\top x_j \geq 0 }\right\|_2\leq \sqrt{\alpha}\cdot \|u(k)-y\|_2,
\end{align*}
so we can replace Eq. \eqref{eq:eps_k_norm_bound} with
\begin{align*}
\|\epsilon(k)\|_2\leq \sum_{i=1}^n\frac{2\eta \sqrt{\alpha}|\ov{S}_i|}{ m} \left\| u(k) - y \right\|_2.
\end{align*}
Also,
by Theorem \ref{thm:cubic},
we can set $R$ to be $O(\frac{ \sqrt{\alpha} \| y - u (0) \|_2 }{ \sqrt{m} \lambda })$.
Putting everything together,
we have the following Theorem.
\begin{theorem}\label{thm:generalization_main_assumption_2}
Assume Part 1 and 2 of Assumption~\ref{ass:data_dependent_assumption}. 
Recall that $\lambda=\lambda_{\min}(H^{\cts})>0$.
Write the eigen-decomposition of $H^{\cts}$ as $H^{\cts}=\sum_{i=1}^n \lambda_i v_i v_i^\top$ where $v_i\in \mathbb{R}^n$ are the eigenvectors,
and $\lambda_i>0$ are the corresponding eigenvalues.
Let 
\begin{align*}
\kappa = & ~ O( \epsilon / ( \sqrt{2n\log ( 2mn / \delta) }\cdot \log ( 4n / \delta ) ) ) \\
m = & ~ \Omega( \lambda^{-4}\kappa^{-2} n^4\alpha^2 \log (n/\delta) ),
\end{align*}
we i.i.d. initialize $w_r \in {\N}(0,\kappa^2I)$, $a_r$ sampled from $\{-1,+1\}$ uniformly at random for $r\in [m]$, and we set the step size $\eta = O( \lambda / n^2 )$,  then with probability at least $1-\delta$ over the random initialization we have for $k = 0,1,2,\cdots$
\begin{align}
\| u (k) - y \|_2 = \left( \sum_{i=1}^n(1-\eta\lambda_i)^{2k}(v_i^\top y)^2 \right)^{1/2} \pm \epsilon.
\end{align}
\end{theorem}
\end{remark} 
\section{Generalization }\label{sec:generalization}

In this section we improve the generalization result in \cite{adhlw19}.
We first list some useful definitions.
\begin{definition}[Non-degenerate Data Distribution, Definition 5.1 in \cite{adhlw19}]
A distribution $\mathcal{D}$ over $\mathbb{R}^d\times\mathbb{R}$ is \emph{$(\lambda,\delta,n)$-non-degenerate},
if with probability at least $1-\delta$,
for $n$ i.i.d samples $(x_i,y_i)_{i=1}^n$ chosen from $\mathcal{D}$,
$\lambda_{\min}(H^{\cts})\geq \lambda>0$.
\end{definition}

We adapt the notations on the loss functions from \cite{adhlw19}.
\begin{definition}[Loss Functions]
Let $\ell:\mathbb{R}\times \mathbb{R}\rightarrow \mathbb{R}$ be the loss function.
For function $f:\mathbb{R}^d\rightarrow \mathbb{R}$,
for distribution $\mathcal{D}$ over $\mathbb{R}^d\times\mathbb{R}$,
the \emph{population loss} is defined as
\begin{align*}
L_{\mathcal{D}}(f):=\E_{(x,y)\sim \mathcal{D}}[\ell(f(x),y)].
\end{align*}
Let $S=\{(x_i,y_i)\}_{i=1}^n$ be $n$ samples.
The empirical loss over $S$ is defined as
\begin{align*}
L_S(f):=\frac{1}{n}\sum_{i=1}^n \ell(f(x_i),y_i).
\end{align*}
\end{definition}

Rademacher complexity is a useful tool to work with the generalization error.
Here we give the definition.
\begin{definition}[Rademacher Complexity]
Let $\mathcal{F}$ be a class of functions mapping from $\mathbb{R}^d$ to $\mathbb{R}$.
Given $n$ samples $S=\{x_1,\cdots,x_n\}$ where $x_i\in \mathbb{R}^d$ for $i\in [n]$,
the \emph{empirical Rademacher complexity} of $\mathcal{F}$ is defined as
\begin{align*}
\mathcal{R}_S(\mathcal{F}):=\frac{1}{n}\E_{\epsilon}\left[\sup_{f\in \mathcal{F}}\sum_{i=1}^n \epsilon_i f(x_i)\right].
\end{align*}
where $\epsilon\in \mathbb{R}^d$ and each entry of $\epsilon$ are drawn from independently uniform at random from $\{\pm 1\}$.
\end{definition}

\begin{theorem}[Theorem B.1 in \cite{adhlw19}]\label{thm:sample_to_generalization}
Suppose the loss function $\ell(\cdot,\cdot)$ is bounded in $[0,c]$ for some $c>0$ and is $\rho$-Lipschitz in its first argument.
Then with probability at least $1-\delta$ over samples $S$ of size $n$,
\begin{align*}
\sup_{f\in \mathcal{F}} \{L_{\mathcal{D}}(f)-L_S(f)\}\leq 2\rho \mathcal{R}_S(\mathcal{F})+3c\sqrt{\frac{\log(2/\delta)}{2n}}.
\end{align*}
\end{theorem}

\begin{lemma}[Lemma 5.4 in \cite{adhlw19}]\label{lem:rademacher_upper_bound}
Given $R>0$,
with probability at least $1-\delta$ over the random initialization on $W(0)\in \mathbb{R}^{m\times d}$ and $a\in \mathbb{R}^m$,
for all $B>0$,
the function class
\begin{align*}
\mathcal{F}_{R,B}^{W(0),a}=\{f(W,\cdot,a):\|w_r-w_r(0)\|_2\leq R, \forall r\in [m]; \|W-W(0)\|_F\leq B\}
\end{align*}
has bounded empirical Rademacher complexity
\begin{align*}
\mathcal{R}_S(\mathcal{F}_{R,B}^{W(0),a})\leq \frac{B}{\sqrt{2n}}\left(1+(\frac{2\log(2/\delta)}{m})^{1/4}\right)+\frac{2R^2\sqrt{m}}{\kappa}+R\sqrt{\log(2/\delta)}.
\end{align*}
\end{lemma}

Now we prove some technical lemmas which will be used to prove the main result.
\begin{lemma}[Improved version of Lemma 5.3 in \cite{adhlw19}]\label{lem:total_movement}
Suppose $m=\Omega( \lambda^{-4} \kappa^{-2} n^4 \log (n/\delta) )$ and $\eta=O(\frac{\lambda}{n^2})$. 
Then with probability at least $1-\delta$ over the random initialization,
we have for all $k\geq 0$,
\begin{itemize}
\item $\|w_r(k)-w_r(0)\|_2= O(\frac{n\sqrt{\log(m/\delta) \log^2(n/\delta)}}{\lambda\sqrt{m}}), \forall r\in [m]$,
\item $\|W(k)-W(0)\|_F\leq (y^\top (H^{\cts})^{-1}y)^{1/2}+O(\frac{n\kappa\cdot \sqrt{2\log ( 2mn / \delta) }\cdot \log ( 4n / \delta )}{\lambda}+\frac{n^{7/2}\poly(\log m,\log(1/\delta),\lambda^{-1})}{m^{1/4}\kappa^{1/2}})$.
\end{itemize}
\end{lemma}
\begin{proof}
The first part of Lemma \ref{lem:total_movement} follows from Lemma \ref{lem:4.1_restatement} and Claim \ref{cla:yu0}.
We then focus on the proof of the second part.
For weights $w_1,\cdots,w_m\in \mathbb{R}^d$,
Let $\vect(W)=[w_1^\top \, w_2^\top \,\cdots w_m^\top]^\top\in \mathbb{R}^{md}$ be the concatenation of $w_1,\cdots,w_m$.
Then we can rewrite the gradient descent update rule as
\begin{align}\label{eq:vec_grad_update}
\vect(W(k+1))=\vect(W(k))-\eta Z(k)(u(k)-y).
\end{align}

By Eq. \eqref{eq:uk_y}, \eqref{eq:u0_bound} and \eqref{eq:sum_xi_bound},
we have
\begin{align}\label{eq:uk_y_recursion}
u(k)-y=-(I-\eta H^{\cts})^ky+e(k),
\end{align}
where
\begin{align}\label{eq:ek_bound}
e(k)=O\left((1-\eta\lambda)^k \cdot \sqrt{n\kappa^2 }\cdot \sqrt{2\log ( 2mn / \delta) }\cdot \log ( 4n / \delta )+k(1-\eta\lambda/4)^{k-1}\cdot (\frac{\eta n^{3}\log(m/\delta)\log^2(n/\delta)}{\lambda \kappa \sqrt{m}} )\right).
\end{align}

Plugging Eq. \eqref{eq:uk_y_recursion} into Eq. \eqref{eq:vec_grad_update},
we have
\begin{align*}
\vect(W(K))-\vect(W(0))= & ~\sum_{k=0}^{K-1}(\vect(W(k+1))-\vect(W(k)))\\
= & ~\sum_{k=0}^{K-1}(-\eta Z(k)(u(k)-y))\\
= & ~\sum_{k=0}^{K-1}(-\eta Z(k)(-(I-\eta H^{\cts})^ky+e(k)))\\
= & ~\sum_{k=0}^{K-1}\eta Z(k)(I-\eta H^{\cts})^ky-\sum_{k=0}^{K-1}\eta Z(k)e(k)\\
= & ~B_1+B_2+B_3
\end{align*}
where 
\begin{align*}
B_1:= & ~\sum_{k=0}^{K-1}\eta Z(0)(I-\eta H^{\cts})^ky\\
B_2:= & ~\sum_{k=0}^{K-1}\eta (Z(k)-Z(0))(I-\eta H^{\cts})^ky\\
B_3:= & ~-\sum_{k=0}^{K-1}\eta Z(k)e(k)
\end{align*}
We bound these terms separately.
For $B_2$,
we have
\begin{align*}
\|B_2\|_2= & ~ \|\sum_{k=0}^{K-1}\eta (Z(k)-Z(0))(I-\eta H^{\cts})^ky\|_2\\
\leq & ~ \sum_{k=0}^{K-1}\|\eta(Z(k)-Z(0))(I-\eta H^{\cts})^ky\|_2\\
\leq & ~ \sum_{k=0}^{K-1} \eta \cdot \|Z(k)-Z(0)\| \cdot \|(I-\eta H^{\cts})^k\|\cdot \|y\|_2\\
= & ~ O\left(  n \cdot \Big(\delta+\frac{  n\sqrt{\log(m/\delta)\log^2(n/\delta)}} { \kappa\lambda\sqrt{m} } \Big) \right)^{1/2}\cdot \eta\cdot 
\sum_{k=0}^{K-1} \|(I-\eta H^{\cts})\|^k\cdot \|y\|_2\\
\leq & ~ O\left(\frac{n\poly(\log m,\log n, \log(1/\delta))}{m^{1/4}\kappa^{1/2}\lambda^{1/2}}\right)\cdot \eta\cdot 
\sum_{k=0}^{K-1} (1-\eta\lambda)^k\cdot \sqrt{n}\\
= & ~O\left(\frac{n^{3/2}\poly(\log m,\log n, \log(1/\delta))}{m^{1/4}\kappa^{1/2}\lambda^{3/2}}\right)
\end{align*}
where the second step follows from triangle inequality,
the fourth step follows from Lemma \ref{lem:HZ_norm_bound},
the fifth step follows from $y_i=O(1)$ for $i\in [n]$ and Eq. \ref{eq:I-cts_spectrum},
and the last step follows from $\sum_{k=0}^{\infty} (1-\eta\lambda)^k=O(\eta^{-1}\lambda^{-1})$.

Next we bound $B_3$.
Since for $k\geq 0$,
$\|Z(k)\|_F^2\leq \frac{mn}{m}=n$,
we have
\begin{align*}
\|B_3\|_2= & ~\left\|-\sum_{k=0}^{K-1}\eta Z(k)e(k)\right\|_2\\
\leq & ~ \eta \cdot \sqrt{n} \cdot \sum_{k=0}^{K-1} O((1-\eta\lambda)^k \cdot \sqrt{n\kappa^2 }\cdot \sqrt{2\log ( 2mn / \delta) }\cdot \log ( 4n / \delta )\\
+ & ~ k(1-\eta\lambda/4)^{k-1}\cdot (\frac{\eta n^{3}\log(m/\delta)\log^2(n/\delta)}{\lambda \kappa \sqrt{m}} ))\\
\leq & ~ O\left(\frac{n\kappa\cdot \sqrt{2\log ( 2mn / \delta) }\cdot \log ( 4n / \delta )}{\lambda}+\frac{n^{7/2}\log(m/\delta)\log^2(n/\delta)}{\lambda^3 \kappa \sqrt{m}}\right)
\end{align*}
where the second step follows from $\|Z(k)\|\leq \|Z(k)\|_F\leq \sqrt{n}$ and Eq. \eqref{eq:ek_bound},
the third step follows from $\sum_{k=0}^{K-1}(1-\eta\lambda)^k\leq \sum_{k=0}^{\infty}(1-\eta\lambda)^k=\frac{1}{\eta\lambda}$ and  $\sum_{k=0}^{K-1}k(1-\eta\lambda/4)^{k-1}\leq \sum_{k=0}^{\infty}k(1-\eta\lambda/4)^{k-1}=O(\frac{1}{\eta^2\lambda^2})$.

Finally we bound $B_1$.
Define $T=\sum_{k=0}^{K-1}\eta (I-\eta H^{\cts})^k\in \mathbb{R}^{n\times n}$,
then we have
\begin{align*}
\|B_1\|_2^2= & ~\|\sum_{k=0}^{K-1}\eta Z(0)(I-\eta H^{\cts})^ky\|_2^2\\
= & ~\|Z(0)Ty\|_2^2\\
= & ~y^\top T^\top Z(0)^\top Z(0) Ty\\
= & ~y^\top T^\top H(0) Ty\\
= & ~y^\top T^\top H^{\cts} Ty+y^\top T^\top (H(0)-H^{\cts}) Ty\\
\leq & ~ y^\top T^\top H^{\cts} Ty+\|H(0)-H^{\cts}\|\cdot \|T\|^2 \|y\|_2^2\\
= & ~ y^\top T^\top H^{\cts} Ty+O(n\sqrt{\log(n/\delta)/m})\cdot (\eta\sum_{k=0}^{K-1}(1-\eta\lambda)^k)^2 \|y\|_2^2\\
= & ~ y^\top T^\top H^{\cts} Ty+O(n\sqrt{\log(n/\delta)/m})\cdot \lambda^{-2}n\\
= & ~ y^\top T^\top H^{\cts} Ty+O(\frac{n^2\sqrt{\log(n/\delta)}}{\lambda^2 \sqrt{m}})
\end{align*}
where the seventh step follows from $H(0)$ is just $H^{\dis}$, Lemma \ref{lem:3.1_para} and Eq. \ref{eq:I-cts_spectrum},
the eighth step follows from $y_i=O(1)$ for $i\in [n]$.

Notice that $T$ is a polynomial of $H^{\cts}$,
so they have the same set of eigenvectors.
Moreover, for $i\in [n]$,
recall that $v_i$ is the eigenvector of $H^{\cts}$ with eigenvalue $\lambda_i$,
then
\begin{align*}
Tv_i=\sum_{k=0}^{K-1}\eta (I-\eta H^{\cts})^kv_i=\eta \sum_{k=0}^{K-1}(1-\lambda_i)^kv_i.
\end{align*}
Namely $v_i$ is an eigenvector of $T$ with eigenvalue $\eta \sum_{k=0}^{K-1}(1-\eta\lambda_i)^k$.
So we can write $T$ as
\begin{align*}
T=\sum_{i=1}^n \eta \sum_{k=0}^{K-1}(1-\eta\lambda_i)^k v_i v_i^\top
\end{align*}
Therefore
\begin{align*}
T^\top H^{\cts} T=\sum_{i=1}^n (\eta \sum_{k=0}^{K-1}(1-\eta\lambda_i)^k)^2\cdot \lambda_i v_i v_i^\top
\end{align*}
For all $i\in [n]$, we have 
\begin{align*}
(\eta \sum_{k=0}^{K-1}(1-\eta\lambda_i)^k)^2\cdot \lambda_i\leq (\eta \sum_{k=0}^{\infty}(1-\eta\lambda_i)^{k})^2\cdot \lambda_i=\lambda_i^{-1}.
\end{align*}
Hence
\begin{align*}
T^\top H^{\cts} T \preceq \sum_{i=1}^n \lambda_i^{-1} v_i v_i^\top=(H^{\cts})^{-1}
\end{align*}
which gives us
\begin{align*}
\|B_1\|_2 \leq (y^\top (H^{\cts})^{-1} y)^{1/2}+O((\frac{n^2\sqrt{\log(n/\delta)}}{\lambda^2 \sqrt{m}})^{1/2}).
\end{align*}
Putting everything together,
we have
\begin{align*}
& ~\|W(k)-W(0)\|_F \\
= & ~\|\vect(W(K))-\vect(W(0))\|_2\\
\leq & ~(y^\top (H^{\cts})^{-1} y)^{1/2}+O((\frac{n^2\sqrt{\log(n/\delta)}}{\lambda^2 \sqrt{m}})^{1/2})
+ O\left(\frac{n^{3/2}\poly(\log m,\log n, \log(1/\delta))}{m^{1/4}\kappa^{1/2}\lambda^{3/2}}\right)\\
+ & ~O\left(\frac{n\kappa\cdot \sqrt{2\log ( 2mn / \delta) }\cdot \log ( 4n / \delta )}{\lambda}+\frac{n^{7/2}\log(m/\delta)\log^2(n/\delta)}{\lambda^3 \kappa \sqrt{m}}\right)\\
= & ~ (y^\top (H^{\cts})^{-1} y)^{1/2} +O(\frac{n\kappa\cdot \sqrt{2\log ( 2mn / \delta) }\cdot \log ( 4n / \delta )}{\lambda}+\frac{n^{7/2}\poly(\log m,\log(1/\delta),\lambda^{-1})}{m^{1/4}\kappa^{1/2}}),
\end{align*}
which completes the proof of Lemma \ref{lem:total_movement}.
\end{proof}

Now we can present our main result in this section.
\begin{theorem}[Improved version of Theorem 5.1 in \cite{adhlw19}]\label{thm:generalization}
Fix failure probability $\delta \in (0,1)$.
Suppose the training data $S=\{(x_i,y_i)\}_{i=1}^n$ are i.i.d samples from a $(\lambda,\delta/3,n)$-non-degenerate distribution $\mathcal{D}$,
and $\kappa=O(\frac{\lambda \poly(\log n,\log(1/\delta))}{n})$,
$m\geq \kappa^{-2}(n^{14}\poly(\log m,\log(1/\delta),\lambda^{-1}))$.
Consider any loss function $\ell:\mathbb{R}\times \mathbb{R}\rightarrow [0,1]$ that is 1-Lipschitz in its first argument.
Then with probability at least $1-\delta$ over the random initialization on $W(0)\in \mathbb{R}^{m\times d}$ and $a\in \mathbb{R}^m$ and the training samples,
the two layer neural network $f(W(k),\cdot,a)$ trained by gradient descent for $k\geq \Omega\left(\frac{1}{\eta\lambda}\log(n\poly(\log n,\log(1/\delta)))\right)$ iterations has population loss $L_{\mathcal{D}}(f)=\E_{(x,y)\sim \mathcal{D}}[\ell(f(W(k),x,a),y)]$ upper bounded as
\begin{align*}
L_{\mathcal{D}}(f)\leq \sqrt{\frac{2y^\top (H^{\cts})^{-1}y}{n}}+O\left(\sqrt{\frac{\log\frac{n}{\lambda \delta}}{2n}}\right).
\end{align*}
\end{theorem}
\begin{proof}
We will define a sequence of failing events and bound these failure probability individually,
then we can apply the union bound to obtain the desired result.

Let $E_1$ be the event that $\lambda_{\min}(H^{\dis})<\lambda$.
Because $\mathcal{D}$ is $(\lambda,\delta/3,n)$-non-degenerate,
$\Pr[E_1]\leq \epsilon/3$.
In the remaining of the proof we assume $E_1$ does not happen.

Let $E_2$ be the event that $L_S(f(W(k),\cdot,a))=\frac{1}{n}\sum_{i=1}^n \ell(f(W(k),x_i,a),y_i)>\frac{1}{\sqrt{n}}$.
By Theorem \ref{thm:quartic} with scaling $\delta$ properly,
 with probability $1-\delta/9$,
\begin{align}\label{eq:uk_y_recursive}
\| u (k) - y \|_2^2 \leq ( 1 - \eta \lambda / 2 )^k \cdot \| u (0) - y \|_2^2.
\end{align}
If this happens,
then when $k=\Omega\left(\frac{1}{\eta\lambda}\cdot \log\left(n\log(m/\delta)\log^2(n/\delta)\right)\right)$,
we have
\begin{align*}
L_S(f(W(k),\cdot,a))= & ~\frac{1}{n}\sum_{i=1}^n \ell(f(W(k),x_i,a),y_i)\\
= & ~\frac{1}{n}\sum_{i=1}^n (\ell(f(W(k),x_i,a),y)-\ell(y_i,y_i))\\
\leq & ~\frac{1}{n}\sum_{i=1}^n|u(k)_i-y_i|\\
\leq & ~\frac{\|u(k)-y\|_2}{\sqrt{n}}\\
\leq & ~\frac{1}{\sqrt{n}}.
\end{align*}
where the second step follows from $\ell(y,y)=0$,
the third step follows from $\ell$ is 1-Lipschitz in its first argument,
and the fifth step follows from the choice of $k$, Eq. \eqref{eq:uk_y_recursive} and Claim \ref{cla:yu0}.
So we have $\Pr[E_2]\leq \delta/9$.

Set $R,B>0$ as
\begin{align*}
R= & ~O(\frac{n\sqrt{\log(m/\delta) \log^2(n/\delta)}}{\lambda\sqrt{m}}),\\
B= & ~(y^\top (H^{\cts})^{-1}y)^{1/2}+O(\frac{n\kappa\cdot \sqrt{2\log ( 2mn / \delta) }\cdot \log ( 4n / \delta )}{\lambda}+\frac{n^{7/2}\poly(\log m,\log(1/\delta),\lambda^{-1})}{m^{1/4}\kappa^{1/2}}).
\end{align*}
Notice that $\|y\|_2=O(\sqrt{n})$ and $\|(H^{\cts})^{-1}\|=1/\lambda$.
By our setting of $\kappa=O(\frac{\lambda \poly(\log n,\log(1/\delta))}{n})$ and $m\kappa^2\geq n^{14}>n^{12}$,
 $B=O(\sqrt{n/\lambda})$.
Let $E_3$ be the event that there exists $r\in [m]$ so that $\|w_r-w_r(0)\|_2> R$, or $\|W-W(0)\|_F> B$.
By Lemma \ref{lem:total_movement}, $\Pr[E_3]\leq \delta/9$.

For $i=1,2,\cdots,$,
let $B_i=i$.
Let $E_4$ be the event that there exists $i>0$ so that 
\begin{align*}
\mathcal{R}_S(\mathcal{F}_{R,B_i}^{W(0),a})> \frac{B_i}{\sqrt{2n}}\left(1+(\frac{2\log(18/\delta)}{m})^{1/4}\right)+\frac{2R^2\sqrt{m}}{\kappa}+R\sqrt{\log(18/\delta)}.
\end{align*}
By Lemma \ref{lem:rademacher_upper_bound},
$\Pr[E_4]\leq 1-\delta/9$.

Assume neither of $E_3,E_4$ happens.
Let $i^*$ be the smallest integer so that $B_{i^*}=i^*\geq B$,
then we have $B_{i^*}\leq B+1$ and $i^*=O(\sqrt{n/\lambda})$.
Since  $E_3$ does not happen,
we have $f(W(k),\cdot,a)\in \mathcal{F}_{R,B_{i^*}}^{W(0),a}$.
Moreover,
\begin{align*}
& ~\mathcal{R}_S(\mathcal{F}_{R,B_{i^*}}^{W(0),a})\\
\leq & ~\frac{B+1}{\sqrt{2n}}\left(1+(\frac{2\log(18/\delta)}{m})^{1/4}\right)+\frac{2R^2\sqrt{m}}{\kappa}+R\sqrt{\log(18/\delta)}\\
= & ~\sqrt{\frac{y^\top (H^{\cts})^{-1}y}{2n}}+\frac{1}{\sqrt{n}}+O(\frac{\sqrt{n}\kappa\cdot \poly(\log m,\log n, \log(1/\delta))}{\lambda})\\
+ & ~\frac{n^{3}\poly(\log m,\log(1/\delta),\lambda^{-1})}{m^{1/4}\kappa^{1/2}}+\frac{2R^2\sqrt{m}}{\kappa}+R\sqrt{\log(18/\delta)}\\
= & ~\sqrt{\frac{y^\top (H^{\cts})^{-1}y}{2n}}+\frac{1}{\sqrt{n}}+O(\frac{\sqrt{n}\kappa\cdot \poly(\log m,\log n, \log(1/\delta))}{\lambda})+\frac{n^{3}\poly(\log m,\log(1/\delta),\lambda^{-1})}{m^{1/4}\kappa^{1/2}}\\
= & ~\sqrt{\frac{y^\top (H^{\cts})^{-1}y}{2n}}+\frac{2}{\sqrt{n}}
\end{align*}
where the first step follows from $E_4$ does not happen and the choice of $B$,
the second step follows from the choice of $R$,
and the last step follows from the choice of $m$ and $\kappa$.

Finally,
let $E_5$ be the event so that there exists $i\in\{1,2,\cdots,O(\sqrt{n/\lambda})\}$ so that
\begin{align*}
\sup_{f\in \mathcal{F}_{R,B_i}^{W(0),a}} \{L_{\mathcal{D}}(f)-L_S(f)\}> 2\mathcal{R}_S(\mathcal{F}_{R,B_i}^{W(0),a})+O\left(\sqrt{\frac{\log\frac{n}{\lambda \delta}}{2n}}\right).
\end{align*}
By Theorem \ref{thm:sample_to_generalization} and applying union bound on $i$,
we have $\Pr[E_5]\leq \delta/3$.

Assume all of the bad events $E_1,E_2,E_3,E_4,E_5$ do not happen.
We have with probability at least $1-\delta$,
\begin{align*}
&f(W(k),\cdot,a)\in \mathcal{F}_{R,B_{i^*}}^{W(0),a},\\
&L_S(f(W(k),\cdot,a))\leq \frac{1}{\sqrt{n}},\\
&\mathcal{R}_S(\mathcal{F}_{R,B_{i^*}}^{W(0),a})\leq \sqrt{\frac{y^\top (H^{\cts})^{-1}y}{2n}}+\frac{2}{\sqrt{n}},\\
&\sup_{f\in \mathcal{F}_{R,B_{i^*}}^{W(0),a}} \{L_{\mathcal{D}}(f)-L_S(f)\}\leq 2\mathcal{R}_S(\mathcal{F}_{R,B_{i^*}}^{W(0),a})+O\left(\sqrt{\frac{\log\frac{n}{\lambda \delta}}{2n}}\right).
\end{align*}
which gives us
\begin{align*}
L_{\mathcal{D}}(f(W(k),\cdot,a))
\leq & ~L_S(f(W(k),\cdot,a))+2\mathcal{R}_S(\mathcal{F}_{R,B_{i^*}}^{W(0),a})+O\left(\sqrt{\frac{\log\frac{n}{\lambda \delta}}{2n}}\right)\\
\leq & ~\sqrt{2\frac{y^\top (H^{\cts})^{-1}y}{2n}}+\frac{5}{\sqrt{n}}+O\left(\sqrt{\frac{\log\frac{n}{\lambda \delta}}{2n}}\right)\\
= & ~\sqrt{2\frac{y^\top (H^{\cts})^{-1}y}{2n}}+O\left(\sqrt{\frac{\log\frac{n}{\lambda \delta}}{2n}}\right).
\end{align*}
which is exactly what we need.
\end{proof}
\section{Training with Regularization}\label{sec:regularization}

In this section, we study the problem of training neural network with regularization. 
We consider a two-layer ReLU activated neural network with $m$ neurons in the hidden layer:
\begin{align*}
f (W,x,a) = \frac{1}{ \sqrt{m} } \sum_{r=1}^m a_r \phi ( w_r^\top x ) ,
\end{align*}
where $x \in \R^d$ is the input, $w_1, \cdots, w_m \in \R^d$ are weight vectors in the first layer, $a_1, \cdots, a_m \in \R$ are weights in the second layer.  For simplicity, we only optimize $W \in \R^{m \times d}$ but not optimize $a \in \R^m$ and $W \in \R^{m \times d}$ at the same time.

Recall that the ReLU function $\phi(x)=\max\{x,0\}$.
Therefore for $r\in [m]$,
we have
\begin{align}\label{eq:relu_derivative_reg}
\frac{f (W,x,a)}{\partial w_r}=\frac{1}{ \sqrt{m} } a_r x{\bf 1}_{ w_r^\top x \geq 0 }.
\end{align}

For $\beta>0$,
we define objective function $L$ (after initialization with weights $W(0),a$) as follows
\begin{align*}
L (W) = \frac{1}{2} \sum_{i=1}^n ( y_i - f (W,x_i,a) )^2 +\frac{\beta}{2m} \|W-W(0)\|_F^2.
\end{align*}

We apply the gradient descent to optimize the weight matrix $W \in \R^{m \times d}$ in the following standard way,
\begin{align}\label{eq:w_update_reg}
W(k+1) = W(k) - \eta \frac{ \partial L( W(k) ) }{ \partial W(k) } .
\end{align}

We can compute the gradient of $L$ in terms of $w_r \in \R^d$
\begin{align}\label{eq:gradient_reg}
\frac{ \partial L(W) }{ \partial w_r } = & ~ \frac{1}{ \sqrt{m} } \sum_{i=1}^n ( f(W,x_i,a_r) - y_i ) a_r x_i {\bf 1}_{ w_r^\top x_i \geq 0 } + \frac{\beta}{m}\cdot (w_r-w_r(0)).
\end{align}

We consider the ordinary differential equation defined by
\begin{align}\label{eq:wr_derivative_reg}
\frac{\d w_r(t)}{\d t}=-\frac{ \partial L(W) }{ \partial w_r }.
\end{align}

At time $t$,
let $u(t)=(u_1(t),\cdots,u_n(t))\in \R^n$ be the prediction vector where each $u_i(t)$ is defined as
\begin{align}\label{eq:ut_def_reg}
u_i(t)=f(W(t),a,x_i).
\end{align}

Now, we can bound the gradient norm. 
For all discrete time $k$,

\begin{align}\label{eq:gradient_bound_reg}
& ~ \left\| \frac{ \partial L(W_k) }{ \partial w_r } \right\|_2 \notag\\
= & ~ \left\| \sum_{i=1}^n (y_i - u_i(k)) \frac{1}{\sqrt{m}} a_r x_i \cdot {\bf 1}_{ w_r(k)^\top x_i \geq 0 } +\frac{\beta}{m}\cdot \left(w_r(k)-w_r(0)\right)\right\|_2 \notag\\
\leq & ~ \frac{1}{ \sqrt{m} } \sum_{i=1}^n | y_i - u_i(k) |+\frac{\beta}{m}\|w_r(k)-w_r(0)\|_2 \notag\\
\leq & ~ \frac{ \sqrt{n} }{ \sqrt{m} } \| y - u(k) \|_2+\frac{\beta}{m}\|w_r(k)-w_r(0)\|_2 
\end{align}
where the first step follows from \eqref{eq:gradient_reg}, the second step follows from triangle inequality and $a_r=\pm 1$ for $r\in [m]$ and $\|x_i\|_2=1$ for $i\in [n]$, the third step follows from Cauchy-Schwartz inequality.

\subsection{Convergence for Training with Regularization}\label{sec:reg_converge}

\begin{theorem}[Main Result for Regularization]\label{thm:reg}
Let $\lambda=\lambda_{\min}(H^{\cts})>0$.
Let $m = \Omega( \lambda^{-4} n^4 \log (n/\delta) )$, we i.i.d. initialize $w_r \in {\cal N}(0,I_d)$,
$a_r$ sampled from $\{-1,+1\}$ uniformly at random for $r\in [m]$, and we set the step size $\eta = O( \lambda / n^2 )$.
For any integer $K\geq 1$,
if the regularization factor $\beta$ satisfies $\beta\leq \min\{\frac{m^2\lambda}{128K^2n\eta},\frac{m}{4K\eta}\}$,
 then with probability at least $1-\delta$ over the random initialization we have for all $k = 0,1,2,\cdots,K$ 
\begin{align}\label{eq:quartic_condition_reg}
\| u (k) - y \|_2^2 \leq ( 1 - \eta \lambda / 2 )^k \cdot \| u (0) - y \|_2^2+\frac{8\beta D^2}{m\eta\lambda}.
\end{align}
where $D$ is defined as 
\begin{align*}
D=\frac{ 8 \sqrt{n} \| y - u (0) \|_2 }{ \sqrt{m} \lambda }.
\end{align*}

\end{theorem}

\paragraph{Correctness}
We prove Theorem \ref{thm:reg} by induction.
The base case is $i=0$ and it is trivially true.
Assume for $i=0,\cdots,k$ we have proved \eqref{eq:quartic_condition_reg} to be true.
We want to show \eqref{eq:quartic_condition_reg} holds for $i=k+1$.

From the induction hypothesis,
we have the following Lemma stating that the weights should not change too much.
\begin{lemma}[Regularization version of Corollary 4.1 in \cite{dzps19}]\label{lem:4.1_reg}
If Eq.~\eqref{eq:quartic_condition_reg} holds for $i = 0, \cdots, k$,
and $\beta\leq \min\{\frac{m^2\lambda}{128K^2n\eta},\frac{m}{4K\eta}\}$,
then we have for all $r\in [m]$
\begin{align*}
\| w_r(k+1) - w_r(0) \|_2 \leq \frac{ 8 \sqrt{n} \| y - u (0) \|_2 }{ \sqrt{m} \lambda } := D
\end{align*}
\end{lemma}
\begin{proof}
We use the norm of gradient to bound this distance,

\begin{align*}
&~\| w_r(k+1) - w_r(0) \|_2 \\
= & ~ \eta \left\|\sum_{i=0}^k  \frac{ \partial L( W(i) ) }{ \partial w_r(i) } \right\|_2 \\
\leq & ~ \eta \sum_{i=0}^k \left(\frac{ \sqrt{n} \| y - u(i) \|_2 }{ \sqrt{m} }\right) +\eta \beta m^{-1} \|\sum_{i=0}^k(w_r(i)-w_r(0))\|_2\\
\leq & ~ \eta \sum_{i=0}^k \frac{ \sqrt{n}  }{ \sqrt{m} } \left(( 1 - \eta \lambda /2 )^{i/2}\| y - u(0) \|_2+\sqrt{\frac{8\beta }{m\eta\lambda}}D\right) + k\eta \beta m^{-1} D \\
\leq & ~ \eta \sum_{i=0}^{\infty} \frac{ \sqrt{n} (1-\eta\lambda/2)^{i/2} }{ \sqrt{m} } \| y - u(0) \|_2 + \sqrt{\frac{8n\eta\beta}{\lambda}}Km^{-1}D+K\eta \beta m^{-1}D \\
= & ~ \frac{ 4 \sqrt{n} \| y - u(0) \|_2 }{ \sqrt{m} \lambda }+ \sqrt{\frac{8n\eta\beta}{\lambda}}Km^{-1}D+K\eta \beta m^{-1}D \\
\leq & ~ \frac{D}{2}+\frac{D}{4}+\frac{D}{4}\\
= & ~ D,
\end{align*}

where the first step follows from \eqref{eq:w_update_reg}, the second step follows from \eqref{eq:gradient_bound_reg}, the third step follows from the induction hypothesis and the inequality $(a^2+b^2)^{1/2}\leq a+b$, the fourth step relaxes the summation to an infinite summation, and the fifth step follows from $\sum_{i=0}^{\infty}(1-\eta\lambda/2)^{i/2}=\frac {2}{\eta\lambda}$, and the last step follows from $\beta\leq \min\{\frac{m^2\lambda}{128K^2n\eta},\frac{m}{4K\eta}\}$ and $\eta=\frac{\lambda}{16n^2}$.

Thus, we complete the proof.
\end{proof}

Next, we calculate the different of predictions between two consecutive iterations.
For each $i \in [n]$, we have
\begin{align*}
& ~ u_i(k+1) - u_i(k)\\
= & ~ \frac{1}{ \sqrt{m} } \sum_{r=1}^m a_r \cdot \left( \phi( w_r(k+1)^\top x_i ) - \phi(w_r(k)^\top x_i ) \right) \\
= & ~ \frac{1}{\sqrt{m}} \sum_{r=1}^m a_r \cdot \Big( \phi \left( \Big( w_r(k) - \eta \frac{ \partial L( W(k) ) }{ \partial w_r(k) } \Big)^\top x_i \right) - \phi ( w_r(k)^\top x_i ) \Big).
\end{align*}

Here we divide the right hand side into two parts. $v_{1,i}$ represents the terms that the pattern does not change and $v_{2,i}$ represents the term that pattern may changes. For each $i \in [n]$,
we define the set $S_i\subset [m]$ as
\begin{align*}
    S_i:= \{ & ~ r\in [m]:\forall 
    w\in \R^d \text{ s.t. } \| w - w_r(0) \|_2\leq R, \mathbf{1}_{w_r(0)^\top x_i\geq 0} = \mathbf{1}_{w^\top x_i\geq 0} \}.
\end{align*} 
Then we define $v_{1,i}$ and $v_{2,i}$ as follows
\begin{align*}
v_{1,i} : = & ~ \frac{1}{ \sqrt{m} } \sum_{r \in S_i} a_r \cdot \Big( \phi \left( \Big( w_r(k) - \eta \frac{ \partial L(W(k)) }{ \partial w_r(k) } \Big)^\top x_i \right)  - \phi( w_r(k)^\top x_i ) \Big), \\
v_{2,i} : = & ~ \frac{1}{ \sqrt{m} } \sum_{r \in \ov{S}_i} a_r \cdot \Big( \phi \left( \Big( w_r(k) - \eta \frac{ \partial L(W(k)) }{ \partial w_r(k) } \Big)^\top x_i \right)  - \phi( w_r(k)^\top x_i ) \Big) .
\end{align*} 

Thus, we can rewrite $u(k+1) - u(k) \in \R^n$ in the following sense
\begin{align*}
u(k+1) - u(k) = v_1 + v_2 .
\end{align*}

In order to analyze $v_1 \in \R^n$, we provide definition of $H$ and $H^{\bot} \in \R^{n \times n}$ first,
\begin{align*}
H(k)_{i,j} = & ~ \frac{1}{m} \sum_{r=1}^m x_i^\top x_j {\bf 1}_{ w_r(k)^\top x_i \geq 0, w_r(k)^\top x_j \geq 0 } , \\
H(k)^{\bot}_{i,j} = & ~ \frac{1}{m} \sum_{r\in \ov{S}_i} x_i^\top x_j {\bf 1}_{ w_r(k)^\top x_i \geq 0, w_r(k)^\top x_j \geq 0 } .
\end{align*}

Then, we can rewrite $v_{1,i} \in \R$
\begin{align*}
& ~v_{1,i}  \\
= & ~ \frac {1}{\sqrt{m}}\sum_{r\in S_i}{\bf 1}_{w_r(k)^\top x_i \geq 0}a_r  \cdot \Bigg(-\eta x_i^\top \cdot  \Big( \frac{1}{\sqrt{m}}\sum_{j=1}^n (u_j-y_j)a_rx_j{\bf 1}_{w_r(k)^\top x_j\geq 0} \Big) +\frac{\beta}{m} (w_r(k)-w_r(0)) \Bigg)\\
= & ~ - \frac{\eta}{ m } \sum_{j=1}^n x_i^\top x_j (u_j - y_j) \sum_{r \in S_i} {\bf 1}_{ w_r(k)^\top x_i \geq 0 , w_r(k)^\top x_j \geq 0 } \\
& ~ - \frac{\eta \beta}{m^{3/2}}\sum_{r\in S_i}{\bf 1}_{w_r(k)^\top x_i \geq 0}a_r(w_r(k)^\top x_i-w_r(0)^\top x_i)
\\
= & ~ - \eta \sum_{j=1}^n (u_j - y_j) ( H_{i,j}(k) - H_{i,j}^{\bot}(k) ) \\
& ~ -\frac{\eta \beta}{m^{3/2}}\sum_{r\in S_i}{\bf 1}_{w_r(k)^\top x_i \geq 0}a_r(w_r(k)^\top x_i-w_r(0)^\top x_i),
\end{align*}
which means vector $v_1 \in \R^n$ can be written as
\begin{align}\label{eq:rewrite_v1_reg}
v_1 = \eta ( y - u(k) )^\top ( H( k ) - H^{\bot}( k ) ) -\eta v_3,
\end{align}
where for $i\in [n]$,
\begin{align*}
    v_{3,i}=\frac{ \beta}{m^{3/2}}\sum_{r\in S_i}{\bf 1}_{w_r(k)^\top x_i \geq 0}a_r(w_r(k)^\top x_i-w_r(0)^\top x_i).
\end{align*}

We are ready to prove the induction hypothesis. We can rewrite $\| y - u(k+1) \|_2^2$ as follows:
\begin{align*}
& ~ \| y - u(k+1) \|_2^2\\
= & ~ \| y - u(k) - ( u(k+1) - u(k) ) \|_2^2 \\
= & ~ \| y - u(k) \|_2^2 - 2 ( y - u(k) )^\top  ( u(k+1) - u(k) ) + \| u (k+1) - u(k) \|_2^2 .
\end{align*}

We can rewrite the second term in the above Equation in the following sense,
\begin{align*}
 & ~ ( y - u(k) )^\top ( u(k+1) - u(k) ) \\
= & ~ ( y - u(k) )^\top ( v_1 + v_2 )  \\
= & ~ ( y - u(k) )^\top v_1 + ( y - u(k) )^\top v_2  \\
= & ~ \eta ( y - u(k) )^\top H(k) ( y - u (k) ) \\
& ~ - \eta ( y - u(k) )^\top H(k)^{\bot} ( y - u(k) ) \\
& ~ -\eta (y-u(k))^\top v_3
+ ( y - u(k) )^\top v_2 ,
\end{align*}
where the third step follows from Eq.~\eqref{eq:rewrite_v1_reg}.

We define 
\begin{align*}
C_1 = & ~ -2 \eta (y - u(k))^\top H(k) ( y - u(k) ) , \\
C_2 = & ~ 2 \eta ( y - u(k) )^\top H(k)^{\bot} ( y - u(k) ) , \\
C_3 = & ~ - 2 ( y - u(k) )^\top v_2 , \\
C_4 = & ~ \| u (k+1) - u(k) \|_2^2,\\
C_5 = & ~ 2\eta(y-u(k))^\top v_3. 
\end{align*}
Thus, we have
\begin{align*}
& ~ \| y - u(k+1) \|_2^2 \\
= & ~ \| y - u(k) \|_2^2 + C_1 + C_2 + C_3 + C_4 +C_5 \\
\leq & ~ \| y - u(k) \|_2^2 ( 1 - \eta \lambda + 8 \eta n R  + 8 \eta n R  + 2\eta^2 n^2 ) \\
+ & ~ 2\eta\beta m^{-1}\sqrt{n}\|y-u(k)\|_2\| W(k)-W(0)\|_F\\
+ & ~ 8R\beta\eta\sqrt{n/m}  \| y - u(k) \|_2\cdot \max_{r\in [m]}\|w_r(k)-w_r(0)\|_2\\
+ & ~2\eta^2n\beta^2m^{-2}\|W(k)-W(0)\|_F^2,
\end{align*}
where the last step follows from Claim~\ref{cla:C1_reg},  \ref{cla:C2_reg}, \ref{cla:C3_reg}, \ref{cla:C4_reg} and \ref{cla:C5_reg},
which we will prove given later.

Notice that $\|W(k)-W(0)\|_F\leq \sqrt{m}\cdot \max_{r\in [m]}\|w_r(k)-w_r(0)\|_2$.
For simplicity, let $\Delta$ be defined as 
\begin{align*}
\Delta = \max_{r\in [m]}\|w_r(k)-w_r(0)\|_2
\end{align*}
From the inequality $ka^2+b^2/k\geq 2ab$ for all $k,a,b\geq 0$,
we have
\begin{align*}
& ~ m^{-1}\|y-u(k)\|_2\cdot \sqrt{m}\max_{r\in [m]}\|w_r(k)-w_r(0)\|_2 \\
\leq & ~ \frac{\eta\sqrt{n}}{2}\|y-u(k)\|_2^2 \\
+ & ~ \frac{1}{2\eta m^2\sqrt{n}} \cdot \left( \sqrt{m} \max_{r\in [m]} \|w_r(k)-w_r(0)\|_2 \right)^2
\end{align*}
We will set $R$ so that $R\leq 1$,
and in  this case we have
\begin{align*}
& ~ \| y - u(k+1) \|_2^2 \\
\leq & ~ ( 1 - \eta \lambda + 8 \eta n R  + 8 \eta n R  + 2\eta^2 n^2 + 2 \eta^2 \beta n  ) \cdot \|y-u(k)\|_2^2\\
+ & ~ (2\eta^2n\beta^2m^{-2}+2\beta m^{-2}) \cdot \left( \sqrt{m}\max_{r\in [m]} \| w_r(k) - w_r(0) \|_2 \right)^2
\end{align*}

\paragraph{Choice of $\eta$ and $R$.}

Next, we want to choose $\eta$ and $R$ such that
\begin{align}\label{eq:choice_of_eta_R_reg}
( 1 - \eta \lambda + 8 \eta n R  + 8 \eta n R  +  2\eta^2 n^2 + 2\eta^2\beta n   ) \leq 1 - \eta \lambda / 2 .
\end{align}

If we set $\eta=\frac{\lambda }{16n^2}$ and $R=\frac{\lambda}{64n}$, when $\beta$ is a fixed constant,
we have 
\begin{align*}
8 \eta n R  + 8 \eta n R =16\eta nR \leq  \eta \lambda /4 ,
\mathrm{~and~} \eta^2\beta n \leq \eta^2 n^2 \leq  \eta \lambda / 16.
\end{align*}
Moreover,
when $\beta<\frac{1}{n\eta^2}=\frac{256n^3}{\lambda^2}$,
we have
\begin{align*}
2\eta^2n\beta^2m^{-2}+2\beta m^{-2}
\leq & ~ 2\beta m^{-2}+2\beta m^{-2} \\
= & ~ 4\beta m^{-2}.
\end{align*}
This implies
\begin{align*}
& ~ \| y - u(k+1) \|_2^2 \\
\leq & ~ \| y - u(k) \|_2^2 \cdot ( 1 - \eta \lambda / 2 )  +4\beta m^{-2} \cdot \left( \sqrt{m}\max_{r\in [m]}\|w_r(k)-w_r(0)\|_2 \right)^2
\end{align*}
holds with probability at least $1-3n^2 \cdot \exp(-m R /10)$.

Recall that $D\geq \max_{r\in [m]}\|w_r(k)-w_r(0)\|_2$ for $k \in [K]$.
Hence we have
\begin{align*}
\| y - u(k) \|_2^2 \leq & ~ \| y - u(0) \|_2^2 \cdot ( 1 - \eta \lambda / 2 )^k + 4\beta m^{-2}(\sqrt{m}D)^2\sum_{i=0}^{k-1}( 1 - \eta \lambda / 2 )^i\\
\leq & ~\| y - u(0) \|_2^2 \cdot ( 1 - \eta \lambda / 2 )^k+\frac{8\beta D^2}{m\eta\lambda}.
\end{align*}
\paragraph{Over-parameterization size, lower bound on $m$.}

We require 
\begin{align*}
D= \frac{8\sqrt{n}\|y-u(0)\|_2}{\sqrt{m}\lambda} < R = \frac{\lambda}{64n} ,
\end{align*}
and 
\begin{align*}
3n^2 \cdot \exp(-m R /10)\leq \delta .
\end{align*}
By Claim \ref{cla:yu0_reg},
 it is sufficient to choose $m = \Omega( \lambda^{-4} n^4 \log(m/\delta)\log^2(n/\delta) )$.

 \subsection{Technical claims}
\begin{claim}\label{cla:yu0_reg}
For $0<\delta<1$,
with probability at least $1-\delta$,
\begin{align*}
\|y-u(0)\|_2^2=O(n\log(m/\delta)\log^2(n/\delta)).
\end{align*}
\end{claim}
\begin{proof}

\begin{align*}
\|y-u(0)\|_2^2
= & ~ \sum_{i=1}^n(y_i-f(W(0),a,x_i))^2\\
= & ~ \sum_{i=1}^n \Big( y_i-\frac {1} {\sqrt{m}}\sum_{r=1}^{m} a_r\phi(w_r^\top x_i) \Big)^2\\
= & ~ \sum_{i=1}^n y_i^2-2\sum_{i=1}^n \frac{y_i}{\sqrt{m}}\sum_{r=1}^{m} a_r\phi(w_r^\top x_i)+\sum_{i=1}^n \frac {1}{m} \Big( \sum_{r=1}^{m} a_r\phi(w_r^\top x_i) \Big)^2.
\end{align*}

Fix $r\in [m]$ and $i\in [n]$.
Since $w_r\sim N(0,I)$ and $\|x_i\|_2=1$,
$w_r^\top x_i$ follows distribution $N(0,1)$.
From concentration of Gaussian distribution,
we have
\begin{align*}
\Pr_{w_r} \left[ w_r^\top x_i\geq \sqrt{2\log (2mn / \delta) } \right] \leq \frac{\delta}{2mn}.
\end{align*}
Let $E_1$ be the event that
for all $r\in [m]$ and $i\in [n]$ we have
\begin{align*}
\phi(w_r^\top x_i)\leq \sqrt{2\log ( 2mn/ \delta) }.
\end{align*}
Then by union bound,
$\Pr[E_1]\geq 1-\frac {\delta}{2}$,

Fix $i\in [n]$.
For every $r\in [m]$,
we define random variable $z_{i,r}$ as
\begin{align*}
z_{i,r}:=\frac {1}{\sqrt{m}} \cdot a_r \cdot \phi(w_r^\top x_i) \cdot \mathbf{1}_{w_r^\top x_i\leq \sqrt{2\log ( 2mn / \delta ) }}.
\end{align*}
Then $z_{i,r}$ only depends on $a_r\in \{-1,1\}$ and $w_r\sim N(0,I)$.
Notice that $\E_{a_r,w_r}[z_{i,r}]=0$,
and $|z_{i,r}|\leq \sqrt{2\log ( 2mn / \delta ) }$.
Moreover,
\begin{align*}
\E_{a_r,w_r}[z_{i,r}^2]
= & ~\E_{a_r,w_r}\left[\frac {1}{m}a_r^2\phi^2(w_r^\top x_i)\mathbf{1}^2_{w_r^\top x_i\leq \sqrt{2\log ( 2mn / \delta) }}\right]\\
= & ~\frac {1}{m}\E_{a_r}[a_r^2] \cdot \E_{w_r} \Big[\phi^2(w_r^\top x_i)\mathbf{1}^2_{w_r^\top x_i\leq \sqrt{2\log ( 2mn / \delta) }} \Big]\\
\leq & ~\frac {1}{m}\cdot 1 \cdot \E_{w_r}[(w_r^\top x_i)^2]\\
= & ~\frac {1} {m},
\end{align*}
where the second step uses independence between $a_r$ and $w_r$,
the third step uses $a_r\in \{-1,1\}$ and $\phi(t) = \max \{ t,0\}$,
and the last step follows from $w_r^\top x_i\sim N(0,1)$.

Now we are ready to apply Bernstein inequality~(Lemma \ref{lem:bernstein}) to get for all $t>0$,
\begin{align*}
\Pr \left[ \sum_{r=1}^m z_{i,r}>t \right] \leq \exp\left(-\frac{t^2/2}{m\cdot \frac{1}{m}+\sqrt{2\log (2mn/\delta)} \cdot t/3} \right).
\end{align*}
Setting $t=\sqrt{2\log ( 2mn / \delta) }\cdot \log ( 4n / \delta )$,
we have with probability at least $1-\frac {\delta}{4n}$,
\begin{align*}
\sum_{r=1}^m z_{i,r}\leq \sqrt{2\log ( 2mn / \delta) }\cdot \log ( 4n / \delta ).
\end{align*}

Notice that we can also apply Bernstein inequality (Lemma~\ref{lem:bernstein}) on $-z_{i,r}$ to get
\begin{align*}
\Pr \left[ \sum_{r=1}^m z_{i,r}<-t \right] \leq \exp\left(-\frac{t^2/2}{m\cdot \frac{1}{m}+\sqrt{2\log (2mn/\delta)} \cdot t/3} \right).
\end{align*}
Let $E_2$ be the event that
for all $i\in[n]$,
\begin{align*}
\left| \sum_{r=1}^m z_{i,r} \right| \leq \sqrt{2\log ( 2mn / \delta) }\cdot \log ( 4n / \delta ).
\end{align*}
By applying union bound on all $i\in [n]$,
we have $\Pr[E_2]\geq 1-\frac {\delta}{2}$.

If both $E_1$ and $E_2$ happen,
we have
\begin{align*}
\|y-u(0)\|_2^2
= & ~ \sum_{i=1}^n y_i^2-2\sum_{i=1}^n \frac{y_i}{\sqrt{m}}\sum_{r=1}^{m} a_r\phi(w_r^\top x_i)+\sum_{i=1}^n \frac {1}{m} \Big( \sum_{r=1}^{m} a_r\phi(w_r^\top x_i) \Big)^2\\
= & ~ \sum_{i=1}^n y_i^2-2\sum_{i=1}^n y_i\sum_{r=1}^{m} z_{i,r}+\sum_{i=1}^n \Big( \sum_{r=1}^{m}z_{i,r} \Big)^2\\
\leq & ~\sum_{i=1}^n y_i^2+2\sum_{i=1}^n |y_i|\sqrt{2\log ( 2mn / \delta) }\cdot \log ( 4n / \delta )+\sum_{i=1}^n \Big( \sqrt{2\log ( 2mn / \delta) }\cdot \log ( 4n / \delta ) \Big)^2\\
= & ~ O(n\log(m/\delta)\log^2(n/\delta)) ,
\end{align*}
where the second step uses $E_1$, the third step uses $E_2$, and the last step follows from $|y_i| = O(1), \forall i \in [n]$.

By union bound, this will happen with probability at least $1-\delta$.
\end{proof}

\begin{claim}[Upper bound on $C_1$]\label{cla:C1_reg}
Let $C_1 = -2 \eta (y - u(k))^\top H(k) ( y - u(k) )$ . With probability at least $1-n^2 \cdot \exp(-m R /10)$,
we have
\begin{align*}
C_1 \leq - \| y - u(k) \|_2^2 \cdot \eta \lambda .
\end{align*}
\end{claim}
\begin{proof}
By Lemma \ref{lem:3.2} and our choice of $R<\frac{\lambda}{8n}$,
We have $\|H(0)-H(k)\|_F\leq 2n\cdot \frac{\lambda}{8n}=\frac {\lambda}{4}$.
Recall that $\lambda=\lambda_{\min}(H(0))$.
Therefore
\begin{align*}
\lambda_{\min}(H(k)) \geq \lambda_{\min}(H(0))- \|H(0)-H(k)\|\geq \lambda /2.
\end{align*}
Then we have
\begin{align*}
  (y - u(k))^\top H(k) ( y - u(k) ) \geq \| y - u(k) \|_2^2 \cdot \lambda / 2.
\end{align*}
Thus, we complete the proof.
\end{proof}

\begin{claim}[Upper bound on $C_2$]\label{cla:C2_reg}
Let $C_2 = 2 \eta ( y - u(k) )^\top H(k)^{\bot} ( y - u(k) )$. We have
\begin{align*}
C_2 \leq \| y - u(k) \|_2^2 \cdot 8\eta nR.
\end{align*}
holds with probability $1-n\exp(-mR)$.
\end{claim}
\begin{proof}
Note that
\begin{align*}
C_2 \leq 2 \eta \| y - u(k) \|_2^2 \| H(k)^{\bot} \|.
\end{align*}
It suffices to upper bound $\| H(k)^{\bot} \|$. Since $\| \cdot \| \leq \| \cdot \|_F$, then it suffices to upper bound $\| \cdot \|_F$.

For each $i \in [n]$, we define $y_i$ as follows
\begin{align*}
y_i=\sum_{r=1}^m\mathbf{1}_{r\in \ov{S}_i} .
\end{align*}

Then we have
\begin{align*}
\| H(k)^{\bot} \|_F^2
= & ~ \sum_{i=1}^n\sum_{j=1}^n (H(k)^{\bot}_{i,j})^2\\
= & ~ \sum_{i=1}^n\sum_{j=1}^n \Big( \frac {1} {m}\sum_{r\in \ov{S}_i} x_i^\top x_j\mathbf{1}_{w_r(k)^\top x_i\geq 0,w_r(k)^\top x_j\geq 0} \Big)^2\\
= & ~ \sum_{i=1}^n\sum_{j=1}^n \Big( \frac {1} {m}\sum_{r=1}^m x_i^\top x_j\mathbf{1}_{w_r(k)^\top x_i\geq 0,w_r(k)^\top x_j\geq 0} \cdot \mathbf{1}_{r\in \ov{S}_i} \Big)^2\\
= & ~ \sum_{i=1}^n\sum_{j=1}^n ( \frac {x_i^\top x_j} {m} )^2 \Big( \sum_{r=1}^m \mathbf{1}_{w_r(k)^\top x_i\geq 0,w_r(k)^\top x_j\geq 0} \cdot \mathbf{1}_{r\in \ov{S}_i} \Big)^2 \\
\leq & ~ \frac{1}{m^2} \sum_{i=1}^n\sum_{j=1}^n \Big( \sum_{r=1}^m \mathbf{1}_{w_r(k)^\top x_i\geq 0,w_r(k)^\top x_j\geq 0} \cdot \mathbf{1}_{r\in \ov{S}_i} \Big)^2 \\
= & ~ \frac{n}{m^2} \sum_{i=1}^n \Big( \sum_{r=1}^m \mathbf{1}_{r\in \ov{S}_i} \Big)^2 \\
= & ~ \frac{n}{m^2} \sum_{i=1}^n y_i^2 .
\end{align*}

Fix $i \in [n]$. The plan is to use Bernstein inequality to upper bound $y_i$ with high probability.

First by Eq.~\eqref{eq:Air_bound} we have 
\begin{align*}
\E[\mathbf{1}_{r\in \ov{S}_i}]\leq R .
\end{align*}
We also have
\begin{align*}
\E \left[(\mathbf{1}_{r\in \ov{S}_i}-\E[\mathbf{1}_{r\in \ov{S}_i}])^2 \right]
 = & ~ \E[\mathbf{1}_{r\in \ov{S}_i}^2]-\E[\mathbf{1}_{r\in \ov{S}_i}]^2\\
\leq & ~ \E[\mathbf{1}_{r\in \ov{S}_i}^2]\\
\leq & ~ R .
\end{align*}
Finally we have $|\mathbf{1}_{r\in \ov{S}_i}-\E[\mathbf{1}_{r\in \ov{S}_i}]|\leq 1$.

Notice that $\{\mathbf{1}_{r\in \ov{S}_i}\}_{r=1}^m$ are mutually independent,
since $\mathbf{1}_{r\in \ov{S}_i}$ only depends on $w_r(0)$.
Hence from Bernstein inequality (Lemma \ref{lem:bernstein}) we have for all $t>0$,
\begin{align*}
\Pr \left[ y_i > m\cdot R+t \right] \leq \exp \left(-\frac{t^2/2}{m\cdot R+t/3} \right).
\end{align*}
By setting $t=3mR$, we have
\begin{align*}
\Pr \left[ y_i > 4mR \right] \leq \exp(-mR).
\end{align*}
Hence by union bound,
with probability at least $1-n\exp(-mR)$,
\begin{align*}
\| H(k)^{\bot} \|_F^2\leq \frac{n}{m^2}\cdot n\cdot (4mR)^2=16n^2R^2 .
\end{align*}
Putting all together we have
\begin{align*}
\| H(k)^{\bot} \|\leq \| H(k)^{\bot} \|_F\leq 4nR
\end{align*}
with probability at least $1-n\exp(-mR)$.

\end{proof}

\begin{claim}[Upper bound on $C_3$]\label{cla:C3_reg}
Let $C_3 = - 2 (y - u(k))^\top v_2$. Then we have
\begin{align*}
C_3 \leq \| y - u(k) \|_2^2 \cdot 8 \eta nR  +8R\beta\eta\sqrt{n/m}  \| y - u(k) \|_2\cdot \max_{r\in [m]}\|w_r(k)-w_r(0)\|_2.
\end{align*}
with probability at least $1-n\exp(-mR)$.
\end{claim}
\begin{proof}
We have
\begin{align*}
\LHS \leq 2 \| y - u(k) \|_2 \cdot \| v_2 \|_2 .
\end{align*}
We can upper bound $\| v_2 \|_2$ in the following sense
\begin{align*}
\| v_2 \|_2^2
\leq &~ \sum_{i=1}^n \left(\frac{\eta}{ \sqrt{m} } \sum_{ r \in \ov{S}_i } \left| ( \frac{ \partial L(W(k)) }{ \partial w_r(k) } )^\top x_i \right|\right)^2\\
= &~ \frac{\eta^2}{ m }\sum_{i=1}^n \left(\sum_{r=1}^m \mathbf{1}_{r\in \ov{S}_i}\left| ( \frac{ \partial L(W(k)) }{ \partial w_r(k) } )^\top x_i \right|\right)^2\\
\leq &~ \frac{\eta^2}{ m }\cdot \max_{r \in [m]} \left|  \frac{ \partial L(W(k)) }{ \partial w_r(k) } \right|^2\cdot\sum_{i=1}^n \left(\sum_{r=1}^m \mathbf{1}_{r\in \ov{S}_i}\right)^2\\
 \leq & ~ \frac{\eta^2}{ m }\cdot \left(\frac{ \sqrt{n} }{ \sqrt{m} } \| u(k) - y\|_2 +\frac{\beta}{m}\cdot \max_{r\in [m]}\|w_r(k)-w_r(0)\|_2\right)^2 \cdot \sum_{i=1}^n \left(\sum_{r=1}^m \mathbf{1}_{r\in \ov{S}_i}\right)^2\\
  \leq & ~ \frac{\eta^2}{ m }\cdot \left(\frac{ \sqrt{n} }{ \sqrt{m} } \| u(k) - y\|_2 +\frac{\beta}{m}\cdot \max_{r\in [m]}\|w_r(k)-w_r(0)\|_2\right)^2 \cdot \sum_{i=1}^n (4mR)^2\\
\end{align*}
where the first step follows from definition of $v_2$, 
the fifth step follows from $\sum_{r=1}^m {\bf 1}_{r \in \ov{S}_i } \leq 4 m R$ with probability at least $1-\exp(-m R)$,
and the bound of $\| \frac{ \partial L(W(k))}{ \partial w_r(k) }\|_2$ follows from \eqref{eq:gradient_bound}.

Hence we have
\begin{align*}
    \LHS
    \leq & ~ 2 \| y - u(k) \|_2 \cdot \sqrt{16mnR^2\eta^2}\left(\frac{\sqrt{n}}{\sqrt{m}} \| u(k) - y\|_2 +\frac{\beta}{m}\cdot \max_{r\in [m]}\|w_r-w_r(0)\|_2\right)\\
    = & ~  \| y - u(k) \|_2^2 \cdot 8 \eta nR
    +8R\beta\eta\sqrt{n/m} \| y - u(k) \|_2\cdot  \max_{r\in [m]}\|w_r(k)-w_r(0)\|_2
\end{align*}
\end{proof}

\begin{claim}[Upper bound on $C_4$]\label{cla:C4_reg}
Let $C_4  = \| u (k+1) - u(k) \|_2^2$. Then we have
\begin{align*}
C_4 \leq 2\eta^2 n^2 \| y - u(k) \|_2^2
+2\eta^2n\beta^2m^{-2}\|W(k)-W(0)\|_F^2.
\end{align*}
\end{claim}
\begin{proof}
We have
\begin{align*}
\LHS 
= & ~ \sum_{i=1}^n \left(f(W(k+1),x_i,a)-f(W(k),x_i,a)\right)^2\\
= & ~\sum_{i=1}^n \left(\frac{1}{\sqrt{m}}\sum_{r=1}^m a_r(\phi(w_r(k+1)^\top x)-\phi(w_r(k)^\top x))\right)^2\\
\leq & ~\sum_{i=1}^n \left(\frac{1}{\sqrt{m}}\sum_{r=1}^m \|w_r(k+1)-w_r(k)\|_2\right)^2\\
= & ~ \eta^2 \sum_{i=1}^n \frac{1}{m} \left( \sum_{r=1}^m \Big\| \frac{ \partial L( W(k) ) }{ \partial w_r(k) } \Big\|_2 \right)^2 \\
\leq & ~\eta^2 \sum_{i=1}^n \frac{1}{m}\sum_{r=1}^m \left(\frac{ \sqrt{n} }{ \sqrt{m} } \| u(k) - y\|_2 +\frac{\beta}{m}\cdot \|w_r(k)-w_r(0)\|_2\right)^2 \\
\leq & ~ 2\eta^2 n^2 \| y - u(k) \|_2^2
+2\eta^2n\beta^2m^{-2}\|W(k)-W(0)\|_F^2,
\end{align*}
where the first step follows from the definition of $u(k)$, the second step follows from explicit expression of $f$, the third step follows from $\phi$ is 1-lipschitz, the fourth step follows from the update rule of weights,
the last step follows from the fact $(a+b)^2\leq 2a^2+2b^2 $.
\end{proof}

\begin{claim}[Upper bound on $C_5$]\label{cla:C5_reg}
Let $C_5  =2\eta(y-u(k))^\top v_3$. Then we have
\begin{align*}
C_5 \leq 2\eta\beta m^{-1}\sqrt{n}\|y-u(k)\|_2\| W(k)-W(0)\|_F.
\end{align*}
\end{claim}
\begin{proof}
We have
\begin{align*}
|(y-u(k))^\top v_3| \leq & ~
\|y-u(k)\|_2\cdot \|v_3\|_2\\
= & ~\|y-u(k)\|_2\cdot \left(\sum_{i=1}^{n}(\frac{ \beta}{m^{3/2}}\sum_{r\in S_i}{\bf 1}_{w_r(k)^\top x_i \geq 0}a_r(w_r(k)^\top x_i-w_r(0)^\top x_i))^2\right)^{1/2}\\
\leq &  ~\|y-u(k)\|_2\cdot \left(\sum_{i=1}^{n}\left(\frac{ \beta}{m^{3/2}}\left(\sum_{r=1}^m\|w_r(k)-w_r(0)\|_2\right)\right)^2\right)^{1/2}\\
\leq & ~\|y-u(k)\|_2\cdot \left(\sum_{i=1}^{n}\beta^2m^{-2}\| W(k)-W(0)\|_F^2\right)^{1/2}\\
= & ~ \| y-u(k) \|_2 \cdot \beta m^{-1}\sqrt{n} \cdot \| W(k)-W(0)\|_F
\end{align*}
where the first step is the Cauchy-Schwartz inequality,
the second step calls the definition of $v_3$,
the third step uses triangle inequality
the fourth step follows from GM-AM inequality.
\end{proof}





\end{document}